%% file: main.tex
\documentclass[11pt]{article}
\usepackage[a4paper,top=2.75cm,bottom=2.75cm,left=2.5cm,right=2.5cm,marginparwidth=1.75cm]{geometry}
\input{preamble}
\usepackage[round]{natbib}

\usepackage[capitalize,noabbrev]{cleveref} 
\usepackage{autonum}

\newtheorem{remark}{Remark}
\newtheorem{definition}{Definition}
\newtheorem{assumption}{Assumption}
\newtheorem{theorem}{Theorem}
\newtheorem{corollary}{Corollary}
\newtheorem{lemma}{Lemma}

\newtheorem{proposition}{Proposition}

\input{abhin_macros}

\input{cref_setup}

\graphicspath{{../figures/},{figures/}}

\title{Group Fairness with Uncertainty in Sensitive Attributes}
\author[1]{Abhin Shah}
\author[1]{Maohao Shen}
\author[1]{Jongha Jon Ryu}
\author[2,3]{Subhro Das}
\author[2,3]{Prasanna Sattigeri}
\author[4]{Yuheng Bu}
\author[1]{Gregory W. Wornell}

\affil[1]{Massachusetts Institute of Technology}
\affil[2]{MIT-IBM Watson AI Lab}
\affil[3]{IBM Research}
\affil[4]{University of Florida}
\date{}
\begin{document}
\sloppy
\maketitle
\begin{abstract}
\input{0abstract}
\end{abstract}
\input{1introduction}
\input{2related_work.tex}
\input{3problem_formulation}
\input{4main_results}
\input{5algorithm.tex}
\input{6experiments.tex}
\input{7future_work}

\bibliographystyle{abbrvnat}
\bibliography{main}
\clearpage
\appendix
\noindent {\bf \LARGE{Appendix}}
\input{8appendix}
\end{document}

%% file: preamble.tex
\usepackage{xcolor}
\PassOptionsToPackage{hyphens}{url}\usepackage[colorlinks = true,
            linkcolor = blue,
            urlcolor  = blue,
            citecolor = blue,
            anchorcolor = blue]{hyperref}
\usepackage[normalem]{ulem}
\usepackage[none]{hyphenat}
\usepackage{breakcites}
\usepackage[utf8]{inputenc}
\usepackage{amsmath,amsthm}
\usepackage{amssymb}
\usepackage{mathtools}
\usepackage{soul}
\usepackage{scalerel}
\allowdisplaybreaks
\usepackage{tikz}
\usepackage{bm}
\usetikzlibrary{arrows, angles,calc,intersections, quotes,through,backgrounds}
\usepackage{dsfont}
\usepackage{graphicx, graphics}
\usepackage{wrapfig}
\usepackage{thmtools,thm-restate}
\usepackage{enumitem} %
\usepackage{caption}
\usepackage{booktabs} %
\usepackage{lipsum}
\usepackage{microtype}
\usepackage{graphicx}
\usepackage{subfigure}
\usepackage[noblocks]{authblk}
\usepackage{adjustbox}
\usepackage[ruled,algo2e]{algorithm2e}
\usepackage{algorithm}

\SetCommentSty{mycommfont}
\SetKwInput{KwInput}{Input}                %
\SetKwInput{KwOutput}{Output}  
\SetKwInput{KwInitialization}{Initialization}

%% file: abhin_macros.tex
\def\tbf#1{\textbf{#1}}

\newcommand{\tp}{^\top}

\DeclareMathAlphabet{\mathbsf}{OT1}{cmss}{bx}{n}%
\DeclareMathAlphabet{\mathssf}{OT1}{cmss}{m}{sl}%

\def\what#1{\widehat{#1}}
\def\wbar#1{\overline{#1}}
\DeclareMathOperator*{\argmax}{arg\,max}
\DeclareMathOperator*{\argmin}{arg\,min}

\DeclareMathOperator{\proj}{proj}

\newcommand{\rvx}{\mathssf{x}}	%
\newcommand{\rve}{\mathssf{e}}	%
\newcommand{\rvy}{\mathssf{y}}	%
\newcommand{\rvu}{\mathssf{u}}	%

\newcommand{\hrve}{\hat{\mathssf{e}}}	%

\newcommand{\rvbx}{\mathbsf{x}}	%
\newcommand{\rvbv}{\mathbsf{v}}	%
\newcommand{\rvba}{\mathbsf{a}}	%
\newcommand{\truervba}{\mathbsf{a}^{\star}}	%
\newcommand{\limitedinfrvba}{\wbar{\mathbsf{a}}}	%
\newcommand{\limitedrvba}{\wbar{\mathbsf{a}}}	%
\newcommand{\rvbw}{\mathbsf{w}}	%

\newcommand{\svbx}{\boldsymbol{x}} %

\newcommand{\perf}[1][]{\Psi_{#1}(n)}
\newcommand{\loss}{\ell}
\newcommand{\cF}{\mathcal{F}}

\newcommand{\cD}{\mathcal{D}}
\newcommand{\cX}{\mathcal{X}}
\newcommand{\cY}{\mathcal{Y}}
\newcommand{\cE}{\mathcal{E}}
\newcommand{\cN}{\mathcal{N}}
\newcommand{\cSn}[1][n]{\mathcal{A}(\Delta(#1), \phi(#1))}
\newcommand{\bss}[1][S]{\texttt{Bootstrap-#1}}
\newcommand{\s}[1][S]{\texttt{#1}}

\newcommand{\cS}{\mathcal{A}(\Delta, \phi)}
\newcommand{\bcS}{\wbar{\mathcal{A}}(\Delta, \phi)}
\newcommand{\bS}{\mathbb{A}}
\newcommand{\cDu}{\mathcal{D}^{(u)}}
\newcommand{\cDo}{\mathcal{D}^{(o)}}
\newcommand{\cDp}{\mathcal{D}^{(p)}}
\newcommand{\cDi}[1][i]{\mathcal{D}^{(u)}_{#1}}

\newcommand{\bsigma}{\boldsymbol{\Sigma}}
\newcommand{\radius}{\tau}
\newcommand{\ball}{\mathcal{B}(0,1)}
\newcommand{\ballunc}{\mathcal{B}(\BexEstimated,\radius)}
\newcommand{\balluncn}{\mathcal{B}(\BexEstimated,\radius(n))}

\newcommand{\Expectation}{\mathbb{E}}

\newcommand{\Variance}{\mathbb{V}\text{ar}}

\newcommand{\Divergence}{\wbar{D}}

\newcommand{\B}{\mathbf{b}}

\newcommand{\Bxu}{\mathbf{b}_{xu}}
\newcommand{\Byx}{\mathbf{b}_{yx}}
\newcommand{\Bex}{\mathbf{b}_{ex}}
\newcommand{\BexEstimated}{\what{\mathbf{b}}_{ex}}
\newcommand{\Bux}{\mathbf{b}_{ux}}

\newcommand{\SigmaYY}{\mathbf{\Sigma}_{yy}}
\newcommand{\SigmaXX}{\mathbf{\Sigma}_{xx}}
\newcommand{\SigmaXXi}[1][i]{\mathbf{\Sigma}_{x_{#1}x_{#1}}}

\newcommand{\SigmaEX}{\mathbf{\Sigma}_{ex}}
\newcommand{\SigmaEXestimated}{\what{\mathbf{\Sigma}}_{ex}}

\newcommand{\SigmaUY}{\mathbf{\Sigma}_{uy}}
\newcommand{\SigmaYU}{\mathbf{\Sigma}_{yu}}
\newcommand{\SigmaUU}{\mathbf{\Sigma}_{uu}}
\newcommand{\SigmaEE}{\mathbf{\Sigma}_{ee}}
\newcommand{\trace}[1]{\mathrm{Tr}(#1)}

\newcommand{\Reals}{\mathbb{R}} %

\newcommand{\re}{r_e}
\newcommand{\hre}{\what{r}_e}
\newcommand{\ry}{r_y}

\newcommand{\sless}[1]{\stackrel{#1}{\leq}}

\newcommand{\sequal}[1]{\stackrel{#1}{=}}

\newcommand{\normalbrackets}[1]{[ #1 ]}

\newcommand{\bigbrackets}[1]{\big[ #1 \big]}

\newcommand{\normalparenth}[1]{( #1 )}

\newcommand{\bigparenth}[1]{\big( #1 \big)}

\newcommand{\biggparenth}[1]{\bigg( #1 \bigg)}
\newcommand{\normalbraces}[1]{\{ #1  \}}

\newcommand{\bigbraces}[1]{\big\{ #1 \big \}}

\newcommand{\normalabs}[1]{| #1  |}
\newcommand{\bigabs}[1]{\big| #1 \big|}

\newcommand{\normalinner}[2]{\langle #1,#2 \rangle}
\newcommand{\biginner}[2]{\big\langle #1,#2 \big \rangle}

\def\defeq{\triangleq} %
\newcommand{\defn}{\defeq}
\newcommand{\qtext}[1]{\quad\text{#1}\quad} 
\newcommand{\stext}[1]{\ \text{#1}\ }

\def\norm#1{\left\|{#1}\right\|} %
 \def\snorm#1{\|{#1}\|} %
\newcommand{\subGnorm}[1]{\norm{#1}_{\Psi_2}} %
\newcommand{\subEnorm}[1]{\norm{#1}_{\Psi_1}} %
\newcommand{\stwonorm}[1]{\snorm{#1}_2} %
\newcommand{\sinfnorm}[1]{\snorm{#1}_{\infty}} %

\newcommand{\matnorm}[1]{|\!|\!| #1 | \! | \!|}
\newcommand{\fronorm}[1]{\matnorm{#1}_{\mathrm{F}}} %

\newcommand{\polar}[2]{#1(\cos{#2}, \sin{#2})}
\newcommand{\thetay}{\theta_y}
\newcommand{\thetae}{\theta_e}
\newcommand{\hthetae}{\what{\theta}_e}
\newcommand{\balpha}{\wbar{\alpha}}
\newcommand{\Afun}[1][]{\mathbf{a}#1}

\newcommand{\ParameterSetTrue}{\Lambda_{\rvba}}

\newcommand{\Deltal}{\Delta}
\newcommand{\Deltau}{\Delta}
\newcommand{\phil}{\phi}
\newcommand{\phiu}{\phi}

\newcommand{\indep}{\hspace{1mm}{\perp \!\!\! \perp}}
\newcommand{\chidiv}[2]{\chi^2\left( #1\, \middle\Vert #2 \right)}

\newcommand{\qcqpsubspace}{$d = 2$ suffices for QCQP}

%% file: cref_setup.tex
\newtheorem{condition}{Condition}
\newtheorem{problem}{Problem}
\newtheorem{fact}{Fact}

\newtheorem{myproposition}{Proposition}

\newtheorem{mycorollary}{Corollary}

\newtheorem{mylemma}{Lemma}

\newtheorem{mydefinition}{Definition}

 \newtheorem{myproblem}{Problem}

\newtheorem{myremark}{Remark}

\newtheorem{myassumption}{Assumption}

\newtheorem{mycondition}{Condition}

\newtheorem{myfact}{Fact}

 \crefname{appendix}{Appendix}{Appendices}
\crefname{equation}{}{}
\crefname{lemma}{Lemma}{Lemmas}
\crefname{theorem}{Theorem}{Theorems}
\crefname{Corollary}{Corollary}{Corollaries}
\crefname{algorithm}{Algorithm}{Algorithms}

\crefname{section}{Section}{Sections}
\crefname{table}{Table}{Tables}
\crefname{remark}{Remark}{Remarks}
\crefname{definition}{Definition}{Definitions}

\crefname{Proposition}{Proposition}{Propositions}
\crefname{myproblem}{Problem}{Problems}
\crefname{myremark}{Remark}{Remarks}
\crefname{mylemma}{Lemma}{Lemmas}
\crefname{mydefinition}{Definition}{Definitions}
\crefname{myproposition}{Proposition}{Propositions}
\crefname{mycorollary}{Corollary}{Corollaries}
\crefname{myassumption}{Assumption}{Assumptions}
\crefname{mycondition}{Condition}{Conditions}
\crefname{myfact}{Fact}{Facts}
\crefname{figure}{Figure}{Figures}
\crefname{enumi}{}{}
\crefname{name}{}{} %

%% file: 0abstract.tex
Learning a fair predictive model is crucial to mitigate biased decisions against minority groups in high-stakes applications. A common approach to learn such a model involves solving an optimization problem that maximizes the predictive power of the model under an appropriate group fairness constraint. However, in practice, sensitive attributes are often missing or noisy resulting in uncertainty. We demonstrate that solely enforcing fairness constraints on uncertain sensitive attributes can fall significantly short in achieving the level of fairness of models trained without uncertainty. To overcome this limitation, we propose a bootstrap-based algorithm that achieves the target level of fairness despite the uncertainty in sensitive attributes. The algorithm is guided by a Gaussian analysis for the independence notion of fairness where we propose a robust quadratically constrained quadratic problem to ensure a strict fairness guarantee with uncertain sensitive attributes. Our algorithm is applicable to both discrete and continuous sensitive attributes and is effective in real-world classification and regression tasks for various group fairness notions, e.g., independence and separation. 

%% file: 1introduction.tex
\section{Introduction}
\label{sec_intro}
Achieving fairness in predictive modeling, whether in classification or regression tasks, is crucial to avoid discriminatory decisions against marginalized groups. Although various problem formulations exist for ensuring fairness in model training, a widely adopted approach is to formulate an optimization problem that maximizes the model's predictive power while satisfying a group fairness constraint \citep{kamishima2011fairness, zafar2017fairness, agarwal2018reductions, verma2018fairness, golz2019paradoxes, mehrabi2021survey, castelnovo2022clarification}. The notion of group fairness \citep{barocas-hardt-narayanan} stipulates a certain (conditional) independence requirement involving the model prediction and the sensitive attribute. Then, the goal is to minimize the prediction loss while ensuring that the fairness loss, which measures the degree of group unfairness, i.e., the degree of violation of the (conditional) independence requirement, is less than a pre-defined tolerance level $\epsilon$, i.e.,
\begin{align}\label{eq_introFairness}
    \min \text{Prediction Loss} \qtext{s.t.} \text{Fairness Loss} \leq \epsilon.
\end{align}
Typically, it is assumed that the learner has access to true sensitive attributes for every sample in training, but in reality, labeled sensitive attributes are often missing or noisy. For instance, labeling sensitive attributes may require additional annotation of existing datasets for which such labels were not originally collected. Even if available, the sensitive attribute information can be uncertain due to various reasons, such as noisy or unreliable responses from survey participants due to fear of disclosure or discrimination \citep{krumpal2013determinants}. Moreover, privacy and legal regulations often limit the use of labeled sensitive attributes, such as race or gender, which are protected by laws such as the EU's General Data Protection Regulation or California's Consumer Privacy Act. In such cases, privatized sensitive attributes, which are obtained by adding noise, may be the only available option.
\begin{figure}[t]
    \centering
    \includegraphics[width=0.5\linewidth]{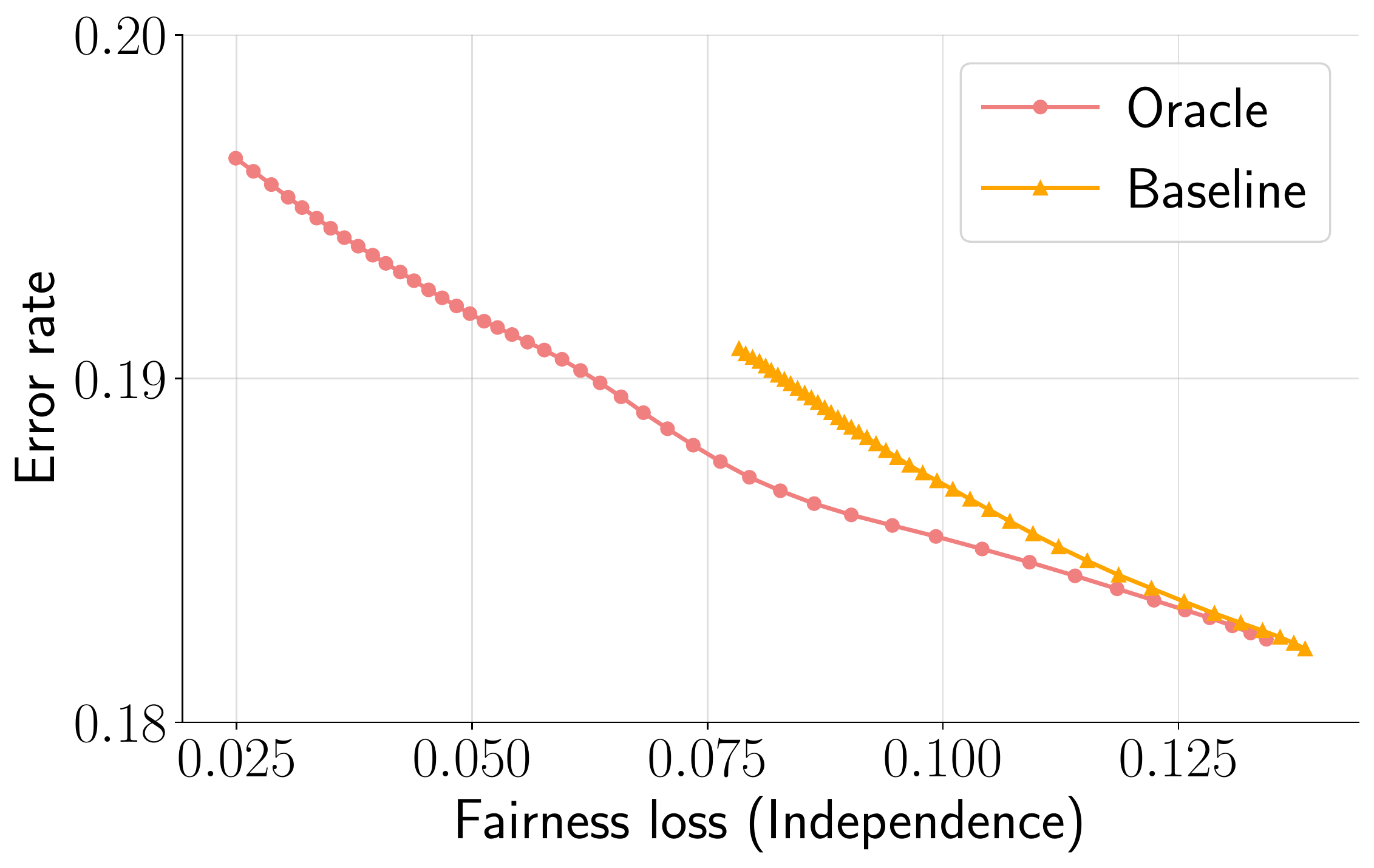}
\caption{Error vs fairness on Adult data for oracle and baseline models that enforce fairness using true and uncertain sensitive attributes, respectively. The baseline falls short of achieving the same range of fairness as the oracle.}
\vspace{-0.2em}
\label{histogram_adult_intro}
\end{figure}
In such scenarios, estimating the fairness loss in \cref{eq_introFairness} using uncertain sensitive attributes, as if correct, can lead to a model that does not accurately capture target fairness. Figure~\ref{histogram_adult_intro} shows trade-off between prediction (measured by error rate) and fairness (measured as violation of independence between predictions and sensitive attributes) obtained by varying $\epsilon$ in \cref{eq_introFairness} for Adult data~\citep{lantz2019machine}. The oracle (in red) has access to true sensitive attributes, denoted by $\cD^{\text{oracle}}$, enforces the fairness constraint: $\text{Fairness Loss} (\cD^{\text{oracle}}) \leq \epsilon$, and covers a wide range of fairness levels. In contrast, the baseline (in orange) has access to a random $<$1\% of the sensitive attributes, denoted by $\cD^{\text{uncertain}}$, enforces the fairness constraint: $\text{Fairness Loss} (\cD^{\text{uncertain}}) \leq \epsilon$, but is unable to achieve fairness below a threshold, i.e., the baseline provides less control over attainable fairness compared to the oracle.

As a result, for high-stakes applications where violating a fairness threshold incurs a significant cost, it is essential to develop a method that can learn fair models despite uncertainty in  sensitive attributes. 

\subsection{Contributions} 
In this work, we propose a solution to fair learning with uncertain sensitive attributes.
\begin{enumerate}[leftmargin=*,topsep=-0.6em,itemsep=-1pt] 
    \item We propose $\bss$, an algorithm that uses a bootstrap approach to impose $\s$ additional constraints to the optimization in \cref{eq_introFairness} for some parameter $\s$. For $i \in [\s]$, constraint $i$ requires $\text{Fairness Loss} (\cD^{\text{uncertain}}_i) \leq \epsilon$, where $\cD^{\text{uncertain}}_i$ is a collection of a fixed number of random subsamples of the uncertain sensitive attributes $\cD^{\text{uncertain}}$, i.e., $\bss$ aims to
    \begin{align}\label{eq_intro_bss}
        \min \text{Prediction Loss} \qtext{s.t.} & \text{Fairness Loss}(\cD^{\text{uncertain}}) \leq \epsilon \stext{and,} \\
        & \text{Fairness Loss}(\cD^{\text{uncertain}}_i) \leq \epsilon \stext{for all} i \in [\s]. \nonumber
    \end{align}
    We illustrate $\bss$ in Figure~\ref{fig_algo_sketch} where it is contrasted with the oracle model that constrains the fairness loss estimated using true sensitive attributes $\cD^{\text{oracle}}$ and the baseline that constrains the fairness loss estimated using available uncertain sensitive attributes $\cD^{\text{uncertain}}$ as if they are correct.
    \item  We analyze fair learning for Gaussian data with a focus on the independence notion of fairness, and this serves as a motivation for $\bss$. We begin by reducing a specific instance of this problem to a quadratically constrained quadratic problem (QCQP) when true sensitive attributes are available. Then, given uncertainty in sensitive attributes, we fully characterize the solution of the QCQP and robustify the QCQP to provide a strict fairness guarantee. Notably, when uncertainty arises due to randomly missing sensitive attributes, in some cases, the robust QCQP can achieve strict fairness without any performance loss, which we refer to as \textit{free fairness}.
    \item We showcase the practical effectiveness of $\bss$ in achieving fairness levels comparable to the oracle while maintaining high prediction performance (unlike the baseline) on a variety of synthetic and real data, including classification and regression tasks, discrete and continuous sensitive attributes, and independence and separation notions of group fairness.
\end{enumerate}
\begin{figure}[t]
\centering
\begin{tikzpicture}[scale=0.75]
   \draw[orange!80, very thick, rounded corners] (-6.0,0) rectangle (-1.5,1.5) node[pos=.5] (ourmain) {\begin{tabular}{c} Fair loss estimated \\  with $\cD^{\text{uncertain}}$ \end{tabular}};
   \draw[orange!70, fill = orange!30, thick, rounded corners] (-6.0,0.0) rectangle (-2.5,-0.5) node[pos=.5,text=black] {$\cD^{\text{uncertain}}$};
   \draw[white] (-1.0,0.0) rectangle (-0.5,1.5) node[pos=.5,text=black,font = {\Large\bfseries\sffamily}] {$+$};
   \draw[blue!80, very thick, rounded corners] (0.0,0.0) rectangle (4.5,1.5) node[pos=.5] (ourfirst) {\begin{tabular}{c} Fair loss estimated \\ with $\cD^{\text{uncertain}}_1$ \end{tabular}};
   \draw[orange!70, fill = orange!20, thick, rounded corners] (0.0,0.0) rectangle (3.5,-0.5) node[pos=.5]  {};
   \draw[blue!70, fill = blue!30, thick, rounded corners] (0.75,0.0) rectangle (3.00,-0.5) node[pos=.5,text=black] {$\cD^{\text{uncertain}}_1$};
   \draw[white, very thick, rounded corners] (4.6,0.0) rectangle (5.9,1.5) node[pos=.5,text=black] (ourinterim) {\begin{tabular}{c} $\cdots$  \end{tabular}};
   \draw[white] (5.9,0.0) rectangle (6.4,1.5) node[pos=.5,text=black,font = {\Large\bfseries\sffamily}] {$+$};
   \draw[blue!80, very thick, rounded corners] (7.0,0.0) rectangle (11.5,1.5) node[pos=.5] (ourlast) {\begin{tabular}{c} Fair loss estimated \\ with $\cD^{\text{uncertain}}_S$ \end{tabular}};
   \draw[orange!70, fill = orange!20, thick, rounded corners] (7.0,0.0) rectangle (10.5,-0.5) node[pos=.5]  {};
   \draw[blue!70, fill = blue!30, thick, rounded corners] (7.0,0.0) rectangle (9.25,-0.5) node[pos=.5,text=black]  {$\cD^{\text{uncertain}}_S$};
   \draw[white] (0.5,-1.1) rectangle (5.0,-0.6) node[pos=.5,text=black]  {\begin{tabular}{c} $(c)$ Our $\bss$ method \end{tabular}};
   \draw[orange!80, very thick, rounded corners] (4.5,3.0) rectangle (9,4.5) node[pos=.5] (baseline) {\begin{tabular}{c} Fair loss estimated \\ with  $\cD^{\text{uncertain}}$ \end{tabular}};
   \draw[orange!70, fill = orange!30, thick, rounded corners] (4.5,2.5) rectangle (8.0,3.0) node[pos=.5,text=black] (basedata) {$\cD^{\text{uncertain}}$};
   \draw[white] (4.5,1.9) rectangle (9.0,2.4) node[pos=.5,text=black]  {\begin{tabular}{c} $(b)$ Baseline method \end{tabular}};
   \draw[red!70, very thick, rounded corners] (-3.5,3.0) rectangle (1.0,4.5) node[pos=.5] (oracle) {\begin{tabular}{c} Fair loss estimated \\ with $\cD^{\text{oracle}}$ \end{tabular}};
   \draw[red!70, fill = red!30, thick, rounded corners] (-3.5,2.5) rectangle (1.0,3.0) node[pos=.5,text=black] (oracledata) {$\cD^{\text{oracle}}$};
   \draw[white] (-3.5,1.9) rectangle (1.0,2.4) node[pos=.5,text=black]  {\begin{tabular}{c} $(a)$ Oracle method\end{tabular}};
\end{tikzpicture}
\caption{Comparing fairness loss estimation methods: (a) Oracle uses true sensitive attributes $\cD^{\text{oracle}}$; (b) Baseline uses available uncertain sensitive attributes $\cD^{\text{uncertain}}$ as if correct; (c) $\bss$ constrains the optimization with additional fairness losses estimated using subsamples $\cD^{\text{uncertain}}_i, \forall i \in [\s]$. 
The horizontal bars illustrate these methods when uncertainty is due to missing sensitive attributes.} 
\label{fig_algo_sketch}
\end{figure}

%% file: 2related_work.tex
\section{Related Work}
\noindent \textbf{Group fairness in machine learning.} 
Various metrics and criteria have been proposed to ensure group fairness in machine learning~\citep{verma2018fairness, mehrabi2021survey, castelnovo2022clarification,shah2022selective}, but many of these criteria are mutually exclusive in non-trivial cases~\citep{golz2019paradoxes}. For example, the independence and the separation criteria cannot both be satisfied simultaneously. Different approaches exist to enforce these criteria, mainly falling into one of three categories: $(a)$ pre-processing methods~\citep{zemel2013learning, calmon2017optimized}, $(b)$ post-processing methods~\citep{hardt2016equality, pleiss2017fairness}, and $(c)$ in-processing methods~\citep{kamishima2011fairness, zafar2017fairness, agarwal2018reductions}. In this work, we consider independence and separation, using an in-processing approach where the objective function accounts for both accuracy and fairness.\\

\noindent \textbf{Fairness without certain sensitive attributes.} 
The growing literature on fairness in the absence of true sensitive attributes can be broadly categorized into the following three groups.\\

\noindent \emph{A. Perturbed sensitive attribute.}
Several approaches have been proposed to handle perturbations in sensitive attributes. For example, in-processing methods for fair classification have been developed by \cite{lamy2019noise} and \cite{celis2021fair} to deal with noisy group labels, while \cite{awasthi2020equalized} investigated the performance of a post-processing algorithm proposed by \cite{hardt2016equality} with noisy sensitive labels. Additionally, \cite{mozannar2020fair} and \cite{celis2021fairclassification} explored achieving group fairness with adversarially perturbed and differentially private data, respectively. \\

\noindent \emph{B. Proxy variables.}
\cite{gupta2018proxy}, \cite{chen2019fairness}, and \cite{kallus2022assessing} proposed methods to achieve group fairness when proxy variables are available as substitutes for the sensitive attribute (e.g., zip code as a proxy of race). However, the effectiveness of these methods may be reduced if the correlation between the sensitive attribute and the proxy variables is weak. \cite{jung2022learning} proposed a semi-supervised learning approach to generate proxy pseudo-labels for partially observed sensitive attributes. While these proxy-based methods can be useful, they risk perpetuating biases.\\

\noindent \emph{C. No sensitive attribute.}
\cite{hashimoto2018fairness} and \cite{lahoti2020fairness} proposed methods for achieving fairness without relying on a labeled sensitive attribute, utilizing distributionally robust optimization to improve the performance of the worst-case risk for all distributions close to the empirical distribution. They aim to achieve Rawlsian max-min fairness, but their notion of fairness is not defined by the population distribution, which sets it apart from our focus on group fairness. Additionally, it may not be straightforward to combine their methods with existing fair-training methods, while our method can be generally applied to any fair-training method.\\

\noindent Our work is most closely related to \cite{wang2020robust}, who also focus on achieving a strict group fairness guarantee given uncertain sensitive attributes. However, they focus on classification problems with discrete sensitive attributes. In contrast, our approach is widely applicable, including both regression and classification, as well as both discrete and continuous sensitive attributes. 

%% file: 3problem_formulation.tex
\section{Problem Formulation}\label{sec_problem}
We consider a scenario where $\rvbx$ represents $d$-dimensional input features defined on the alphabet $\cX$, while $\rvy$ and $\rve$ denote $1$-dimensional target and sensitive attribute defined on the alphabets $\cY$ and $\cE$, respectively. Fair supervised learning seeks to find a predictor $f : \cX \to \cY$ that (a) accurately estimates the target variable for new input features and (b) avoids discrimination based on the sensitive attribute. To achieve this, we are given (a) a loss function $\loss : \cY \times \cY \to \Reals_+$, where $\loss(\rvy, f(\rvbx))$ measures the disagreement between the target variable and its prediction, and (b) a fairness measure $\Phi : \cY \times \cY \times \cE \to \Reals_+$, where $\Phi(\rvy, f(\rvbx), \rve)$ measures the level of discrimination of $f$. Given a fairness target $\epsilon \geq 0$ and a class of predictors $\cF$, the goal of fair learning is to find an $f \in \cF$ that minimizes the expected loss $\loss$, subject to the fairness measure $\Phi$ being small:
\begin{align}\label{eq_mainopt}
    f^* \in \argmin_{f \in \cF} \Expectation\bigbrackets{\loss(\rvy, f(\rvbx))} \qtext{s.t.} \Phi(\rvy, f(\rvbx), \rve) \leq \epsilon.
\end{align}
For ease of notation, hereon, we define $\rvu \defn f(\rvbx)$.\\

\noindent {\bf Choice of the loss function.} 
The choice of loss function $\loss$ depends on the specific alphabet $\cY$. In this work, we focus on regression and binary classification tasks, where $\cY$ is either $\Reals$ or $\normalbraces{0,1}$, respectively. For $\cY = \Reals$, we use the mean squared error (MSE) loss, defined as $\loss(\rvy, \rvu) = \normalparenth{\rvy - \rvu}^2$. For $\cY = \normalbraces{0,1}$, we use the log loss, defined as $\loss(\rvy, \rvu) = - \rvy \log \rvu - \normalparenth{1 - \rvy} \log \normalparenth{1- \rvu}$.\\

\noindent {\bf Choice of the fairness measure.} 
To design a fairness measure $\Phi$, it is important to establish what is meant by a perfectly fair predictor, i.e., $\epsilon = 0$ in \cref{eq_mainopt}. Typically, perfect fairness is described in terms of statistical independence. This work focuses on two commonly used fairness criteria: \textit{independence} and \textit{separation}. The independence criterion, also called demographic parity, demands that $\rvu \indep \rve$, meaning that predictions should not reveal any information about sensitive attributes. The separation criterion, also known as equalized odds, requires that $\rvu \indep \rve | \rvy$, indicating that predictions should not disclose any information about sensitive attributes given the knowledge of true target variables. 

Achieving perfect fairness is not feasible when learning a predictor from finite training samples \citep{agarwal2018reductions}. Instead, in practice, one often works with measures of approximate fairness. This is done by choosing $\epsilon > 0$ in \cref{eq_mainopt}, and then varying $\epsilon$ to find a balance between fairness and accuracy. As perfect fairness measures assert that certain random variables should be independent, a natural way to measure approximate fairness is to use divergence that measures the degree of independence between these variables. In recent years, $\chi^2$-divergence has emerged as an effective measure of approximate fairness \citep{mary2019fairness}. Following this, we adopt $\chi^2$-divergence as our measure of the degree of independence, except in cases where the data is Gaussian, where we use a different analytically convenient divergence (introduced later). For independence, the measure is given by $\Phi(\rvy, \rvu, \rve) = \chidiv{p_{\rve, \rvu}}{p_{\rve}p_{\rve}}$ where $p_{\rve, \rvu}, p_{\rve}$, and $p_{\rvu}$ are marginal distributions of $(\rve, \rvu)$, $\rve$, and $\rvu$, respectively. Likewise, for  separation $\Phi(\rvy, \rvu, \rve) = \Expectation_{p_{\rvy}}\normalbrackets{\chidiv{p_{\rve, \rvu | \rvy}}{p_{\rve| \rvy}p_{\rve| \rvy}}}$ where $p_{\rve, \rvu| \rvy}, p_{\rve| \rvy}$, and $p_{\rvu| \rvy}$ are conditional distributions of $(\rve, \rvu)$, $\rve$, and $\rvu$ given $\rvy$, respectively.

\subsection{Uncertain sensitive attributes} 
Typically, $N$ independent and identically distributed (i.i.d.) samples of the tuple $(\rvbx, \rvy, \rve)$ are available, denoted by $\cDo \defn \normalbraces{\rvbx^{(i)}, \rvy^{(i)}, \rve^{(i)}}_{i \in [N]}$. Then, in the optimization in \cref{eq_mainopt}, the objective is estimated using the subset $\cDp \defn \normalbraces{\rvbx^{(i)}, \rvy^{(i)}}_{i \in [N]}$ while the constraint is estimated using an appropriate subset of $\cDo$ depending on the functional form $\Phi$. We denote these estimates by $\Expectation_{\cDp}\bigbrackets{\loss(\rvy, f(\rvbx))}$ and $\Phi_{\cDo}(\rvy, f(\rvbx), \rve)$, respectively, for brevity. We assume that $N$ is sufficiently large and ignore any errors in these estimates to focus on errors due to uncertainty in sensitive attributes.

When dealing with uncertain sensitive attributes, access to $\cDo$ may not be possible. To account for such uncertainty, we assume access to $\cDp$ as well as $n \leq N$ (potentially noisy) labeled sensitive attributes $\cDu \defn \normalbraces{\rvbx^{(i)}, \rvy^{(i)}, \hrve^{(i)}}_{i \in [n]}$. For $ i \in [N]$, if $\hrve^{(i)} \neq \rve^{(i)}$, then sensitive attribute $\hrve^{(i)}$ is noisy. Further, if $n < N$, then sensitive attributes $\normalbraces{\rve^{(i)}}_{i = n+1}^{N}$ are missing. Then, the goal of fair learning with uncertain sensitive attributes is to solve the optimization in \cref{eq_mainopt} with access to $\cDp$ and $\cDu$. While this is an intuitively appealing goal, simply computing the constraint in \cref{eq_mainopt} with $\cDu$ may be sub-optimal as discussed in \cref{sec_intro}. In other words, a predictor $\rvu$ satisfying $\Phi_{\cDu}(\rvy, \rvu, \rve) \leq \epsilon$ may not necessarily satisfy $\Phi_{\cDo}(\rvy, \rvu, \rve) \leq \epsilon$. To address this issue and gain some insight, we first consider the case where $(\rvbx, \rvy, \rve, \rvu)$ is jointly Gaussian. Under this scenario, we fully characterize the optimization problem in \cref{eq_mainopt} and guarantee strict fairness despite uncertainty in sensitive attributes. Then, building on this analysis, we develop our general-purpose algorithm.

\subsection{Gaussian setting} 
For the ease of exposition, we consider zero-mean Gaussian variables, and assume that the marginal distribution $p_{\rvbx, \rvy}$ is known or can be learned from $\cDp$. {We think of $\rvu$ as a representation of the features and let the predictor be $\Expectation[\rvy|\rvu]$.} Naturally, the loss function $\loss$ is chosen to be the mean squared loss as $\cY = \Reals$. We focus on the independence criterion of fairness and measure the degree of independence between $\rvu$ and $\rve$ using the notion of $\Divergence$-divergence, a second-order approximation of Kullback--Leibler divergence, introduced by \cite{huang2019universal}. 

\begin{definition}[$\Divergence$-divergence]\label{def_div_inf_meausres}
The $\Divergence$-divergence between zero-mean Gaussian random vectors $\rvbv \sim p_{\rvbv} = \cN(\tbf{0}, \bsigma_v)$ and $\rvbw \sim p_{\rvbw} = \cN(\tbf{0}, \bsigma_w)$, with $\fronorm{\cdot}$ denoting the Frobenius norm, is given by
\begin{align}
    \Divergence(p_{\rvbv} \| p_{\rvbw}) \defn \frac{1}{2} \fronorm{\bsigma_w^{-1/2} \normalparenth{\bsigma_v - \bsigma_w} \bsigma_w^{-1/2}}^2.
\end{align}
\end{definition}
\noindent For these choices, the optimization in \cref{eq_mainopt} reduces to learning a Gaussian variable $\rvu$ such that
\begin{align}\label{eq_gaussian_opt}
    \rvu^* \in \argmin_{\rvu} \Expectation\bigbrackets{(\rvy - {\Expectation[\rvy|\rvu]})^2)} \qtext{s.t.} \Divergence(p_{\rve, \rvu} \| p_{\rve} p_{\rvu}) \leq \epsilon.
\end{align}
\noindent Next, we reformulate \cref{eq_gaussian_opt} into a quadratically constrained quadratic program (QCQP) by utilizing the notion of canonical correlation matrices (CCMs) defined by \cite{huang2019universal}.
\begin{definition}[Canonical correlation matrix]\label{def_ccm}
The canonical correlation matrix (CCM) between zero-mean jointly Gaussian random vectors $\rvbv \sim \cN(\tbf{0}, \bsigma_v)$ and $\rvbw \sim \cN(\tbf{0}, \bsigma_w)$ is given by $\tbf{b}_{vw} \defn \bsigma_{vv}^{-1/2}\bsigma_{vw}\bsigma_{ww}^{-1/2}$, where $\bsigma_{vw}$ is the cross-covariance matrix between $\rvbv$ and $\rvbw$.
\end{definition}
\noindent The $\Divergence$-divergence is conveniently represented by these CCMs. We now formally state the equivalence between \cref{eq_gaussian_opt} and a QCQP that uses CCMs. We prove this in \cref{proof_them_ibqcqp} by drawing connections to the information bottleneck principle \citep{bu2021sdp}.
\newcommand{\ibeqqcqp}{Gaussian Fair Learning $\iff$ QCQP}
\begin{theorem}[\ibeqqcqp]\label{thm_ibqcqp}
The optimization problem in \cref{eq_gaussian_opt} is equivalent to
\begin{align}
    \max_{\rvba \in \ball} \biginner{\rvba}{\Byx}^2 \qtext{s.t.}  \biginner{\rvba}{\Bex}^2 \leq \varepsilon,\label{eq_qcqp}
\end{align}
where $\ball$ denotes an $\ell_2$ ball centered at 0 with radius 1, $\normalinner{\cdot}{\!\cdot}$ denotes the inner product, and $\rvba$ plays the role of $\Bux$.
\end{theorem}
\noindent We note that $\rvba$ in \cref{eq_qcqp} has the same dimension as $\rvbx$, i.e., $d$. The following result, with a proof in \cref{subsec_proof_thm_2d}, shows that any $d$-dimensional QCQP in \cref{eq_qcqp} can be mapped to a 2-dimensional QCQP. 
\begin{proposition}[\qcqpsubspace]\label{thm_2d}
The optimal solution $\rvba^\star$ of the QCQP in \cref{eq_qcqp} lies in the subspace spanned by the vectors $\Byx$ and $\Bex$. 
\end{proposition}
\noindent In \cref{subsec_characterizing_qcqp}, we characterize the optimal $\rvba^\star$ in \cref{eq_qcqp} as a function of $\Byx$, $\Bex$, and $\varepsilon$ for $d = 2$. \cref{thm_ibqcqp} demonstrates that considering the uncertainty in the canonical correlation matrix $b_{ex}$ is sufficient to capture the uncertainty in sensitive attributes, which we will now explore in detail.

%% file: 4main_results.tex
\section{Theoretical Results}
\label{sec_main_results}
In this section, we provide a characterization of fair learning for Gaussian data given some uncertainty in sensitive attributes. Specifically, we study how to \emph{robustify} the QCQP in \cref{eq_qcqp} to ensure strict fairness guarantee with high probability, as well as how this robustification affects the optimal objective. 

Let $\BexEstimated$ be an estimate of $\Bex$, say  obtained from  $\cDu$, such that $\stwonorm{\Bex - \BexEstimated} \leq \tau$ (with probability $1 - \delta$), for some $\tau \geq 0$,\footnote{For ease of the exposition, we assume $\radius \leq \stwonorm{\Bex}$.} and denoted by $\Bex \in \ballunc$. To achieve fairness as in \cref{eq_qcqp} with probability $1-\delta$, in the worst case, $\normalinner{\rvba}{\B}^2 \leq \varepsilon$ should hold for all $\B \in \ballunc$. Then, the following robust optimization maximizes the desired objective while achieving fairness as in \cref{eq_qcqp} (with probability $1-\delta$) without the precise knowledge of $\Bex$:
\begin{align}
\max_{\rvba \in \ball}\! \normalinner{\rvba}{\Byx}\hspace{-0.0mm}^2 \qtext{s.t.}  \normalinner{\rvba}{\B}\hspace{-0.0mm}^2 \! \leq \varepsilon, ~\forall \B \in \ballunc. \label{eq_qcqp_limited_infinite_worse_ball}   
\end{align}
In the following proposition, we show that any $d$-dimensional robust QCQP in \cref{eq_qcqp_limited_infinite_worse_ball} can be mapped to a 2-dimensional QCQP. See \cref{subsec_proof_thm_2d_robust} for a proof.
\newcommand{\robustqcqpsubspace}{$d = 2$ suffices for robust QCQP}
\begin{proposition}[\robustqcqpsubspace]\label{thm_2d_robust}
The optimal solution $\truervba$ of the robust QCQP in \cref{eq_qcqp_limited_infinite_worse_ball} lies in the subspace spanned by the vectors $\Byx$ and $\BexEstimated$. 
\end{proposition}
\begin{figure*}[t]
\centering
\begin{tabular}{cc}
\adjustbox{valign=t}
{\begin{tikzpicture}[scale=2.0, declare function={axis = 1; R = 1; phi = 15; alpha = 60; remin = 1.4; remax = 1.8;},
dot/.style={circle,fill,inner sep=1.5pt},
]
\centering
\begin{scope}[nodes={dot}]
\draw[name path=circ]   (0,0) coordinate (origin) circle[radius=R]; 
\draw[name path=yaxis, black] (0,-axis) -- (0,axis);
\draw[name path=xaxis, black] (-axis,0) -- (axis,0) coordinate (xplus);
\draw[name path=ori_to_eminus_upper, thick, orange, dotted] (phi:remin) node[label={[label distance=-16mm]290:\rotatebox{40}{\normalsize$(\hat{r}_e \!-\! \Delta, \phi)$}}] (eminusupper){}  -- (origin);
\draw[name path=ori_to_eminus_lower, thick, orange, dotted] (-phi:remin) node[label={[label distance=-10mm]270:\rotatebox{40}{\normalsize$(\hat{r}_e \!-\! \Delta, \!-\phi)$}}] (eminuslower){}  -- (origin);
\draw[name path=ori_to_eplus_upper,  thick, densely dashed, orange] (phi:remax) node[blue, label={[label distance=-16mm]290:\rotatebox{40}{\normalsize$(\hat{r}_e \!+\! \Delta, \phi)$}}](eplusupper){} -- (eminusupper);
\draw[name path=ori_to_eplus_lower,  thick, densely dashed, orange] (-phi:remax) node[green, label={[label distance=-10mm]270:\rotatebox{40}{\normalsize$(\hat{r}_e \!+\! \Delta, \!-\phi)$}}] (epluslower){}  -- (eminuslower);
\node[magenta] (thirdnode) at (0:2.0){};
\foreach \phii in {0,5,...,15}
{
    \draw[very thick, red] (alpha+\phii:R)  -- (-alpha+\phii:R);
    \draw[very thick, red] (180-alpha+\phii:R) -- (180+alpha+\phii:R);
    \draw[very thick, red] (origin) ++(alpha+\phii:R) arc (alpha+\phii:180-alpha+\phii:R);
    \draw[very thick, red] (origin) ++(180+alpha+\phii:R) arc (180+alpha+\phii:360-alpha+\phii:R);
}
\foreach \phii in {0,-5,...,-15}
{
    \draw[very thick, red] (alpha+\phii:R)  -- (-alpha+\phii:R);
    \draw[very thick, red] (180-alpha+\phii:R) -- (180+alpha+\phii:R);
    \draw[very thick, red] (origin) ++(alpha+\phii:R) arc (alpha+\phii:180-alpha+\phii:R);
    \draw[very thick, red] (origin) ++(180+alpha+\phii:R) arc (180+alpha+\phii:360-alpha+\phii:R);
}
\draw[ultra thick, densely dashed, orange] (origin) ++(phi:remin) arc (phi:-phi:remin);
\draw[ultra thick, orange] (origin) ++(phi:remax) arc (phi:-phi:remax);
\end{scope}
\begin{scope}[on background layer]
\draw[fill=red,opacity=.0,fill opacity=.2] (0,0) -- (phi:{cos{alpha}}) -- (phi+alpha:1) -- cycle;
\draw[fill=red,opacity=.0,fill opacity=.2] (0,0) -- (180-phi:{cos{alpha}}) -- (180-phi-alpha:1) -- cycle;
\draw[fill=red,opacity=.0,fill opacity=.2] (0,0) -- (-phi:{cos{alpha}}) -- (-phi-alpha:1) -- cycle;
\draw[fill=red,opacity=.0,fill opacity=.2] (0,0) -- (180+phi:{cos{alpha}}) -- (180+phi+alpha:1) -- cycle;
\draw[fill=red,opacity=.0,fill opacity=.2] (0,0) -- (phi+alpha:R) -- (180-phi-alpha:R) -- cycle;
\draw[fill=red,opacity=.0,fill opacity=.2]  ++(phi+alpha:R) arc (phi+alpha:180-phi-alpha:R);
\draw[fill=red,opacity=.0,fill opacity=.2] (0,0) -- (-phi-alpha:R) -- (180+phi+alpha:R) -- cycle;
\draw[fill=red,opacity=.0,fill opacity=.2]  ++(360-phi-alpha:R) arc (360-phi-alpha:180+phi+alpha:R);
\draw[fill=red,opacity=.0,fill opacity=.2] (0,0) -- (phi:{cos{alpha}}) -- (-phi:{cos{alpha}}) -- cycle;
\draw[fill=red,opacity=.0,fill opacity=.2]  ++(phi:{cos{alpha}}) arc (phi:-phi:{cos{alpha}});
\draw[fill=red,opacity=.0,fill opacity=.2] (0,0) -- (180-phi:{cos{alpha}}) -- (180+phi:{cos{alpha}}) -- cycle;
\draw[fill=red,opacity=.0,fill opacity=.2]  ++(180-phi:{cos{alpha}}) arc (180-phi:180+phi:{cos{alpha}});
\end{scope}
\end{tikzpicture}}
&
\adjustbox{valign=t}{\begin{tikzpicture}[scale=1.9, declare function={axis = 1; R = 1; phi = 30; alpha = 63; remin = 1.2; remax = 2.0;yintercept = 0.9;},
dot/.style={circle,fill,inner sep=1.5pt},
]
\centering
\begin{scope}[nodes={dot}]
\draw[name path=circ]   (0,0) coordinate (origin) circle[radius=R]; 
\draw[name path=yaxis, black] (0,-axis) -- (0,axis);
\draw[name path=xaxis, black] (-axis,0) -- (axis,0) coordinate (xplus);
\draw[name path=ori_to_eminus_upper, orange, thick, dotted] (phi:remin) node[label={[label distance=-15mm]90:\rotatebox{30}{\normalsize$(\hat{r}_e \!-\! \Delta, \phi)$}}] (eminusupper){}  -- (origin);
\draw[name path=ori_to_eminus_lower, orange, thick, dotted] (-phi:remin) node[label={[label distance=-16mm]90:\rotatebox{30}{\normalsize$(\hat{r}_e \!-\! \Delta, \!-\phi)$}}] (eminuslower){}  -- (origin);
\draw[name path=ori_to_eplus_upper, thick, orange, densely dashed] (phi:remax) node[blue, label={[label distance=-16mm]90:\rotatebox{30}{\normalsize$(\hat{r}_e \!+\! \Delta, \phi)$}}](eplusupper){} -- (eminusupper);
\draw[name path=ori_to_eplus_lower, thick, orange, densely dashed] (-phi:remax) node[green, label={[label distance=-17.5mm]270:\rotatebox{30}{\normalsize$(\hat{r}_e \!+\! \Delta, \!-\phi)$}}] (epluslower){}  -- (eminuslower);
\node[magenta] (thirdnode) at (0:2.2){};
\foreach \phii in {0,5,...,30}
{
    \draw[very thick, red] (alpha+\phii:R)  -- (-alpha+\phii:R);
    \draw[very thick, red] (180-alpha+\phii:R) -- (180+alpha+\phii:R);
    \draw[very thick, red] (origin) ++(alpha+\phii:R) arc (alpha+\phii:180-alpha+\phii:R);
    \draw[very thick, red] (origin) ++(180+alpha+\phii:R) arc (180+alpha+\phii:360-alpha+\phii:R);
}
\foreach \phii in {0,-5,...,-30}
{
    \draw[very thick, red] (alpha+\phii:R)  -- (-alpha+\phii:R);
    \draw[very thick, red] (180-alpha+\phii:R) -- (180+alpha+\phii:R);
    \draw[very thick, red] (origin) ++(alpha+\phii:R) arc (alpha+\phii:180-alpha+\phii:R);
    \draw[very thick, red] (origin) ++(180+alpha+\phii:R) arc (180+alpha+\phii:360-alpha+\phii:R);
}
\draw[ultra thick, densely dashed, orange] (origin) ++(phi:remin) arc (phi:-phi:remin);
\draw[ultra thick, orange] (origin) ++(phi:remax) arc (phi:-phi:remax);
\end{scope}
\begin{scope}[on background layer]
\draw[fill=red,opacity=.0,fill opacity=.2] (0,0) -- (phi:{cos{alpha}}) -- (0,yintercept) -- cycle;
\draw[fill=red,opacity=.0,fill opacity=.2] (0,0) -- (180-phi:{cos{alpha}}) -- (0,yintercept) -- cycle;
\draw[fill=red,opacity=.0,fill opacity=.2] (0,0) -- (-phi:{cos{alpha}}) -- (0,-yintercept) -- cycle;
\draw[fill=red,opacity=.0,fill opacity=.2] (0,0) -- (180+phi:{cos{alpha}}) -- (0,-yintercept) -- cycle;
\draw[fill=red,opacity=.0,fill opacity=.2] (0,0) -- (phi:{cos{alpha}}) -- (-phi:{cos{alpha}}) -- cycle;
\draw[fill=red,opacity=.0,fill opacity=.2]  ++(phi:{cos{alpha}}) arc (phi:-phi:{cos{alpha}});
\draw[fill=red,opacity=.0,fill opacity=.2] (0,0) -- (180-phi:{cos{alpha}}) -- (180+phi:{cos{alpha}}) -- cycle;
\draw[fill=red,opacity=.0,fill opacity=.2]  ++(180-phi:{cos{alpha}}) arc (180-phi:180+phi:{cos{alpha}});
\end{scope}
\end{tikzpicture}}
\\
(a) Feasible space when $\sqrt{\varepsilon} > (\hre+ \Delta) \sin{\phi}$.
&
(b) Feasible space when $\sqrt{\varepsilon} < (\hre+ \Delta) \sin{\phi}$.
\end{tabular}
\caption{Visualizing the feasible space of the robust QCQP in \cref{thm_ibqcqp_limited_infinite}, i.e., \cref{eq_qcqp_limited_infinite}, for $\varepsilon = 0.9$, $\hre = 1.6$, and $\hthetae = 0$. We set $\Delta = 0.2$ and $\phi = \pi/12$ for panel $(a)$, and $\Delta = 0.4$ and $\phi = \pi/6$ for panel $(b)$. Each point is shown in polar coordinates, i.e., a point $(r,\theta)$ denotes $(r\cos{\theta}, r\sin{\theta})$. The annular sector $\cS$ is the region enclosed by dashed lines, dashed arc, and solid arc in orange. The arc $\bcS$ is the solid arc in orange. The shaded region is the feasible space. The points $\Bex^{(1)}$, $\Bex^{(2)}$, and $\Bex^{(3)}$ from \cref{prop_ibqcqp_limited} are shown in magenta, blue and green, respectively.}
\label{fig_proofs_1}
\vspace{-5mm}
\end{figure*}
Now, to characterize the solution of the robust QCQP in \cref{eq_qcqp_limited_infinite_worse_ball}, we focus on $d = 2$ and work with polar coordinates. Further, to analyze the corresponding feasible space, we relax the uncertainty space $\ballunc$ from a ball to an annular sector. Formally, we let $\BexEstimated \defn \polar{\hre}{\hthetae}$ be the estimate of $\Bex$ such that $\normalabs{\re - \hre} \leq \Delta$ and $\normalabs{\thetae - \hthetae} \leq \phiu$ with probability $1-\delta$ where $\Deltal \defn \radius \geq 0$ and $\phil \defn \sin^{-1}(\radius / \BexEstimated) \in [0,\pi/2]$. In other words, given $\hre, \hthetae, \Delta$, and $\phi$, with probability $1-\delta$,
\begin{align}
    \Bex \in \cS \defn \normalbraces{\B = \polar{r}{\theta}\colon |r - \hre| \leq  \Deltau \stext{and} |\theta - \hthetae| \leq  \phiu},  \label{eq_angular_sector}
\end{align}
i.e., $\cS (\supset \ballunc)$ denotes the smallest annular sector around $\BexEstimated$ capturing our uncertainty in knowing $\Bex$ (see \cref{fig_proofs_1} where $\cS$ is the shown in orange). Now, to achieve fairness as in \cref{eq_qcqp} (with probability $1-\delta$), we constrain the robust QCQP in \cref{eq_qcqp_limited_infinite_worse_ball} as follows:
\begin{align}
\max_{\rvba \in \ball}\! \normalinner{\rvba}{\Byx}^2 \qtext{s.t.}  \biginner{\rvba}{\B}^2 \! \leq \varepsilon,  ~\forall \B \!\in\! \cS. \label{eq_qcqp_limited_infinite_worse}    
\end{align}
As we show below, the constraint in \cref{eq_qcqp_limited_infinite_worse} is equivalent to ensuring $\normalinner{\rvba}{\B}^2 \leq \varepsilon$ for all $\B \in \bcS$ where $\bcS$ is the arc on the boundary of the angular sector $\cS$ with maximum radius (shown in solid orange in \cref{fig_proofs_1}). See \cref{sec_proof_qcqpinf} for a proof. Further, in \cref{subsec_characterizing_qcqp_limited_infinite}, we characterize the optimal $\rvba$ in \cref{eq_qcqp_limited_infinite_worse} as a function of $\BexEstimated$, $\Byx$, $\Delta$, $\phi$, and $\varepsilon$.
\newcommand{\qcqpinf}{Robust QCQP with infinite constraints}
\begin{theorem}[\qcqpinf]\label{thm_ibqcqp_limited_infinite}
Let $\bcS \defn \normalbraces{\B: \B = \polar{(\hre+\Deltau)}{\theta}$ $\stext{and} |\theta - \hthetae| \leq  \phiu}$ be the arc on the boundary of $\cS$ with maximum radius. Then, the robust QCQP in \cref{eq_qcqp_limited_infinite_worse} is equivalent to
\begin{align}
\max_{\rvba \in \ball}\! \normalinner{\rvba}{\Byx}^2 \qtext{s.t.}  \biginner{\rvba}{\B}^2 \! \leq \varepsilon, ~\forall \B \!\in\! \bcS. \label{eq_qcqp_limited_infinite}  
\end{align}
\end{theorem}
\noindent There is a phase transition in the nature of the feasible space of the robust QCQP in \cref{eq_qcqp_limited_infinite}. \cref{fig_proofs_1}(a) illustrates the space if $\sqrt{\varepsilon} \!\geq\! (\hre+ \Delta) \sin{\phi}$, and \cref{fig_proofs_1}(b) illustrates the space if $\sqrt{\varepsilon} \!\leq\! (\hre+ \Delta) \sin{\phi}$.

While \cref{thm_ibqcqp_limited_infinite} simplifies the optimization in \cref{eq_qcqp_limited_infinite_worse}, the resulting optimization in \cref{eq_qcqp_limited_infinite} still has infinite constraints. Below, we provide an approximation to the feasible space in \cref{eq_qcqp_limited_infinite} such that it has finitely many constraints. See \cref{subsec_proof_lemma_ibqcqp_approx} for a proof. We note that alternative approximations are possible.
\newcommand{\approxqcqp}{Robust QCQP with 3 constraints}
\begin{theorem}[\approxqcqp]\label{prop_ibqcqp_limited}
Let $\Bex^{(1)}=$ $\polar{\frac{(\hre+\Deltau)}{\cos{\phi}}}{\hthetae},$ $\Bex^{(2)}= $ $\polar{(\hre+\Deltau)}{(\hthetae+\phiu)}, \stext{and} \Bex^{(3)} = \polar{(\hre+\Deltau)}{(\hthetae-\phil)}$. Then, the feasible space of the optimization below is a subset of the feasible space of the optimization in \cref{eq_qcqp_limited_infinite}:
\begin{align}
    \max_{\rvba \in \ball} \biginner{\rvba}{\Byx}^2 \qtext{s.t}  \biginner{\rvba}{\Bex^{(i)}}^2 \leq \varepsilon \stext{ for all} i \in [3]. \label{eq_qcqp_limited_approx}
\end{align}
\end{theorem}
\noindent See \cref{subsec_characterizing_qcqp_limited_approx} for a characterization of the optimal $\rvba$ in \cref{eq_qcqp_limited_approx} as a function of $\BexEstimated$, $\Byx$, $\Delta$, $\phi$, and $\varepsilon$. We visualize $\Bex^{(1)}$, $\Bex^{(2)}$, and $\Bex^{(3)}$ in \cref{fig_proofs_1} (in magenta, blue, and green, respectively), and note that $\Bex^{(1)}$ approximates the effect of the points in-between $\Bex^{(2)}$ and $\Bex^{(3)}$ on $\cS$. 

\subsection{Uncertainty due to randomly missing sensitive attributes}
Now, we focus on understanding how the optimal objective of the robust QCQP in \cref{eq_qcqp_limited_approx} changes when the uncertainty set $\cS$ changes. For concreteness, we consider the case where uncertainty only stems from sensitive attributes missing at random, and the uncertainty can be improved by collecting more labeled sensitive attributes. Then, we analyze the power of each new labeled sensitive attribute by characterizing the exact difference in the optimal objectives of the  QCQP in \cref{eq_qcqp} and the robust QCQP in \cref{eq_qcqp_limited_approx}. To enable this, we show (in \cref{sec_mono_perf}) that the optimal objective in \cref{eq_qcqp_limited_approx} either monotonically increases or coincides with the optimal objective in \cref{eq_qcqp} whenever the uncertainty set monotonically decreases with $n$.  
 
The response of our algorithm to collecting more labeled sensitive attributes can be classified into 3 categories (see the formal result in \cref{subsec_power_each_sample}): $(a)$ \textit{Any uncertainty hurts:} Here, the optimal objective of the robust QCQP in \cref{eq_qcqp_limited_approx} matches the optimal objective in \cref{eq_qcqp} only when all the uncertainty is removed, i.e., $n \to N$; $(b)$ \textit{Some uncertainty does not hurt:}  Here, the optimal objective of the robust QCQP in \cref{eq_qcqp_limited_approx} matches  the optimal objective in \cref{eq_qcqp} when some uncertainty is removed, i.e., by collecting few extra labeled sensitive attributes; $(c)$ \textit{Uncertainty does not hurt:} Here, the optimal objective of the robust QCQP in \cref{eq_qcqp_limited_approx} matches  the optimal objective in \cref{eq_qcqp} without collecting any extra labeled sensitive attributes. We state this informally below (see \cref{lemma_power_new_sample} for a formal statement).
\begin{corollary}[Free fairness]
\label{coro_free_fairness}
There exist problem instances of the robust QCQP in \eqref{eq_qcqp_limited_approx} where the uncertainty incurs no performance loss while achieving a strict fairness guarantee, without requiring additional labeled sensitive attributes. 
\end{corollary}

Solving the robust QCQP in \cref{eq_qcqp_limited_approx} for $d = 2$ is straightforward using a standard convex optimization solver. However, as the dimensionality of $\rvbx$ increases, solving the problem becomes more challenging. Furthermore, real datasets may not be well-modeled by Gaussian distributions, and thus, the proposed robust QCQP may not be directly deployable. 

%% file: 5algorithm.tex
\section{$\bss$: A General-Purpose Algorithm}
\label{sec_algorithm}
In this section, we leverage our theoretical analysis to propose a generic algorithm that handles high-dimensional features and non-Gaussian data while accounting for uncertainty. The core idea of the robust QCQP (in \cref{prop_ibqcqp_limited}) is to construct an uncertainty set around the estimated canonical correlation matrix $\BexEstimated$ by imposing additional constraints. This approach effectively addresses the unknown nature of the true $\Bex$. Another perspective is to view the robust QCQP in \cref{eq_qcqp_limited_approx} as
\begin{align}
    \max_{\rvba \in \ball} \biginner{\rvba}{\Byx}^2 \qtext{s.t}  \biginner{\rvba}{\BexEstimated}^2 \leq \varepsilon \qtext{and} \biginner{\rvba}{\Bex^{(i)}}^2 \leq \varepsilon \stext{ for all} i \in [3], \label{eq_qcqp_limited_approx_algo}
\end{align}
where the constraint $\biginner{\rvba}{\BexEstimated}^2 \leq \varepsilon$ becomes redundant in the presence of the constraint $\biginner{\rvba}{\Bex^{(1)}}^2 \leq \varepsilon$, and $\normalbraces{\Bex^{(i)}}_{i \in [3]}$ can be viewed as multiple estimates of $\BexEstimated$. For general non-Gaussian data, we use a similar idea but in a non-parametric fashion following the 
\emph{bootstrap} procedure~\citep{efron1992}. 

Specifically, given uncertain sensitive attribute data $\cDu = \normalbraces{\rvbx^{(i)}, \rvy^{(i)}, \hrve^{(i)}}_{i \in [n]}$, a fairness measure $\Phi(\rvy, \rvu, \rve)$, and a parameter $\s$: we draw uniformly $\s$ subsets $\cDi[1],\ldots,\cDi[\s]$ of some size $k \in [n]$ from $\cDu$ at random with replacement. Then, we estimate the fairness measure using each of these subsets as well as $\cDu$, and impose the collection of $\s$ constraints $\{\Phi_{\cDi[i]}(\rvy, \rvu, \rve) \leq \epsilon\}_{i \in [\s]}$ together with the constraint $\Phi_{\cDu}(\rvy, \rvu, \rve) \leq \epsilon$. In summary, we aim to solve the following optimization:
\begin{align}\label{eq_alg_bss}
         \min_{\rvu} \Expectation_{\cDp}\bigbrackets{\loss(\rvy, \rvu)} \qtext{s.t} \Phi_{\cDu}(\rvy, \rvu, \rve)  \leq \epsilon \stext{and} \Phi_{\cDi[i]}(\rvy, \rvu, \rve) \leq \epsilon \stext{for all} i \in [\s].
\end{align}
At a high level, the idea is similar to bootstrap confidence intervals \citep{wasserman2006} allowing construction of better uncertainty set with a larger number of subsamples $\s$.
\begin{algorithm}[t]
\KwInput{loss function $\loss$, fairness measure $\Phi$, dataset $\cDp = \normalbraces{\rvbx^{(i)}, \rvy^{(i)}}_{i \in [N]}$, dataset $\cDu \defn \normalbraces{\rvbx^{(i)}, \rvy^{(i)}, \hrve^{(i)}}_{i \in [n]}$, subsample size $k \in [n]$, number of subsamples $\s \ge 1$}
\For{$i = 1,\cdots,\s$} 
{
Draw a subsample $\cDi[i]$ of size $k$ from $\cDu$ at random with replacement.
}
Then, solve the optimization problem in \cref{eq_alg_lag}.
\caption{\bss}
\label{alg_bootstrap}
\end{algorithm}
Notice that \cref{eq_alg_bss} is a constrained optimization problem, which is non-trivial to solve in practice, especially for neural network training. Typically, this problem is addressed by simply adding the fairness constraints as regularizers with hyperparameters to control the trade-off during optimization, i.e., $ \min \text{Prediction Loss} + \lambda \times  \text{Fairness Loss}$. However, the performance can be sub-optimal as it depends on the choice of $\lambda$. Instead, we follow the approach of  \citet{lee2019learning} by considering the Lagrangian dual of \cref{eq_alg_bss} and optimizing the resulting objective over the duality variables, i.e.,
\begin{equation}\label{eq_alg_lag}
        \min_{\rvu} \max_{\lambda, \lambda_1, \cdots, \lambda_{\s}} \Expectation_{\cDp}\bigbrackets{\loss(\rvy, \rvu)} + \lambda\bigparenth{\Phi_{\cDu}(\rvy, \rvu, \rve) - \epsilon} + \sum_{i \in [\s]} \lambda_i\bigparenth{\Phi_{\cDi[i]}(\rvy, \rvu, \rve) - \epsilon}
\end{equation}
We summarize our method, referred to as \texttt{Bootstrap-S}, in Algorithm \ref{alg_bootstrap}.

%% file: 6experiments.tex
\section{Empirical Evaluation}
\label{sec_exp}
In Section~\ref{subsec_gaussian_expts}, we show the efficacy of the robust QCQP in achieving strict fairness on synthetic Gaussian data and demonstrate that $\bss$ serves as a good approximation for the robust QCQP. In Section~\ref{subsec_real_expts}, we show the performance of $\bss$ on various real-world datasets.

\begin{figure*}[t]
    \centering
    \begin{tabular}{ccc}
    \includegraphics[width=0.3\linewidth,clip]{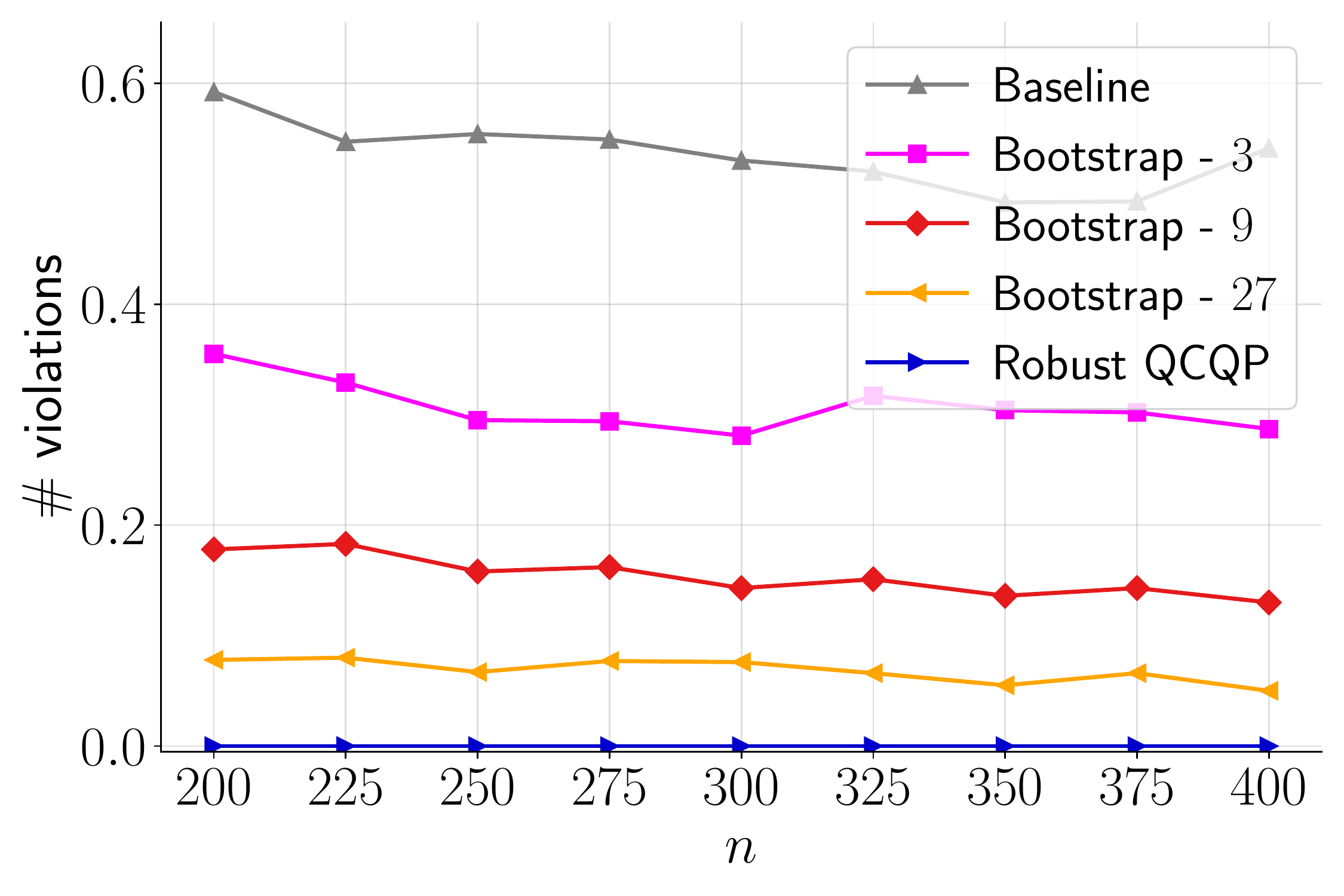}&
    \includegraphics[width=0.3\linewidth,clip]{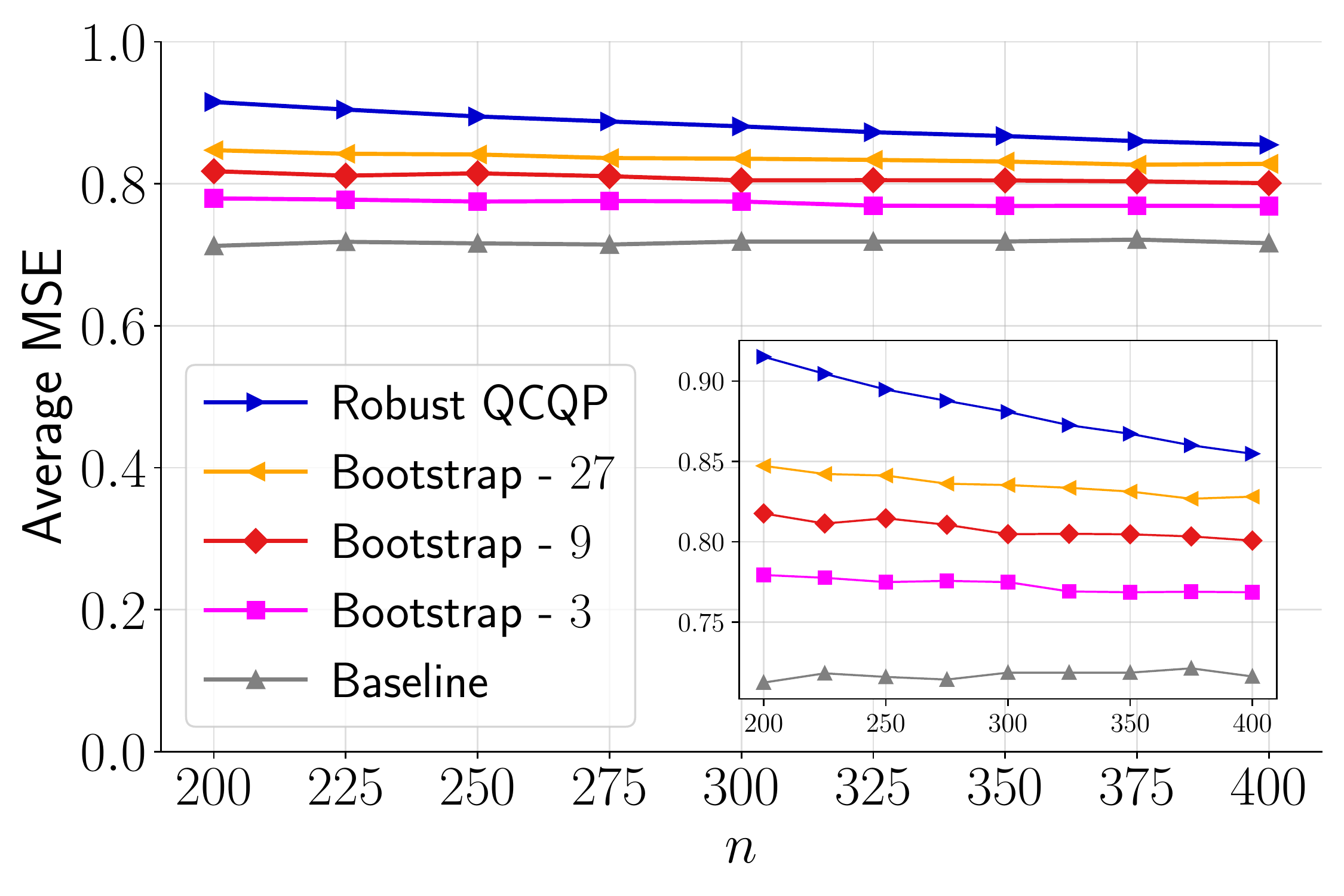}&
    \includegraphics[width=0.3\linewidth,clip]{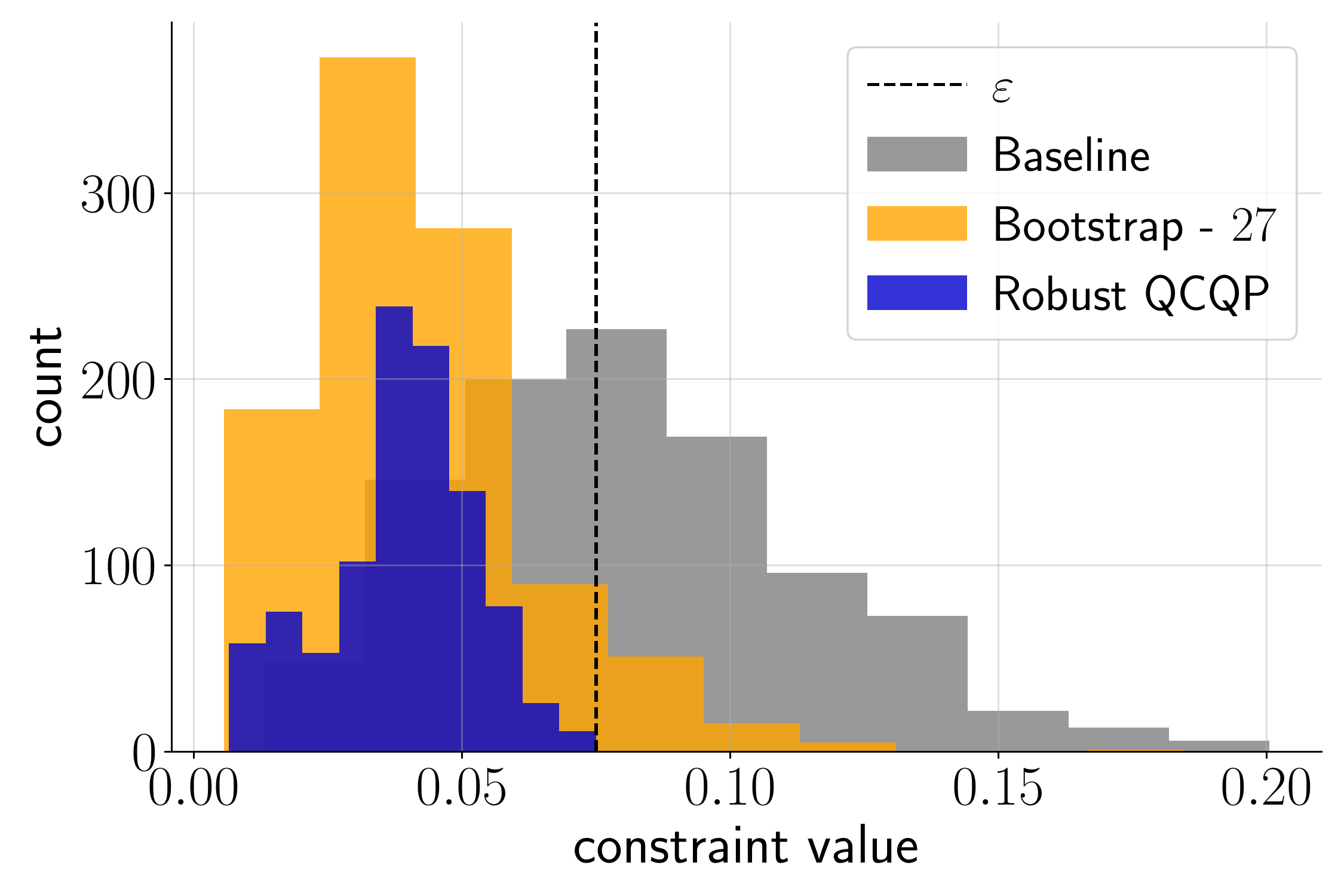}
    \\  & (a) $d = 2$, $\Sigma^{\star} = \Sigma^{\text{gen}}_2$, $\varepsilon = 0.075$ & \\
   
    \includegraphics[width=0.3\linewidth,clip]{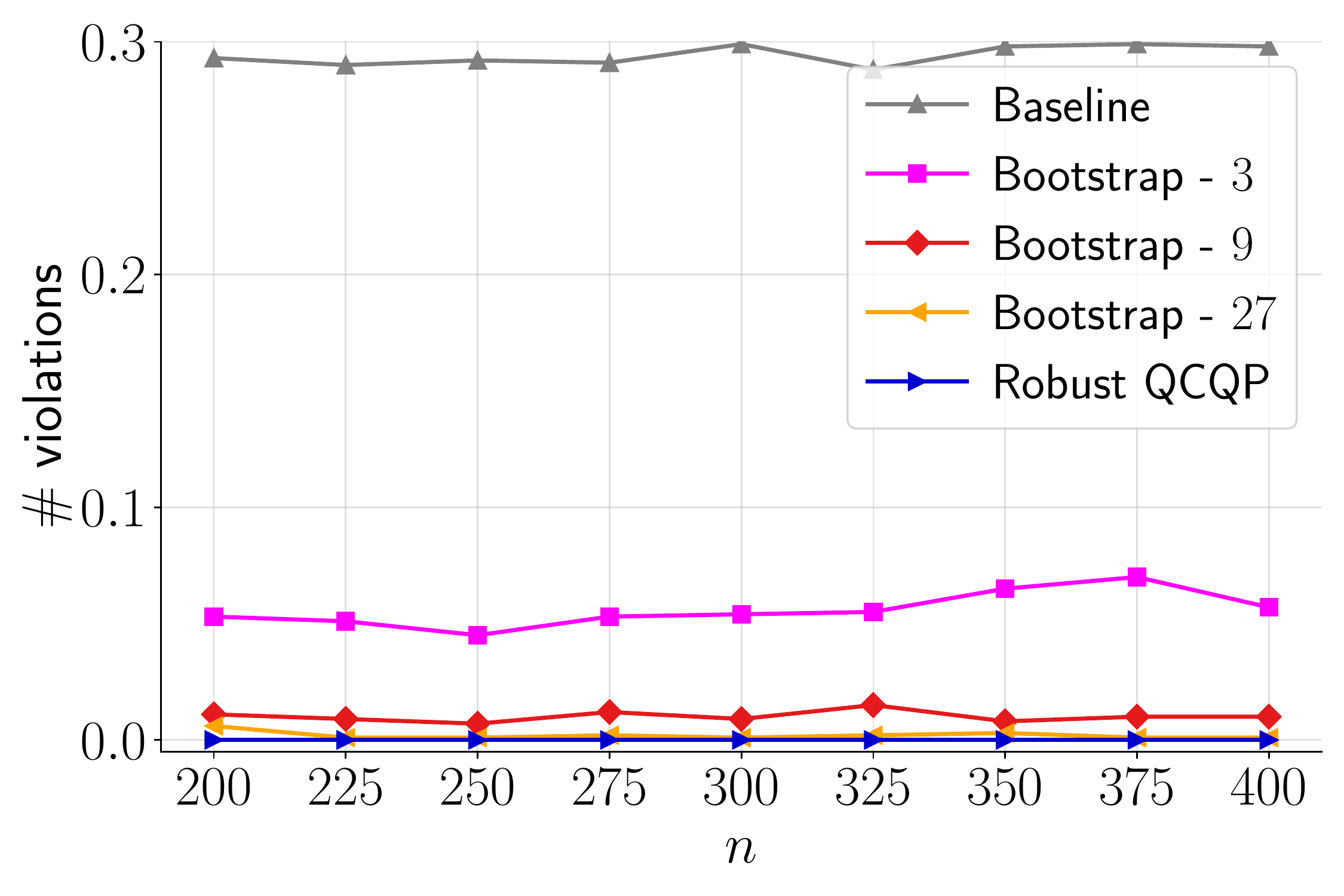}&
    \includegraphics[width=0.3\linewidth,clip]{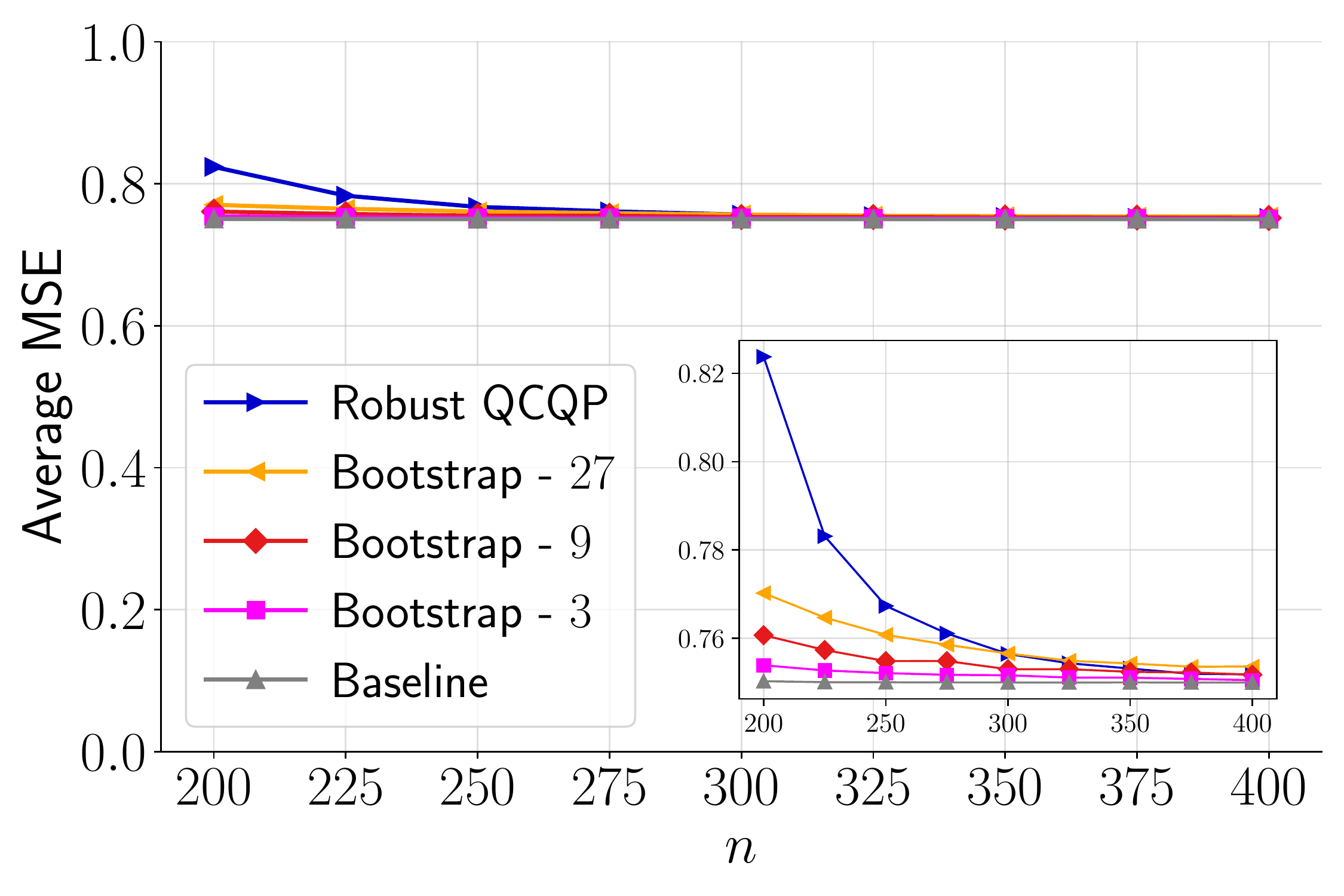}&
    \includegraphics[width=0.3\linewidth,clip]{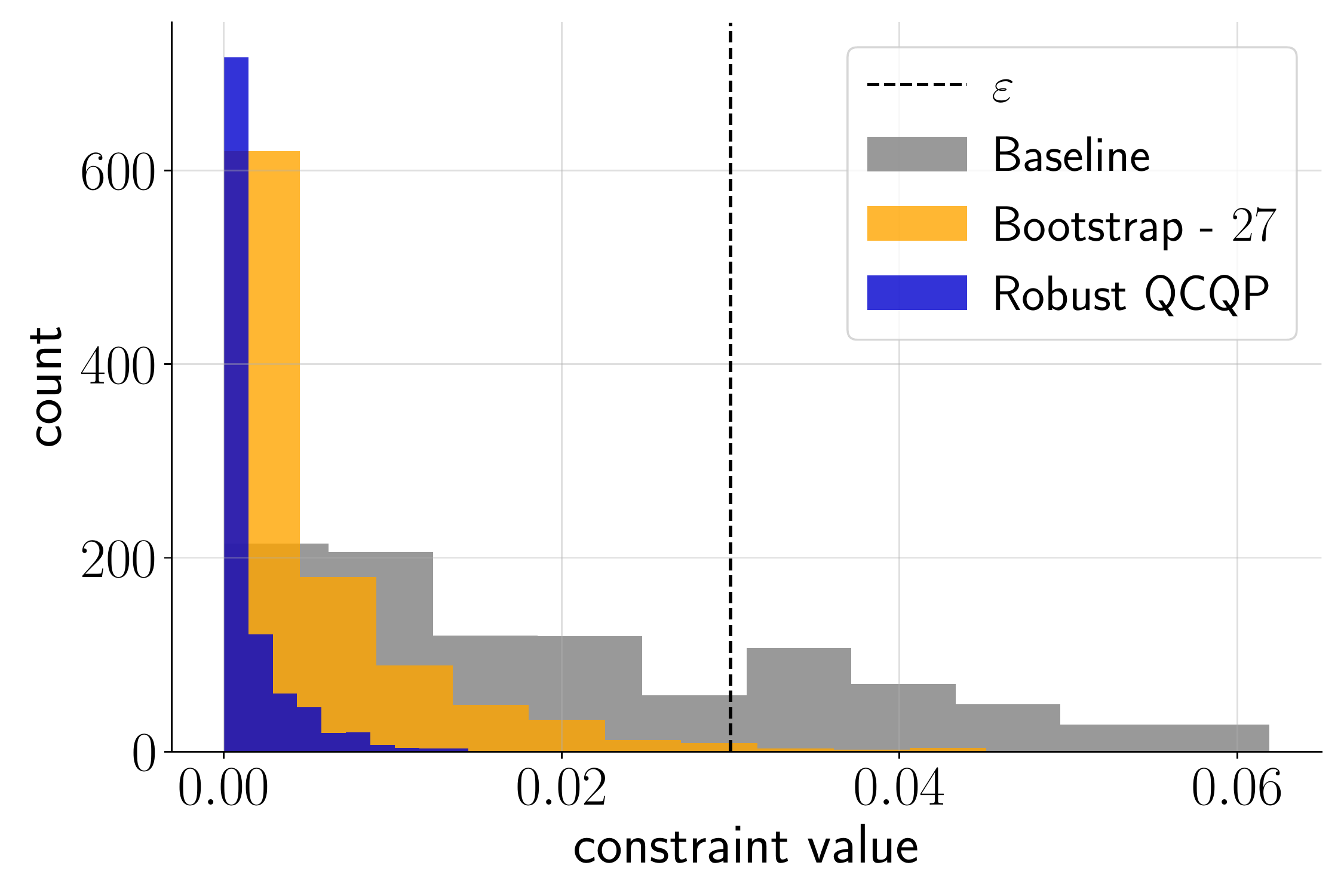}
    \\ & (b) $d = 2$, $\Sigma^{\star} = \Sigma^{\text{free}}_2$, $\varepsilon = 0.025$ & \\
    \end{tabular}
\caption{The performance of robust QCQP in \cref{eq_qcqp_limited_approx}, $\bss$ with $\s\in \{3,9,27\}$, and \texttt{Baseline} for $d=2$ and various $\Sigma^{\star}$. In the left column, we plot the fraction of violations of the true fairness constraint $\normalinner{\rvba}{\Bex}^2 \leq \varepsilon$ vs. $n$; in the middle column, we plot average MSE vs. $n$; in the right column, we plot the histogram of the value of $\normalinner{\rvba}{\Bex}^2$ over 1,000 trials for $n = 250$.}
\label{fig_gaussian_general}
\end{figure*}

\subsection{Synthetic data}
\label{subsec_gaussian_expts}
We generate synthetic data using two Gaussian distributions with zero mean, $d = 2$, and covariances matrices (a) $\Sigma^{\text{gen}}_2$ and (b) $\Sigma^{\text{fair}}_2$ defined in \cref{sec_more_details_gaussian}. We present our results when uncertainty is due to randomly missing sensitive attributes. In \cref{sec_more_details_gaussian}, we present results on $(i)$ another choice of covariance, $(ii)$ $d = 3$, and $(iii)$ uncertainty in sensitive attribute due to noise. We estimate $\BexEstimated$ using $n$ samples of $(\rvbx, \rve)$ for various choices of $n$. Then, we compare the robust QCQP in \eqref{eq_qcqp_limited_approx} and $\bss$ applied to the QCQP in \eqref{eq_qcqp} (for various $\s$) against \texttt{Baseline}, which solves the QCQP in \eqref{eq_qcqp} using $\BexEstimated$. We solve the resulting optimization problems using the CVXPY library \citep{diamond2016cvxpy}. The covariances $\Sigma^{\text{gen}}_2$ an $\Sigma^{\text{fair}}_2$ are designed to demonstrate the general behavior (where uncertainty hurts) and the free-fairness behavior (\cref{coro_free_fairness}), respectively. 

The results, averaged over 1000 random trials, are shown in \cref{fig_gaussian_general} (the error bars are too small to see). We observe that robust QCQP always ensures no fairness violations. Additionally, the performance (in terms of average MSE) of robust QCQP monotonically improves as $n$ increases, as stated in \cref{sec_mono_perf}. More importantly, in \cref{fig_gaussian_general} (b), robust QCQP does not incur any significant loss in the performance, and thus demonstrates the free-fairness phenomenon in \cref{coro_free_fairness}, say, $n \approx 350$ onwards. 
We also see that $\bss$ well approximates the performance of robust QCQP and outperforms \texttt{Baseline} in terms of fairness violations. As alluded to earlier, $\bss$ achieves a better fairness criterion as we increase the number of subsamples $\s$ by forming a more accurate uncertainty set. However, the benefit of larger $\s$ comes with an increased computation.

\subsection{Real-world data}
\label{subsec_real_expts}
\begin{table}[b]
\caption{Overview of datasets.}
\centering
{\small
\begin{tabular}{p{14mm}p{20mm}p{55mm}p{35mm}}
\toprule
\textbf{Dataset} & \textbf{Task} & \textbf{Outcome} & \textbf{Sensitive Attribute} \\
\midrule
Adult &Classification & Income $\geq $ \$50000 (binary) & Sex (binary) \\
Crime & Regression & Crimes per population (continuous) & Race (continuous) \\
Insurance & Regression & Medical expenses (continuous) & Sex (binary) \\
\bottomrule
\end{tabular}}\label{table:datasets}
\end{table}
We test $\bss$ on real-world classification and regression tasks for group fairness notions of independence and separation using 3 datasets: Adult, Crime, and Insurance. We provide an overview of these datasets in \cref{table:datasets} with detailed descriptions and pre-processing steps in \cref{sec_more_details_real}.\\

\noindent \textbf{Training details.} For all datasets, we train a two-layered neural network. 
We use the log loss for classification and MSE for regression. We use the $\chi^2$-divergence to impose the independence (conditional independence) for independence (separation). For continuous sensitive attributes, we induce uncertainty in every sensitive attribute by adding independent $\mathcal{N}(0,\sigma^2)$ noise ($\sigma = 0.5$ for Crime). For binary sensitive attributes, we induce uncertainty by keeping only $n$ out of $N$ sensitive attributes ($n=100$ for Adult and $n=10$ Insurance). For $\bss$, we set $\s=5$ for all three datasets. Given a fairness target $\epsilon$, we train a model over 50 independent trials of random missingness (for Adult and Insurance) or random noise (for Crime), and report the average performance (the error bars are too small to see). We sweep over 500 different $\epsilon$ from $0.001$ to $0.5$, and plot the prediction-fairness trade-off frontier by using a simple moving average over 5 entries. We provide more implementation details and experiments with different $n$ and $\sigma$ in \cref{sec_more_details_real}.\\

\noindent \textbf{Evaluation metrics.} On a held-out test set, we report predictive power using error rate (lower is better) for classification and MSE (lower is better) for regression. We evaluate (a) independence (lower is better) using demographic parity, i.e., $|P(\hat{\rvy}|\rve=1)-P(\hat{\rvy}|\rve=0)|$ for classification and $\chi^2$-divergence for regression, and (b) separation (lower is better) using equal opportunity, i.e., $|P(\hat{\rvy}|\rve=1, \rvy=1)-P(\hat{\rvy}|\rve=0, \rvy=1)|$ for classification and $\chi^2$-divergence for regression.\\

\begin{figure*}[t]
    \centering
    \begin{tabular}{ccc}
    \includegraphics[width=0.31\linewidth,clip]{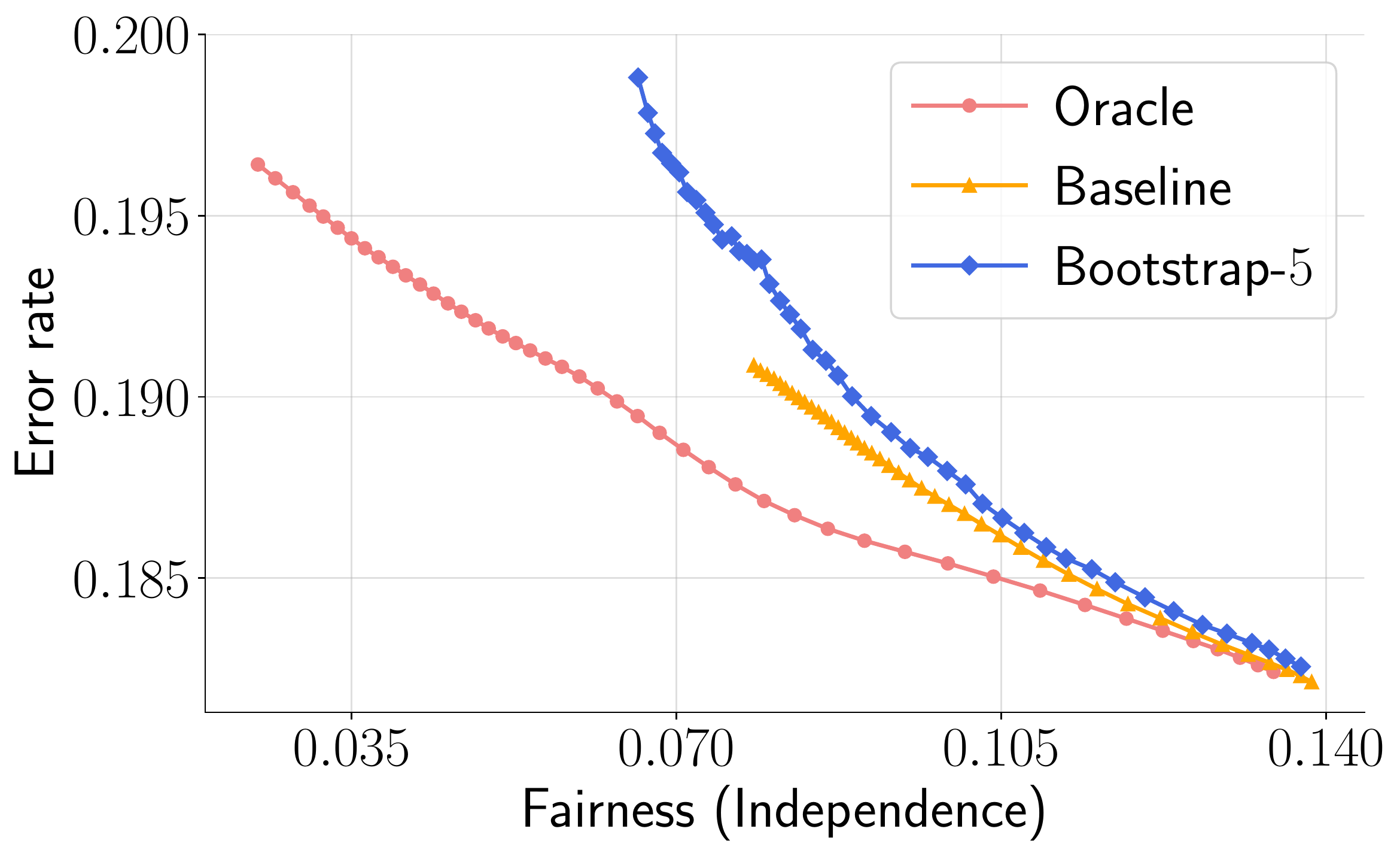}&
    \includegraphics[width=0.31\linewidth,clip]{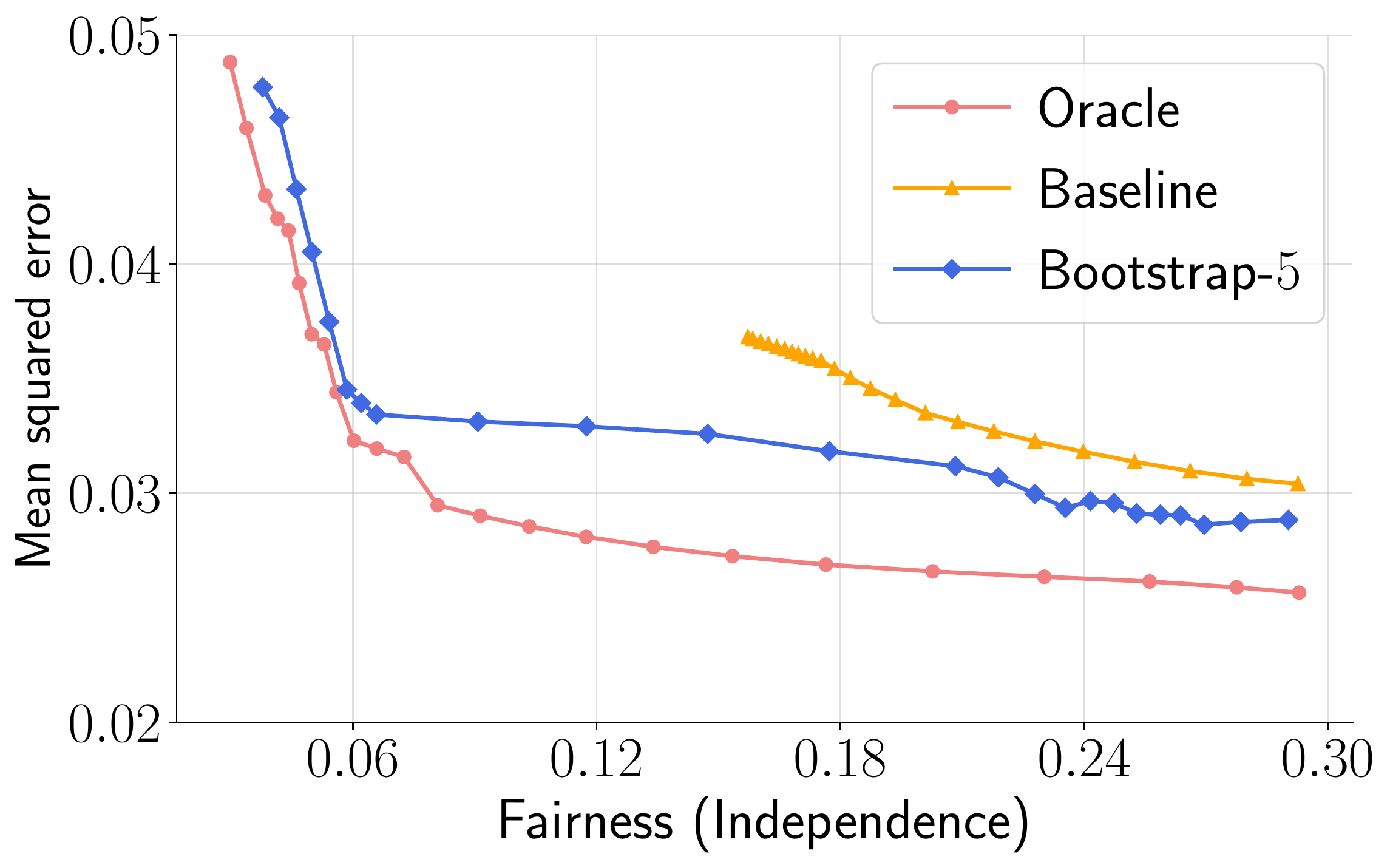}&
    \includegraphics[width=0.31\linewidth,clip]{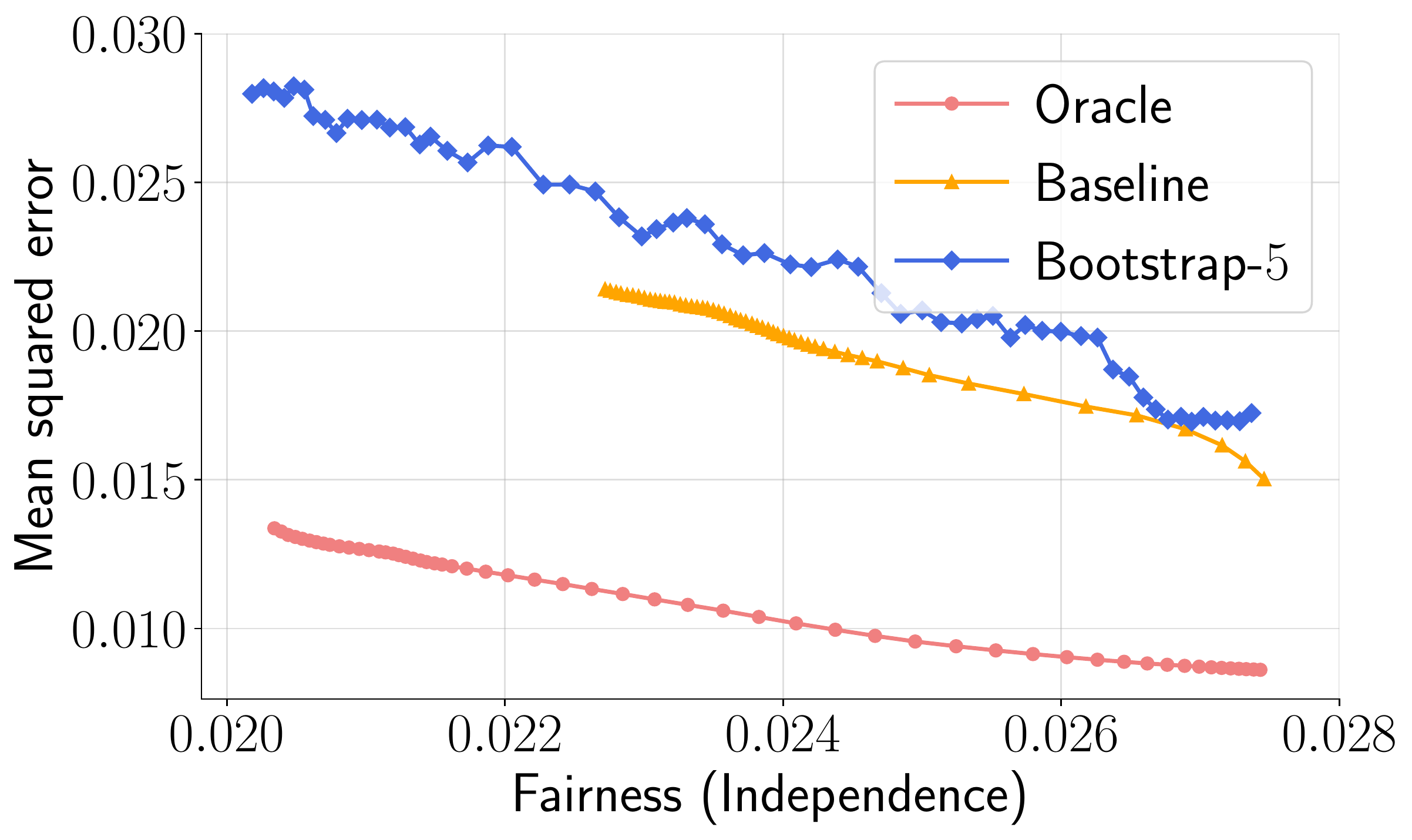}
    \\ (a) Adult dataset & (b) Crime dataset & (c) Insurance  dataset \\
    \includegraphics[width=0.31\linewidth]{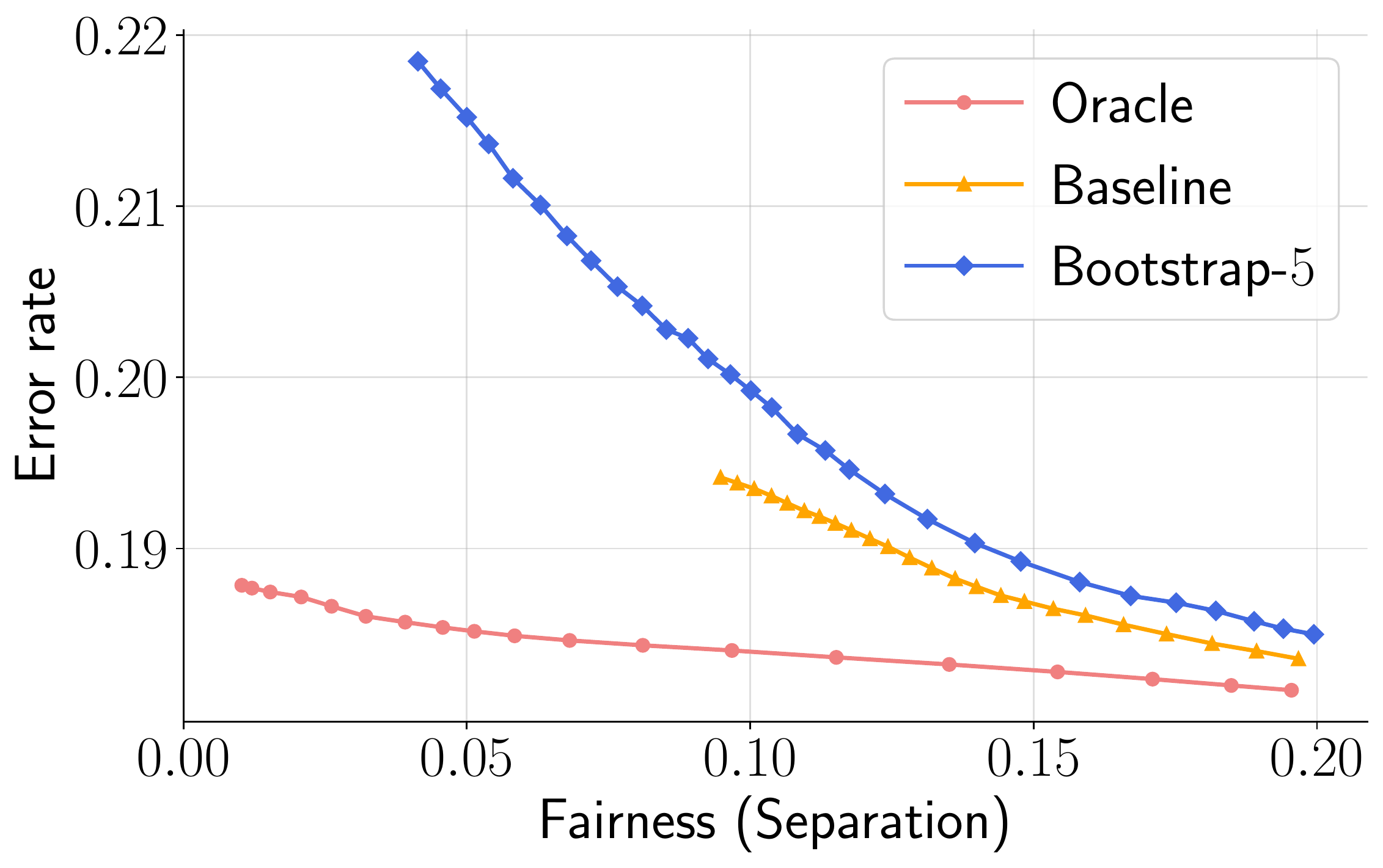}&
    \includegraphics[width=0.31\linewidth]{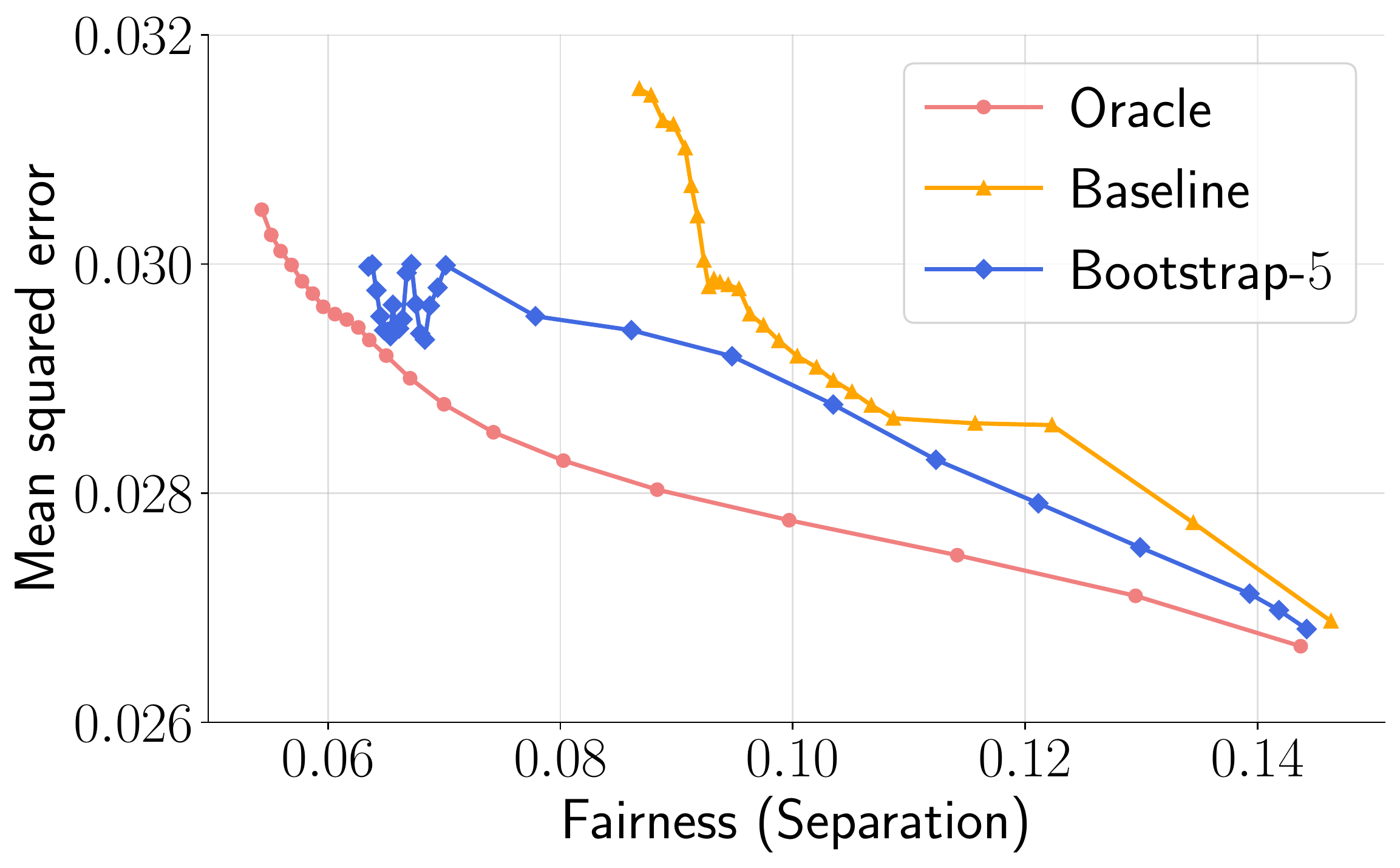}&
    \includegraphics[width=0.31\linewidth,clip]{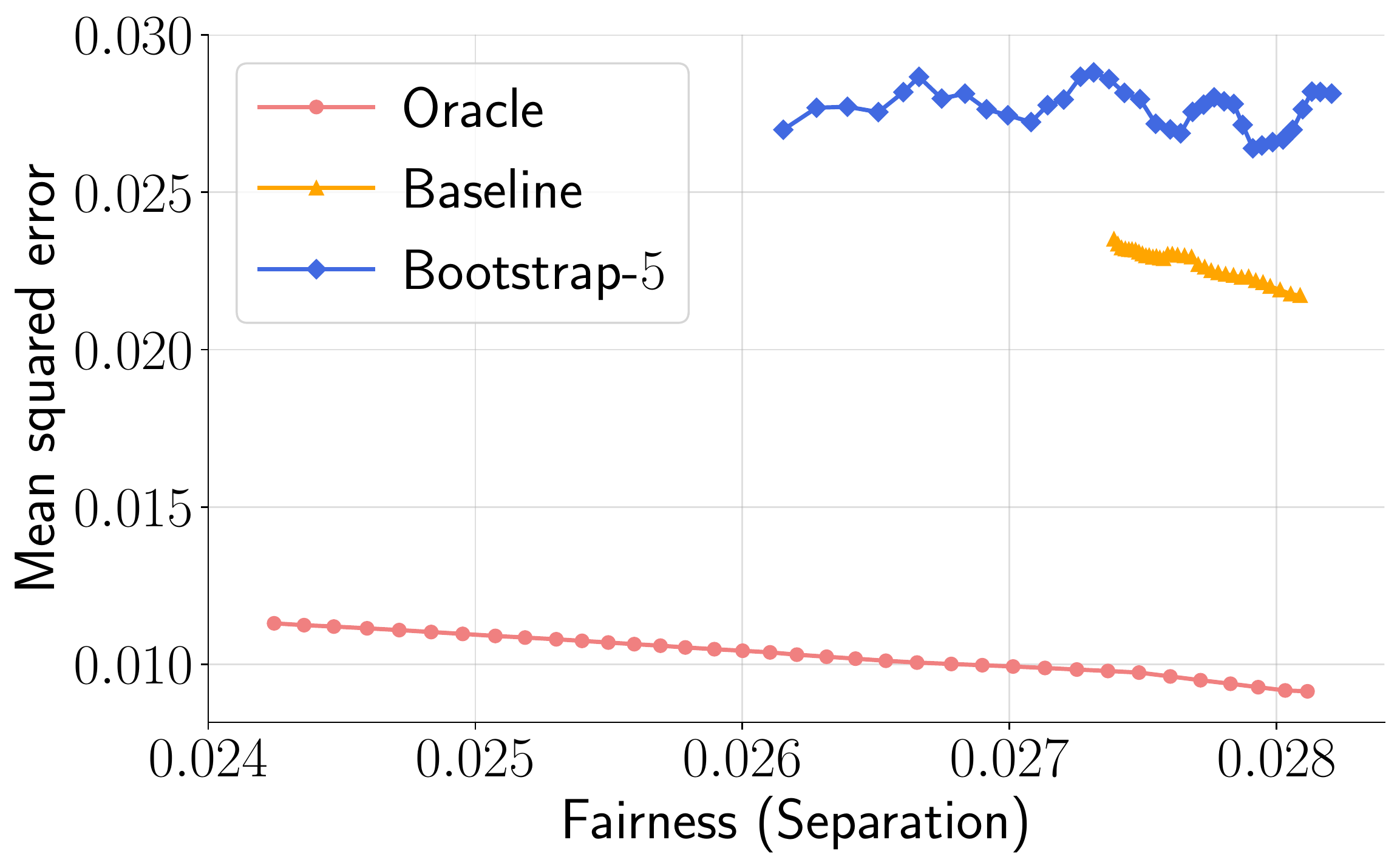}
    \\ (d) Adult dataset & (e) Crime  dataset & (f) Insurance dataset 
    \end{tabular}
\caption{Performance of $\bss$ and \texttt{Baseline} for independence (top row) and separation (bottom row). $\bss$ achieves much better fairness levels compared to \texttt{Baseline} throughout.}
\label{plots_real}
\end{figure*}

\noindent \textbf{Results.} We compare with \texttt{Baseline} which does not have the additional constraints of $\bss$, i.e., it solves for \cref{eq_alg_lag} with $\lambda_1, \cdots, \lambda_{\s}$ fixed to 0. For reference, we also compare with \texttt{Oracle} that has access to all the true sensitive attributes. \cref{plots_real} show that \texttt{Baseline} and $\bss$ exhibit a concentration of fairness levels near the extreme values of $\epsilon$ with $\bss$ being more noisy than \texttt{Baseline}. However, $\bss$ achieves significantly smaller fairness levels compared to \texttt{Baseline}. In fact, in most cases, $\bss$ achieves fairness levels that are comparable to \texttt{Oracle} while maintaining a relatively high level of predictive power.

%% file: 7future_work.tex
\section{Future Work}
In this work, inspired by our theoretical analysis, we propose an algorithm that introduces additional constraints to address the challenge of uncertain sensitive attributes in fair learning problems. While our algorithm demonstrates effectiveness in various practical scenarios, it is important to consider its potential limitations and areas for further exploration.

First, our theoretical analysis is focused on Gaussian distributions and the independence notion of fairness. The framework of the robust quadratically constrained quadratic problem, which provides a strict fairness guarantee, is applicable only within this specific setting. It is desirable to extend the analysis and provide rigorous guarantees for fair learning problems involving non-Gaussian distributions and alternative notions of fairness.

Second, it is crucial to fully characterize the behavior of our algorithm. For instance, understanding the trade-off associated with different choices of subsample size $k$ and the number of subsamples $\s$ would be valuable. While we explore various choices of $\s$ in experiments with synthetic data, we fix $\s = 5$ for experiments with real-world data due to computational constraints. Intuitively, $\s$ controls the balance between fairness and prediction by influencing the effectiveness of the fairness constraint in the optimization process. Analyzing the three-way trade-off between fairness, prediction accuracy, and computational requirements is essential for a comprehensive understanding of the proposed algorithm.

Third, in our experiments involving missing sensitive attributes, we assume that these values are missing at random. However, the behavior of our algorithm may differ when the missingness of sensitive attributes is not random. Further investigation is necessary to understand how our algorithm performs under scenarios where the missingness is not at random, as it could impact the fairness and accuracy of the model.

%% file: 8appendix.tex
\section{Proof of \cref{thm_ibqcqp,thm_2d}}
In this section, we prove \cref{thm_ibqcqp,thm_2d} as well as provide a characterization of the optimal $\rvba^\star$ in \cref{eq_qcqp}.

\subsection{Proof of \cref{thm_ibqcqp}: \ibeqqcqp}
\label{proof_them_ibqcqp}
First, we show that optimization in \cref{eq_gaussian_opt} is equivalent to the optimization below:
\begin{align}\label{eq_opt_inter_gaussian}
    \argmax_{\rvu} \Divergence(p_{\rvy, \rvu} \| p_{\rvy} p_{\rvu})  \qtext{s.t.} \Divergence(p_{\rve, \rvu} \| p_{\rve} p_{\rvu}) \leq \epsilon.
\end{align}
To that end, from the definition of conditional variance, we have
\begin{align}
    \Expectation\normalbrackets{(\rvy - \Expectation\normalbrackets{\rvy | \rvu})^2} = \Expectation\bigbrackets{\Variance\normalbrackets{\rvy | \rvu}}.  \label{eq_mmse}
\end{align}
Now, for a joint Gaussian vector $(\rvy, \rvu) \in \Reals^2$ with covariance $\bsigma$, we have
\begin{align}
\Expectation\bigbrackets{\Variance\normalbrackets{\rvy | \rvu}} \sequal{(a)} \Expectation\bigbrackets{\SigmaYY - \SigmaYU\SigmaUU^{-1} \SigmaUY} = \SigmaYY - \SigmaYU\SigmaUU^{-1} \SigmaUY,
\end{align}
where $(a)$ follows from Schur's complement. Since $\SigmaYY$ is a constant w.r.t. $\rvu$, we can write
\begin{align}
 \arg \min_{\rvu} \Expectation\bigbrackets{\Variance\normalbrackets{\rvy | \rvu}} =  \arg \min_{\rvu}- \SigmaYU\SigmaUU^{-1} \SigmaUY   = \arg \max_{\rvu} \SigmaYU\SigmaUU^{-1} \SigmaUY. \label{eq_optimal_mse_gaussian}
\end{align}
Now, from \citet[Lemma. 68]{huang2019universal}, we have
\begin{align}
\Divergence(p_{\rvy, \rvu} \| p_{\rvy} p_{\rvu})  = \fronorm{\SigmaYY^{-1/2} \SigmaYU \SigmaUU^{-1/2}}^2 \sequal{(a)} \SigmaYY^{-1/2} \SigmaYU \SigmaUU^{-1}\SigmaUY\SigmaYY^{-1/2},
\end{align}
where $\fronorm{\cdot}$ denotes the Frobenius norm, and $(a)$ follows because $\SigmaYY^{-1/2} \SigmaYU \SigmaUU^{-1}\SigmaUY\SigmaYY^{-1/2} \in \Reals$. As before, since $\SigmaYY$ is a constant w.r.t. $\rvu$, we can write
\begin{align}
    \arg \max_{\rvu} \Divergence(p_{\rvy, \rvu} \| p_{\rvy} p_{\rvu})  = \arg \max_{\rvu} \SigmaYU \SigmaUU^{-1}\SigmaUY. \label{eq_Ibar_gaussian}
\end{align}
Then, \cref{eq_opt_inter_gaussian} follows by combining \cref{eq_mmse,eq_optimal_mse_gaussian,eq_Ibar_gaussian}. 

Next, from \citet[Theorem. 1]{bu2021sdp}, the optimization problem in \cref{eq_opt_inter_gaussian} is equivalent to the following semi-definite program (SDP) using the notion of canonical correlation matrices (\cref{def_ccm}):
\begin{align}
    \max_{\tbf{A} \in \bS^d} \trace{\Byx\tp \Byx \tbf{A}} \qtext{s.t.}  \trace{\Bex\tp \Bex \tbf{A}} \leq \varepsilon \stext{and} 0 \preceq \tbf{A} \preceq \tbf{I}, \label{eq_sdp}
\end{align}
where $\trace{\cdot}$ denote the trace of a matrix, $\bS^d$ is the space of $d \times d$ symmetric matrices, and $\tbf{A}$ is of the form $\Bxu\Bxu\tp$. Finally, we show that the SDP in \cref{eq_sdp} is equivalent to the QCQP in \cref{eq_qcqp}. First, we have
\begin{align}
    \trace{\Byx\tp \Byx \tbf{A}} \sequal{(a)} \trace{\Byx\tp \Byx \Bxu\Bxu\tp} \sequal{(b)} \trace{\Bxu\tp\Byx\tp \Byx \Bxu} & = \trace{(\Byx \Bxu)\tp \Byx \Bxu} \\
    & \sequal{(c)}\biginner{\Byx}{\Bxu\tp}^2, \label{eq_equivalence_1}
\end{align}
where $(a)$ follows because $\tbf{A} = \Bxu\Bxu\tp$, $(b)$ follows from the cyclic property of trace, and $(c)$ follows because $\normalinner{\Byx}{\Bxu\tp} = \Byx \Bxu  \in \Reals^{1 \times 1}$. Similarly, we have
\begin{align}
    \trace{\Bex\tp \Bex \tbf{A}} = \biginner{\Bex}{\Bxu\tp}^2. \label{eq_equivalence_2}
\end{align}
Finally, noting that $\tbf{A}$ is a rank 1 matrix, we have
\begin{align}
    0 \preceq \tbf{A} \preceq \tbf{I} \iff \Bxu \in \ball. \label{eq_equivalence_3}
\end{align}
Putting together \cref{eq_sdp,eq_equivalence_1,eq_equivalence_2,eq_equivalence_3} completes the proof.

\subsection{Proof of \cref{thm_2d}: \qcqpsubspace}
\label{subsec_proof_thm_2d}
It suffices to show that the projection of any $\rvba \in \ball$, satisfying the constraint $\normalinner{\rvba}{\Bex}^2 \leq \varepsilon$, onto the subspace spanned by $\Byx$ and $\Bex$ preserves the value of the objective and continues to satisfy the constraint. Fix some $\rvba$ in the feasible space. Let $\proj(\rvba)$ denote the projection of $\rvba$ on the subspace spanned by $\Byx$ and $\Bex$.

\paragraph{The value of the objective is preserved.} It suffices to show $\normalinner{\proj(\rvba)}{\Byx}^2 = \normalinner{\rvba}{\Byx}^2$. Consider an orthonormal basis $\normalbraces{\mathbf{n}_y,\mathbf{n}_e} \defn \bigbraces{\frac{\Byx}{\stwonorm{\Byx}}, \mathbf{n}_e}$ spanned by $\Byx$ and $\Bex$ where $\mathbf{n}_y = \frac{\Byx}{\stwonorm{\Byx}}$ and $\mathbf{n}_e$ is chosen to be orthogonal to $\mathbf{n}_y$. Then, 
\begin{equation}
    \proj(\rvba) = \biginner{\rvba}{\mathbf{n}_y} \mathbf{n}_y + \biginner{\rvba}{\mathbf{n}_e} \mathbf{n}_e.
\end{equation}
Using the definition of $\mathbf{n}_y$ and the orthogonality between $\mathbf{n}_e$ and $\Byx$, we have
\begin{equation}
    \biginner{\proj(\rvba)}{\Byx}^2 = \biginner{\rvba}{\mathbf{n}_y}^2 \stwonorm{\Byx}^2.
\end{equation}
Similarly, using the definition of $\mathbf{n}_y$, we have $\biginner{\rvba}{\Byx}^2 = \biginner{\rvba}{\mathbf{n}_y}^2 \stwonorm{\Byx}^2$.

\paragraph{The constraint is satisfied.} It suffices to show $\proj(\rvba) \in \ball$ and $\normalinner{\proj(\rvba)}{\Bex}^2 \leq \varepsilon$. It is easy to see $\proj(\rvba) \in \ball$ because $\stwonorm{\proj(\rvba)}\le \stwonorm{\rvba}$. Now, consider a different set of orthonormal basis $\normalbraces{\mathbf{n}'_e,\mathbf{n}'_y} \defn \bigbraces{\frac{\Bex}{\stwonorm{\Bex}}, \mathbf{n}'_y}$ spanned by $\Bex$ and $\Byx$ where $\mathbf{n}'_e = \frac{\Bex}{\stwonorm{\Bex}}$ and $\mathbf{n}'_y$ is chosen to be orthogonal to $\mathbf{n}'_e$. Then,
\begin{equation}
    \proj(\rvba) = \biginner{\rvba}{\mathbf{n}'_e} \mathbf{n}'_e + \biginner{\rvba}{\mathbf{n}'_y} \mathbf{n}'_y.
\end{equation}
Using the definition of $\mathbf{n}'_e$ and the orthogonality between $\mathbf{n}'_y$ and $\Bex$, we have
\begin{equation}
    \biginner{\proj(\rvba)}{\Bex}^2 = \biginner{\rvba}{\mathbf{n}'_e}^2 \stwonorm{\Bex}^2 \sequal{(a)} \biginner{\rvba}{\Bex}^2 \sless{\cref{eq_qcqp}} \varepsilon,
\end{equation}
where $(a)$ follows from the definition of $\mathbf{n}'_e$.

\subsection{Characterizing the optimal $\rvba^\star$ in \cref{eq_qcqp}}
\label{subsec_characterizing_qcqp}
Fix any $\varepsilon > 0$. For $d = 2$, it is convenient to work with polar coordinates. Let $\Bex = \polar{\re}{\thetae}$ for some $\re > 0$ and $\thetae \in [0,2\pi]$. Similarly, let $\Byx = \polar{\ry}{\thetay}$ for some $\ry > 0$ and $\thetay \in [0,2\pi]$. Let $\truervba$ denote the set of optimal $\rvba$ in \cref{eq_qcqp}, i.e.,
\begin{align}
    \truervba = \arg\max_{\rvba \in \ball} \biginner{\rvba}{\Byx}^2 \stext{s.t}  \biginner{\rvba}{\Bex}^2 \leq \varepsilon.
\end{align}
Then, the following lemma characterizes $\truervba$ depending on the values of $\re, \thetae$, $\thetay$, and $\varepsilon$.
\begin{lemma}\label{lemma_optimal_a_qcqp}
Let $\alpha \in [0,2\pi]$ be such that $\cos{\alpha} =
\sqrt{\varepsilon}/\re$ and $\sin{\alpha} = 
\sqrt{1 - \varepsilon/\re^2}$. Define the function $\Afun: [0,2\pi] \to [-1,1]^2$ such that $\Afun[(\theta)] = \polar{}{\theta}$. Then,
\begin{enumerate}[label=Case 1.\arabic*., leftmargin=20mm, labelsep=0.5em]
    \item \label{item:regime_1} $\truervba = \bigbraces{\Afun[(\alpha + \thetae)], \Afun[(\pi + \alpha + \thetae)]}$ $\iff$ $\thetay \in [\thetae, \alpha + \thetae] \cup [\pi + \thetae,\pi + \alpha + \thetae]$, 
    \item \label{item:regime_2} $\truervba = \bigbraces{\Afun[(\thetay)], \Afun[(\pi + \thetay)]}$ $\iff$ $\thetay \in [\alpha + \thetae, \pi - \alpha + \thetae] \cup [\pi + \alpha + \thetae, 2\pi- \alpha + \thetae]$, and
    \item \label{item:regime_3} $\truervba = \bigbraces{\Afun[(\pi - \alpha + \thetae)], \Afun[(- \alpha + \thetae)]}$ $\iff$ $\thetay \in [\pi - \alpha + \thetae, \pi + \thetae] \cup [-\alpha + \thetae, \thetae]$.
\end{enumerate}
\end{lemma}

\begin{figure}[ht]
\centering
\begin{tabular}{cc}
\begin{tikzpicture}[scale=2.5, declare function={axis = 1; R = 1; re = 1.5; thetae = 10; alpha = 53; thetay = 170;},
dot/.style={circle,fill,inner sep=1.5pt},
]
\centering
\begin{scope}[nodes={dot}]
\draw[name path=circ]   (0,0) coordinate (origin) circle[radius=R]; 
\draw[name path=yaxis, black] (0,-axis) -- (0,axis);
\draw[name path=xaxis, black] (-axis,0) -- (axis,0) coordinate (xplus);
\draw[name path=e, very thick, orange, dotted] (thetae:re) node[label={[label distance=-17mm]270:\rotatebox{10}{\normalsize$\Bex = (\re, \thetae)$}}] (e){}  -- (180+thetae:re) node[](){};
\draw[very thick, magenta, dotted] (thetay:re) node[](){} -- (180+thetay:re) node[label={[label distance=-17mm]270:\rotatebox{350}{\normalsize$\Byx = (\ry, \thetay)$}}] (){};
\draw[name path=eperpone, ultra thick, red] (alpha+thetae:R) node[label={[label distance=-5mm]85:\normalsize$(1,\alpha + \thetae)$}] (circleone){}  -- (-alpha+thetae:R) node[label={[label distance=-4mm]300:\normalsize$(1,-\alpha + \thetae)$}] (circletwo){};
\draw[name path=epertwo, ultra thick, red] (180-alpha+thetae:R) node[label={[label distance=-1mm]180:\normalsize$(1,\pi -\alpha + \thetae)$}] (circlethree){}  -- (180+alpha+thetae:R) node[label={[label distance=0mm]180:\normalsize$(1,\pi + \alpha + \thetae)$}] (circlefour){};
\draw[very thick, red, dashed] (circleone) -- (origin);
\draw[ultra thick, red] (origin) ++(alpha+thetae:1) arc (alpha+thetae:180-alpha+thetae:R);
\draw[ultra thick, red] (origin) ++(180+alpha+thetae:1) arc (180+alpha+thetae:360-alpha+thetae:R);
\draw[thick, blue] (180-alpha+thetae:R) -- ($(origin)!(180-alpha+thetae:R)!(thetay:re)$) node (){};
\draw[thick, blue] (-alpha+thetae:R) -- ($(origin)!(-alpha+thetae:R)!(180+thetay:re)$) node (){};
\end{scope}
\begin{scope}[on background layer]
\draw[fill=blue,opacity=.0,fill opacity=.2] (0,0) -- (alpha+thetae:R) -- (thetae:{cos(alpha)}) -- cycle;
\draw[fill=blue,opacity=.0,fill opacity=.2] (0,0) -- (180+alpha+thetae:R) -- (180+thetae:{cos(alpha)}) -- cycle;
\draw[fill=olive,opacity=.0,fill opacity=.2] (0,0) -- (-alpha+thetae:R) -- (thetae:{cos(alpha)}) -- cycle;
\draw[fill=olive,opacity=.0,fill opacity=.2] (0,0) -- (180-alpha+thetae:R) -- (180+thetae:{cos(alpha)}) -- cycle;
\draw[fill=red,opacity=.0,fill opacity=.2]  ++(alpha+thetae:1) arc (alpha+thetae:180-alpha+thetae:R);
\draw[fill=red,opacity=.0,fill opacity=.2] (0,0) -- (alpha+thetae:R) -- (180-alpha+thetae:R) -- cycle;
\draw[fill=red,opacity=.0,fill opacity=.2]  ++(180+alpha+thetae:1) arc (180+alpha+thetae:360-alpha+thetae:R);
\draw[fill=red,opacity=.0,fill opacity=.2] (0,0) -- (180+alpha+thetae:R) -- (-alpha+thetae:R) -- cycle;
\end{scope}
\path pic["\normalsize$\alpha$",draw=black,thick,angle eccentricity=1.3,double,angle radius=1cm] {angle=e--origin--circleone};
\end{tikzpicture}
\end{tabular}
\caption{Visualizing QCQP in \cref{eq_qcqp}. Each point is shown in polar coordinates, i.e., a point $(r,\theta)$ denotes $(r\cos{\theta}, r\sin{\theta})$. The point $\Bex$ is shown in orange. The region enclosed by the solid red lines and red arcs is the feasible space in \cref{eq_qcqp}. The feasible space is shaded in blue, red, and olive to represent \cref{item:regime_1}, \cref{item:regime_2}, and \cref{item:regime_3}, respectively. The point $\Byx$ is shown in magenta and falls under \cref{item:regime_3} The optimal solution set $\truervba$ consists of points $(1, \pi - \alpha + \thetae)$ and $(1, - \alpha, + \thetae$).}
\label{fig_proofs_qcqp}
\end{figure}
\begin{proof}
Consider any $\Bex = \polar{\re}{\thetae}$ (shown in orange in \cref{fig_proofs_qcqp}) and $\varepsilon > 0$. 
Let $\ParameterSetTrue(\Bex) \defn \normalbraces{\rvba: \normalinner{\rvba}{\Bex}^2 \leq \varepsilon \stext{and} \rvba \in \ball}$ be the feasible space in \cref{eq_qcqp}, i.e., the set of all $\rvba$ satisfying the fairness constraint in \cref{eq_qcqp} w.r.t $\Bex$. Then, $\ParameterSetTrue(\Bex)$ is the region enclosed by the solid red lines and red arcs in \cref{fig_proofs_qcqp}, i.e., the union of the regions shaded in various colors. To obtain this region, consider any $\rvba \in \ball$. Then, ensuring $\biginner{\rvba}{\Bex}^2 \leq \varepsilon$ is equivalent to ensuring that the projection of $\rvba$ on the line joining the origin and $\Bex$ (shown in dotted orange in \cref{fig_proofs_qcqp}) is no more than $\sqrt{\varepsilon}/\re$. Therefore, we drop perpendiculars to this line at a distance $\sqrt{\varepsilon}/\re$ on either side of the origin. We obtain the intersection of these perpendiculars and $\ball$ by standard algebra and trigonometry (these points are shown in \cref{fig_proofs_qcqp} in polar coordinates), and the enclosed region is $\ParameterSetTrue(\Bex)$. 

Now, consider any $\Byx = \polar{\ry}{\thetay}$ (shown in magenta in \cref{fig_proofs_qcqp}). Then,  obtaining $\truervba$ is equivalent to obtaining every $\rvba \in \ParameterSetTrue(\Bex)$ such that the projection of $\rvba$ on the line joining the origin and $\Byx$ (shown in dotted magenta in \cref{fig_proofs_qcqp}) is maximized. As a result, every $\rvba \in \truervba$ lies on the boundary of this region. To obtain $\truervba$, we drop perpendiculars to the line joining the origin and $\Byx$ from every $\rvba$ on the boundary of this region, e.g., we show two such perpendiculars in solid blue in \cref{fig_proofs_qcqp}. We add $\rvba$ to $\truervba$ if the distance between the origin and the point where the corresponding perpendicular intersects the line is maximum. It is straightforward to verify
that \cref{item:regime_1}, \cref{item:regime_2}, and \cref{item:regime_3} correspond to the regions shaded in blue, red, and olive, respectively (see \cref{fig_proofs_qcqp}).
\end{proof}

\section{Proof of \cref{thm_2d_robust}: \robustqcqpsubspace}
\label{subsec_proof_thm_2d_robust}
It suffices to show that the projection of any $\rvba \in \ball$, satisfying the constraint $\normalinner{\rvba}{\B}^2 \leq \varepsilon$ for all $\B \in \ballunc$, onto the subspace spanned by $\Byx$ and $\BexEstimated$ preserves the value of the objective and continues to satisfy the constraint.  Fix some $\rvba$ in the feasible space. Let $\proj(\rvba)$ denote the projection of $\rvba$ on the subspace spanned by $\Byx$ and $\BexEstimated$. The proof that the value of the objective remains preserved is analogous to the proof of \cref{thm_2d} (see \cref{subsec_proof_thm_2d}). 

\paragraph{The constraint is satisfied.} It suffices to show $\proj(\rvba) \in \ball$ and $\normalinner{\proj(\rvba)}{\B}^2 \leq \varepsilon$ for all $\B \in \ballunc$. It is easy to see $\proj(\rvba) \in \ball$ because $\stwonorm{\proj(\rvba)}\le \stwonorm{\rvba}$. Now, consider a set of orthonormal basis $\normalbraces{\mathbf{n}_e,\mathbf{n}_y} \defn \bigbraces{\frac{\BexEstimated}{\stwonorm{\BexEstimated}}, \mathbf{n}_y}$ spanned by $\BexEstimated$ and $\Byx$ where $\mathbf{n}_e = \frac{\BexEstimated}{\stwonorm{\BexEstimated}}$ and $\mathbf{n}_y$ is chosen to be orthogonal to $\mathbf{n}_e$. Then,
\begin{equation}
    \proj(\rvba) = \biginner{\rvba}{\mathbf{n}_e} \mathbf{n}_e + \biginner{\rvba}{\mathbf{n}_y} \mathbf{n}_y.
\end{equation}
Using the definition of $\mathbf{n}_e$ and the orthogonality between $\mathbf{n}_y$ and $\BexEstimated$, we have
\begin{equation}
    \biginner{\proj(\rvba)}{\BexEstimated} = \biginner{\rvba}{\mathbf{n}_e} \stwonorm{\BexEstimated} \sequal{(a)} \biginner{\rvba}{\BexEstimated}, \label{eq_proj_estimated}
\end{equation}
where $(a)$ follows from the definition of $\mathbf{n}_e$.

Now, for every $\B \in \ballunc$, we have $\B = \BexEstimated +\Delta \B$, where $\|\Delta \B\|\le \radius$. Then,
\begin{align}
    \biginner{\rvba}{\B} = \biginner{\rvba}{\BexEstimated +\Delta \B} = \biginner{\proj(\rvba)}{\BexEstimated} +\biginner{\rvba}{\Delta \B} & \sequal{\cref{eq_proj_estimated}} \biginner{\proj(\rvba)}{\BexEstimated} +\biginner{\rvba}{\Delta \B} \\
    & \leq  
    \biginner{\proj(\rvba)}{\BexEstimated} +\stwonorm{\rvba} \radius \sless{(a)} \sqrt{\varepsilon},
    \label{eq_abc}
\end{align}
where $(a)$ follows because $\biginner{\rvba}{\B}^2 \leq \varepsilon$. Similarly, we can show $\normalinner{\rvba}{\B} \geq -\sqrt{\varepsilon}$. Now, we can bound $\normalinner{\proj(\rvba)}{\B}$ as follows:
\begin{align}
    \biginner{\proj(\rvba)}{\B} = \biginner{\proj(\rvba)}{\BexEstimated +\Delta \B}& =  \biginner{\proj(\rvba)}{\BexEstimated} +\biginner{\proj(\rvba)}{\Delta \B} \\
    & \leq \biginner{\proj(\rvba)}{\BexEstimated} +\stwonorm{\proj(\rvba)} \radius \\
    & \sless{(a)} \biginner{\proj(\rvba)}{\BexEstimated} +\stwonorm{\rvba}r  \sless{\cref{eq_abc}} \sqrt{\varepsilon},
    \end{align}
where $(a)$ follows because $\stwonorm{\proj(\rvba)}\le \stwonorm{\rvba}$. Similarly, we can show $\normalinner{\proj(\rvba)}{\B} \geq -\sqrt{\varepsilon}$ completing the proof.

\section{Proof of \cref{thm_ibqcqp_limited_infinite}: \qcqpinf}
\label{sec_proof_qcqpinf}
Fix any $\hre \geq 0$, $\hthetae \in [0,2\pi]$, $\Deltal \geq 0$, and $\phil \in [0,\pi/2]$. For any $\B$, let $\ParameterSetTrue(\B) \defn \normalbraces{\rvba: \normalinner{\rvba}{\B}^2 \leq \varepsilon \stext{and} \rvba \in \ball}$ be the set of all $\rvba$ satisfying the fairness constraint in \cref{eq_qcqp} w.r.t $\B$. For a given $\B$, $\ParameterSetTrue(\B)$ can be constructed as described in the proof of \cref{lemma_optimal_a_qcqp} in \cref{subsec_characterizing_qcqp}. See \cref{fig_proofs_qcqp} for reference.

Now, to show that the optimization problem in \cref{eq_qcqp_limited_infinite_worse} is equivalent to the optimization problem in \cref{eq_qcqp_limited_infinite}, it suffices to show the corresponding constraints are equivalent, i.e.,
\begin{align}
    \bigcap_{\B \in \cS} \ParameterSetTrue(\B) = \bigcap_{\B \in \bcS} \ParameterSetTrue(\B). \label{eq_equal_intersections}
\end{align}
Consider any $\B_1, \B_2 \in \cS$ of the form: $\B_1 = \polar{r_1}{\theta}$ and $\B_2 = \polar{r_2}{\theta}$ for some $r_1, r_2, \theta$ such that 
$\normalabs{r_1 - \hre} \leq \Deltal$, $\normalabs{r_2 - \hre} \leq \Deltal$, and $\normalabs{\theta - \hthetae} \leq \phi$.
Without loss of generality, let $r_1 \geq r_2$. Now, it is straightforward to see that $\ParameterSetTrue(\B_1) \subseteq \ParameterSetTrue(\B_2)$. Therefore, \cref{eq_equal_intersections} follows from the definition of $\bcS$.

\subsection{Characterizing optimal $\rvba$ in \cref{eq_qcqp_limited_infinite_worse}}
\label{subsec_characterizing_qcqp_limited_infinite}
Fix any $\varepsilon > 0$. Let $\BexEstimated = \polar{\hre}{\hthetae}$ for some $\hre > 0$ and $\hthetae \in [0,2\pi]$. Consider some $\Delta \geq 0$ and $\phi \in [0, \pi/2]$ such that $\Bex \in \cS$ with probability at least $1 - \delta$ where $\cS$ is as defined in \cref{eq_angular_sector}. Let $\Byx = \polar{\ry}{\thetay}$ for some $\ry > 0$ and $\thetay \in [0,2\pi]$. Let $\limitedinfrvba$ denote the set of optimal $\rvba$ in \cref{eq_qcqp_limited_infinite_worse}, i.e.,
\begin{align}
    \limitedinfrvba = \arg\max_{\rvba \in \ball} \biginner{\rvba}{\Byx}^2 \stext{s.t}  \biginner{\rvba}{\B}^2 \leq \varepsilon \stext{for all} \B \in \cS. \label{eq_qcqp_limited_infinite_optimal_a}
\end{align}
Then, the following Lemma characterizes $\limitedinfrvba$ depending on the values of $\BexEstimated$, $\Byx$, $\Delta$, $\phi$, and $\varepsilon$.
\begin{lemma}\label{lemma_optimal_a_limited_infinite}
Let $\balpha \in [0,2\pi]$ be such that $\cos{\balpha} = 
\sqrt{\varepsilon}/(\hre + \Delta)
$ and $\sin{\balpha} = 
\sqrt{1 - \varepsilon/(\hre + \Delta)^2}
$. Define the function $\Afun: [0,2\pi] \to [-1,1]^2$ such that $\Afun[(\theta)] = \polar{}{\theta}$. Then,
\begin{enumerate}[leftmargin=7mm]
    \item If $\sqrt{\varepsilon} \geq (\hre+ \Delta) \sin{\phi}$:
\begin{enumerate}[label=Case 2.\arabic*., leftmargin=12mm, labelsep=0.5em]
    \item \label{item:case_21} $\limitedinfrvba = \frac{\sqrt{\varepsilon}}{\hre + \Delta} \bigbraces{\Afun[(\thetay)], \Afun[(\pi + \thetay)]}$ $\iff$ $\thetay \in [\hthetae - \phi, \hthetae + \phi] \cup [\pi + \hthetae- \phi, \pi +\hthetae + \phi]$, 
    \item \label{item:case_22} $\limitedinfrvba = \bigbraces{\Afun[(\hthetae + \phi + \balpha)], \Afun[(\pi + \hthetae + \phi + \balpha)]}$ $\iff$ $\thetay \in [\hthetae + \phi, \hthetae + \phi + \balpha] \cup [\pi + \hthetae+\phi, \pi + \hthetae+\phi + \balpha]$, 
    \item \label{item:case_23} $\limitedinfrvba = \bigbraces{\Afun[(\thetay)], \Afun[(\pi + \thetay)]}$ $\iff$ $\thetay \in [\hthetae + \phi + \balpha, \pi + \hthetae- \phi - \balpha] \cup [\pi +\hthetae + \phi +\balpha, 2\pi + \hthetae - \phi - \balpha]$,
    \item \label{item:case_24} $\limitedinfrvba = \bigbraces{\Afun[(\pi + \hthetae-\phi - \balpha)], \Afun[(\hthetae -\phi - \balpha)]}$ $\iff$ $\thetay \in [\pi +\hthetae-\phi - \balpha, \pi +\hthetae-\phi] \cup [\hthetae-\phi - \balpha,\hthetae -\phi]$.
\end{enumerate}
\item If $\sqrt{\varepsilon} \leq (\hre+ \Delta) \sin{\phi}$:
\begin{enumerate}[label=Case 3.\arabic*., leftmargin=12mm, labelsep=0.5em]
    \item \label{item:case_31} $\limitedinfrvba = \frac{\sqrt{\varepsilon}}{\hre + \Delta} \bigbraces{\Afun[(\thetay)], \Afun[(\pi + \thetay)]}$ $\iff$ $\thetay \in [\hthetae- \phi, \hthetae+\phi] \cup [\pi +\hthetae- \phi, \pi + \hthetae+\phi]$,
    \item \label{item:case_32} $\limitedinfrvba = \frac{\sqrt{\varepsilon}}{(\hre + \Delta)\sin{\phi}} \bigbraces{\Afun(\frac{\pi}{2} + \hthetae),  \Afun(\frac{-\pi}{2} + \hthetae)}$ $\iff$ $\thetay \in [\hthetae + \phi, \pi + \hthetae - \phi] \cup [\pi + \hthetae + \phi, 2\pi + \hthetae- \phi]$.
\end{enumerate}
\end{enumerate}
\end{lemma}

\begin{proof}
For any $\B$, let $\ParameterSetTrue(\B) \defn \normalbraces{\rvba: \normalinner{\rvba}{\B}^2 \leq \varepsilon \stext{and} \rvba \in \ball}$ be the set of all $\rvba$ satisfying the fairness constraint in \cref{eq_qcqp} w.r.t $\B$. Then, from \cref{eq_equal_intersections}, the constraint in \cref{eq_qcqp_limited_infinite_optimal_a} is equivalent to $\cap_{\B \in \bcS} \ParameterSetTrue(\B)$. 
Consider any $\B_1 \in \bcS$. Then, $\B_1 = \polar{(\hre+ \Delta)}{\theta}$ for some $\normalabs{\theta - \hthetae} \leq \phi$. Further, $\ParameterSetTrue(\B_1)$ can be constructed as described in the proof of \cref{lemma_optimal_a_qcqp} in \cref{subsec_characterizing_qcqp}. See \cref{fig_proofs_qcqp} for reference. From standard algebra and trigonometry, it is easy to see the two straight lines forming the boundary of $\ParameterSetTrue(\B_1)$ in \cref{fig_proofs_qcqp} intersect the y-axis at points $(0, \frac{\sqrt{\varepsilon}}{(\hre+ \Delta) \sin{\theta}})$ and $(0, \frac{-\sqrt{\varepsilon}}{(\hre+ \Delta) \sin{\theta}})$. Then, depending on whether these points lie inside or outside $\ball$ when $\theta = \phi$, we have two cases.
\begin{enumerate}
    \item Suppose $\sqrt{\varepsilon} \geq (\hre+ \Delta) \sin{\phi}$. Then, it is straightforward to see that $\cap_{\B \in \bcS} \ParameterSetTrue(\B)$ is the region enclosed by the solid blue lines and blue arcs in \cref{fig_proofs_qcqp_limited_infinite}(a) (see \cref{fig_proofs_1}(a) for reference), i.e., the union of the regions shaded in various colors. The rest of the proof is similar to the proof of \cref{lemma_optimal_a_qcqp}.    Lastly, it is straightforward to verify that \cref{item:case_21}, \cref{item:case_22}, \cref{item:case_23}, and \cref{item:case_24} correspond to the regions shaded in green, blue, red, and olive, respectively. 
    \item Suppose $\sqrt{\varepsilon} \leq (\hre+ \Delta) \sin{\phi}$. Then, it is straightforward to see that $\cap_{\B \in \bcS} \ParameterSetTrue(\B)$ is the region enclosed by the solid blue lines and blue arcs in \cref{fig_proofs_qcqp_limited_infinite}(b) (see \cref{fig_proofs_1}(b) for reference), i.e., the union of the regions shaded in various colors. The rest of the proof is similar to the proof of \cref{lemma_optimal_a_qcqp}. Lastly, it is straightforward to verify that \cref{item:case_31} and \cref{item:case_32} correspond to the regions shaded in green and blue, respectively. 
\end{enumerate}
\end{proof}
\begin{figure}[ht]
\centering
\begin{tabular}{cc}
\adjustbox{valign=c}{
\begin{tikzpicture}[scale=1.75, declare function={axis = 1; R = 1; remin = 1.4; remax = 1.8; phi = 15; alphatilde = 60;},
dot/.style={circle,fill,inner sep=1.5pt},
]
\centering
\begin{scope}[nodes={dot}]
\draw[name path=circ]   (0,0) coordinate (origin) circle[radius=R]; 
\draw[name path=yaxis, black] (0,-axis) -- (0,axis);
\draw[name path=xaxis, black] (-axis,0) -- (axis,0) coordinate (xplus);
\draw[name path=ori_to_eminus_upper, thick, orange, dashed] (phi:remin) node[label={[label distance=-16mm]290:\rotatebox{40}{\normalsize$(\hat{r}_e \!-\! \Delta, \phi)$}}] (eminusupper){}  -- (origin);
\draw[name path=ori_to_eminus_lower, thick, orange, dashed] (-phi:remin) node[label={[label distance=-10mm]270:\rotatebox{40}{\normalsize$(\hat{r}_e \!-\! \Delta, \!-\phi)$}}] (eminuslower){}  -- (origin);
\draw[name path=ori_to_eplus_upper, thick, orange] (phi:remax) node[label={[label distance=-16mm]290:\rotatebox{40}{\normalsize$(\hat{r}_e \!+\! \Delta, \phi)$}}](eplusupper){} -- (eminusupper);
\draw[name path=ori_to_eplus_lower, thick, orange] (-phi:remax) node[label={[label distance=-10mm]270:\rotatebox{40}{\normalsize$(\hat{r}_e \!+\! \Delta, \!-\phi)$}}] (epluslower){}  -- (eminuslower);
\draw[name path=circleone_intersectionone, ultra thick, blue] (phi+alphatilde:R) node[label={[label distance=-20mm]180:\normalsize$(1,\phi + \balpha)$}] (circleone){} -- ($(origin)!(circleone)!(eplusupper)$) node (intersectionone){};
\draw[name path=circletwo_intersectiontwo, ultra thick, blue] (-phi-alphatilde:R) node[label={[label distance=-0mm]0:\normalsize$(1,-\phi - \balpha)$}] (circletwo){} -- ($(origin)!(circletwo)!(epluslower)$) node (intersectiontwo){};
\draw[name path=circlethree_intersectionthree, ultra thick, blue] (180-phi-alphatilde:R) node[label={[label distance=2mm]180:\normalsize$(1,\pi - \phi - \balpha)$}] (circlethree){} -- ($(origin)!(circlethree)!(180-phi:remax)$) node (intersectionthree){};
\draw[name path=circlefour_intersectionfour, ultra thick, blue] (180+phi+alphatilde:R) node[label={[label distance=2mm]180:\normalsize$(1,\pi+\phi+\balpha)$}] (circlefour){} -- ($(origin)!(circlefour)!(180+phi:remax)$) node (intersectionfour){};
\draw[very thick, orange] (origin) ++(phi:remin) arc (phi:-phi:remin);
\draw[very thick, orange] (origin) ++(phi:remax) arc (phi:-phi:remax);
\draw[ultra thick, blue] (origin) ++(phi+alphatilde:R) arc (phi+alphatilde:180-phi-alphatilde:R);
\draw[ultra thick, blue] (origin) ++(360-phi-alphatilde:R) arc (360-phi-alphatilde:180+phi+alphatilde:R);

\draw[ultra thick, blue] (origin) ++(phi:{cos(alphatilde)}) arc (phi:-phi:{cos(alphatilde)});
\draw[ultra thick, blue] (origin) ++(180-phi:{cos(alphatilde)}) arc (180-phi:180+phi:{cos(alphatilde)});
\end{scope}
\draw[very thick, blue, dotted] (circleone) -- (origin);
\path pic["\normalsize$\balpha$",draw=black,thick,angle eccentricity=1.4,double,angle radius=0.5cm] {angle=eplusupper--origin--circleone};
\begin{scope}[on background layer]
\draw[fill=blue,opacity=.0,fill opacity=.2] (0,0) -- (phi:{cos{alphatilde}}) -- (phi+alphatilde:1) -- cycle;
\draw[fill=olive,opacity=.0,fill opacity=.2] (0,0) -- (180-phi:{cos{alphatilde}}) -- (180-phi-alphatilde:1) -- cycle;
\draw[fill=olive,opacity=.0,fill opacity=.2] (0,0) -- (-phi:{cos{alphatilde}}) -- (-phi-alphatilde:1) -- cycle;
\draw[fill=blue,opacity=.0,fill opacity=.2] (0,0) -- (180+phi:{cos{alphatilde}}) -- (180+phi+alphatilde:1) -- cycle;
\draw[fill=red,opacity=.0,fill opacity=.2] (0,0) -- (phi+alphatilde:R) -- (180-phi-alphatilde:R) -- cycle;
\draw[fill=red,opacity=.0,fill opacity=.2]  ++(phi+alphatilde:R) arc (phi+alphatilde:180-phi-alphatilde:R);
\draw[fill=red,opacity=.0,fill opacity=.2] (0,0) -- (-phi-alphatilde:R) -- (180+phi+alphatilde:R) -- cycle;
\draw[fill=red,opacity=.0,fill opacity=.2]  ++(360-phi-alphatilde:R) arc (360-phi-alphatilde:180+phi+alphatilde:R);
\draw[fill=green,opacity=.0,fill opacity=.2] (0,0) -- (phi:{cos{alphatilde}}) -- (-phi:{cos{alphatilde}}) -- cycle;
\draw[fill=green,opacity=.0,fill opacity=.2]  ++(phi:{cos{alphatilde}}) arc (phi:-phi:{cos{alphatilde}});
\draw[fill=green,opacity=.0,fill opacity=.2] (0,0) -- (180-phi:{cos{alphatilde}}) -- (180+phi:{cos{alphatilde}}) -- cycle;
\draw[fill=green,opacity=.0,fill opacity=.2]  ++(180-phi:{cos{alphatilde}}) arc (180-phi:180+phi:{cos{alphatilde}});
\end{scope}
\end{tikzpicture}}
&
\adjustbox{valign=c}{
\begin{tikzpicture}[scale=1.75, declare function={axis = 1; R = 1; remin = 1.2; remax = 2.0; phi = 30; alphatilde = 63; yintercept = 0.9;},
dot/.style={circle,fill,inner sep=1.5pt},
]
\centering
\begin{scope}[nodes={dot}]
\draw[name path=circ]   (0,0) coordinate (origin) circle[radius=R]; 
\draw[name path=yaxis, black] (0,-axis) -- (0,axis);
\draw[name path=xaxis, black] (-axis,0) -- (axis,0) coordinate (xplus);
\draw[name path=ori_to_eminus_upper, thick, orange, dashed] (phi:remin) node[label={[label distance=-16mm]120:\rotatebox{40}{\normalsize$(\hat{r}_e \!-\! \Delta, \phi)$}}] (eminusupper){}  -- (origin);
\draw[name path=ori_to_eminus_lower, thick, orange, dashed] (-phi:remin) node[label={[label distance=-10mm]160:\rotatebox{40}{\normalsize$(\hat{r}_e \!-\! \Delta, \!-\phi)$}}] (eminuslower){}  -- (origin);
\draw[name path=ori_to_eplus_upper, thick, orange] (phi:remax) node[label={[label distance=-16mm]120:\rotatebox{40}{\normalsize$(\hat{r}_e \!+\! \Delta, \phi)$}}](eplusupper){} -- (eminusupper);
\draw[name path=ori_to_eplus_lower, thick, orange] (-phi:remax) node[label={[label distance=-10mm]160:\rotatebox{40}{\normalsize$(\hat{r}_e \!+\! \Delta, \!-\phi)$}}] (epluslower){}  -- (eminuslower);
\draw[name path=circleone_intersectionone, ultra thick, blue] (0,yintercept) node[label={[label distance=-1mm]180:\normalsize$(\frac{\cos{\balpha}}{\sin{\phi}},\frac{\pi}{2})$}] (circleone){} -- ($(origin)!(circleone)!(eplusupper)$) node (intersectionone){};
\draw[name path=circletwo_intersectiontwo, ultra thick, blue] (0,-yintercept) node[label={[label distance=-1mm]180:\normalsize$(\frac{\cos{\balpha}}{\sin{\phi}},\frac{3\pi}{2})$}] (circletwo){} -- ($(origin)!(circletwo)!(epluslower)$) node (intersectiontwo){};
\draw[name path=circlethree_intersectionthree, ultra thick, blue] (0,yintercept) -- ($(origin)!(circleone)!(180-phi:remax)$) node (intersectionthree){};
\draw[name path=circlefour_intersectionfour, ultra thick, blue] (0,-yintercept) -- ($(origin)!(circletwo)!(180+phi:remax)$) node (intersectionfour){};
\draw[very thick, orange] (origin) ++(phi:remin) arc (phi:-phi:remin);
\draw[very thick, orange] (origin) ++(phi:remax) arc (phi:-phi:remax);

\draw[ultra thick, blue] (origin) ++(phi:{cos(alphatilde)}) arc (phi:-phi:{cos(alphatilde)});
\draw[ultra thick, blue] (origin) ++(180-phi:{cos(alphatilde)}) arc (180-phi:180+phi:{cos(alphatilde)});
\end{scope}
\begin{scope}[on background layer]
\draw[fill=blue,opacity=.0,fill opacity=.2] (0,0) -- (phi:{cos{alphatilde}}) -- (0,yintercept) -- (180-phi:{cos{alphatilde}}) -- cycle;
\draw[fill=blue,opacity=.0,fill opacity=.2] (0,0) -- (-phi:{cos{alphatilde}}) -- (0,-yintercept) -- (180+phi:{cos{alphatilde}}) -- cycle;
\draw[fill=green,opacity=.0,fill opacity=.2] (0,0) -- (phi:{cos{alphatilde}}) -- (-phi:{cos{alphatilde}}) -- cycle;
\draw[fill=green,opacity=.0,fill opacity=.2]  ++(phi:{cos{alphatilde}}) arc (phi:-phi:{cos{alphatilde}});
\draw[fill=green,opacity=.0,fill opacity=.2] (0,0) -- (180-phi:{cos{alphatilde}}) -- (180+phi:{cos{alphatilde}}) -- cycle;
\draw[fill=green,opacity=.0,fill opacity=.2]  ++(180-phi:{cos{alphatilde}}) arc (180-phi:180+phi:{cos{alphatilde}});
\end{scope}
\end{tikzpicture}}
\\
(a) QCQP in \cref{eq_qcqp_limited_infinite_worse} if $\sqrt{\varepsilon} \geq (\hre+ \Delta) \sin{\phi}$
&
(b) QCQP in \cref{eq_qcqp_limited_infinite_worse} if $\sqrt{\varepsilon} \leq (\hre+ \Delta) \sin{\phi}$
\end{tabular}
\caption{Visualizing QCQP in \cref{eq_qcqp_limited_infinite_worse} for $\varepsilon = 0.9$, $\hre = 1.6$, and $\hthetae = 0$. We set $\Delta = 0.2$ and $\phi = 15$ for panel $(a)$, and $\Delta = 0.4$ and $\phi = 30$ for panel $(b)$. Each point is shown in polar coordinates, i.e., a point $(r,\theta)$ denotes $(r\cos{\theta}, r\sin{\theta})$. 
The annular sector $\cS$ is shown in orange. The region enclosed by the solid blue lines and blue arcs is the feasible space in \cref{eq_qcqp_limited_infinite_worse}. In Panel (a), the feasible space is shaded in green, blue, red, and olive to represent \cref{item:case_21}, \cref{item:case_22}, \cref{item:case_23}, and \cref{item:case_24} respectively. In Panel (b), the feasible space is shaded in green and blue to represent \cref{item:case_31} and \cref{item:case_32}, respectively.}
\label{fig_proofs_qcqp_limited_infinite}
\end{figure}

\section{Proof of \cref{prop_ibqcqp_limited}: \approxqcqp}
\label{subsec_proof_lemma_ibqcqp_approx}
Fix any $\hre \geq 0$, $\thetae \in [0,2\pi]$, $\Deltal \geq 0$, and $\phil \in [0,\pi/2]$. For any $\B$, let $\ParameterSetTrue(\B) \defn \normalbraces{\rvba: \normalinner{\rvba}{\B}^2 \leq \varepsilon \stext{and} \rvba \in \ball}$ be the set of all $\rvba$ satisfying the fairness constraint in \cref{eq_qcqp} w.r.t $\B$. For a given $\B$, $\ParameterSetTrue(\B)$ can be constructed as described in the proof of \cref{lemma_optimal_a_qcqp} in \cref{subsec_characterizing_qcqp}. See \cref{fig_proofs_qcqp} for reference.
\begin{figure}[b!]
\centering
\begin{tabular}{cc}
\adjustbox{valign=c}{
\begin{tikzpicture}[scale=1.75, declare function={axis = 1; R = 1; phi = 15; alpha = 60; alphaprime = 61.1; remin = 1.4; remax = 1.8;},
dot/.style={circle,fill,inner sep=1.5pt},
]
\centering
\begin{scope}[nodes={dot}]
\draw[name path=circ]   (0,0) coordinate (origin) circle[radius=R]; 
\draw[name path=yaxis, black] (0,-axis) -- (0,axis);
\draw[name path=xaxis, black] (-axis,0) -- (axis,0) coordinate (xplus);
\draw[name path=ori_to_eminus_upper, thick, orange, dashed] (phi:remin) node[label={[label distance=-16mm]290:\rotatebox{40}{\normalsize$(\hat{r}_e \!-\! \Delta, \phi)$}}] (eminusupper){}  -- (origin);
\draw[name path=ori_to_eminus_lower, thick, orange, dashed] (-phi:remin) node[label={[label distance=-10mm]270:\rotatebox{40}{\normalsize$(\hat{r}_e \!-\! \Delta, \!-\phi)$}}] (eminuslower){}  -- (origin);
\draw[name path=ori_to_eplus_upper, thick, orange] (phi:remax) node[blue, label={[label distance=-16mm]290:\rotatebox{40}{\normalsize$(\hat{r}_e \!+\! \Delta, \phi)$}}](eplusupper){} -- (eminusupper);
\draw[name path=ori_to_eplus_lower, thick, orange] (-phi:remax) node[green, label={[label distance=-10mm]270:\rotatebox{40}{\normalsize$(\hat{r}_e \!+\! \Delta, \!-\phi)$}}] (epluslower){}  -- (eminuslower);
\node[magenta] (thirdnode) at (0:2.0){};
\draw[very thick, red] (alpha+phi:R)  -- (-alpha+phi:R);
\draw[very thick, red] (180-alpha+phi:R) -- (180+alpha+phi:R);
\draw[very thick, red] (origin) ++(alpha+phi:R) arc (alpha+phi:180-alpha+phi:R);
\draw[very thick, red] (origin) ++(180+alpha+phi:R) arc (180+alpha+phi:360-alpha+phi:R);

\draw[very thick, red] (alphaprime:R)  -- (-alphaprime:R);
\draw[very thick, red] (180-alphaprime:R) -- (180+alphaprime:R);
\draw[very thick, red] (origin) ++(alphaprime:R) arc (alphaprime:180-alphaprime:R);
\draw[very thick, red] (origin) ++(180+alphaprime:R) arc (180+alphaprime:360-alphaprime:R);

\draw[very thick, red] (alpha-phi:R)  -- (-alpha-phi:R);
\draw[very thick, red] (180-alpha-phi:R) -- (180+alpha-phi:R);
\draw[very thick, red] (origin) ++(alpha-phi:R) arc (alpha-phi:180-alpha-phi:R);
\draw[very thick, red] (origin) ++(180+alpha-phi:R) arc (180+alpha-phi:360-alpha-phi:R);
    
\draw[very thick, orange] (origin) ++(phi:remin) arc (phi:-phi:remin);
\draw[very thick, orange] (origin) ++(phi:remax) arc (phi:-phi:remax);
\end{scope}
\begin{scope}[on background layer]
\draw[fill=red,opacity=.0,fill opacity=.2] (0,0) -- (phi:{cos{alpha}}) -- (phi+alpha:1) -- cycle;
\draw[fill=red,opacity=.0,fill opacity=.2] (0,0) -- (180-phi:{cos{alpha}}) -- (180-phi-alpha:1) -- cycle;
\draw[fill=red,opacity=.0,fill opacity=.2] (0,0) -- (-phi:{cos{alpha}}) -- (-phi-alpha:1) -- cycle;
\draw[fill=red,opacity=.0,fill opacity=.2] (0,0) -- (180+phi:{cos{alpha}}) -- (180+phi+alpha:1) -- cycle;
\draw[fill=red,opacity=.0,fill opacity=.2] (0,0) -- (phi+alpha:R) -- (180-phi-alpha:R) -- cycle;
\draw[fill=red,opacity=.0,fill opacity=.2]  ++(phi+alpha:R) arc (phi+alpha:180-phi-alpha:R);
\draw[fill=red,opacity=.0,fill opacity=.2] (0,0) -- (-phi-alpha:R) -- (180+phi+alpha:R) -- cycle;
\draw[fill=red,opacity=.0,fill opacity=.2]  ++(360-phi-alpha:R) arc (360-phi-alpha:180+phi+alpha:R);
\draw[fill=red,opacity=.0,fill opacity=.2] (0,0) -- (phi:{cos{alpha}}) -- (-phi:{cos{alpha}}) -- cycle;
\draw[fill=red,opacity=.0,fill opacity=.2] (0,0) -- (180-phi:{cos{alpha}}) -- (180+phi:{cos{alpha}}) -- cycle;
\end{scope}
\end{tikzpicture}}
&
\adjustbox{valign=c}{
\begin{tikzpicture}[scale=1.65, declare function={axis = 1; R = 1; phi = 30; alpha = 63; alphaprime = 66.8; remin = 1.2; remax = 2.0;yintercept = 0.9;},
dot/.style={circle,fill,inner sep=1.5pt},
]
\centering
\begin{scope}[nodes={dot}]
\draw[name path=circ]   (0,0) coordinate (origin) circle[radius=R]; 
\draw[name path=yaxis, black] (0,-axis) -- (0,axis);
\draw[name path=xaxis, black] (-axis,0) -- (axis,0) coordinate (xplus);
\draw[name path=ori_to_eminus_upper, thick, orange, dashed] (phi:remin) node[label={[label distance=-16mm]120:\rotatebox{40}{\normalsize$(\hat{r}_e \!-\! \Delta, \phi)$}}] (eminusupper){}  -- (origin);
\draw[name path=ori_to_eminus_lower, thick, orange, dashed] (-phi:remin) node[label={[label distance=-10mm]160:\rotatebox{40}{\normalsize$(\hat{r}_e \!-\! \Delta, \!-\phi)$}}] (eminuslower){}  -- (origin);
\draw[name path=ori_to_eplus_upper, thick, orange] (phi:remax) node[blue, label={[label distance=-16mm]120:\rotatebox{40}{\normalsize$(\hat{r}_e \!+\! \Delta, \phi)$}}](eplusupper){} -- (eminusupper);
\draw[name path=ori_to_eplus_lower, thick, orange] (-phi:remax) node[green, label={[label distance=-10mm]160:\rotatebox{40}{\normalsize$(\hat{r}_e \!+\! \Delta, \!-\phi)$}}] (epluslower){}  -- (eminuslower);
\node[magenta] (thirdnode) at (0:2.2){};
\draw[very thick, red] (alpha+phi:R)  -- (-alpha+phi:R);
\draw[very thick, red] (180-alpha+phi:R) -- (180+alpha+phi:R);
\draw[very thick, red] (origin) ++(alpha+phi:R) arc (alpha+phi:180-alpha+phi:R);
\draw[very thick, red] (origin) ++(180+alpha+phi:R) arc (180+alpha+phi:360-alpha+phi:R);

\draw[very thick, red] (alphaprime:R)  -- (-alphaprime:R);
\draw[very thick, red] (180-alphaprime:R) -- (180+alphaprime:R);
\draw[very thick, red] (origin) ++(alphaprime:R) arc (alphaprime:180-alphaprime:R);
\draw[very thick, red] (origin) ++(180+alphaprime:R) arc (180+alphaprime:360-alphaprime:R);

\draw[very thick, red] (alpha-phi:R)  -- (-alpha-phi:R);
\draw[very thick, red] (180-alpha-phi:R) -- (180+alpha-phi:R);
\draw[very thick, red] (origin) ++(alpha-phi:R) arc (alpha-phi:180-alpha-phi:R);
\draw[very thick, red] (origin) ++(180+alpha-phi:R) arc (180+alpha-phi:360-alpha-phi:R);

\draw[very thick, orange] (origin) ++(phi:remin) arc (phi:-phi:remin);
\draw[very thick, orange] (origin) ++(phi:remax) arc (phi:-phi:remax);
\end{scope}
\begin{scope}[on background layer]
\draw[fill=red,opacity=.0,fill opacity=.2] (0,0) -- (phi:{cos{alpha}}) -- (0,yintercept) -- cycle;
\draw[fill=red,opacity=.0,fill opacity=.2] (0,0) -- (180-phi:{cos{alpha}}) -- (0,yintercept) -- cycle;
\draw[fill=red,opacity=.0,fill opacity=.2] (0,0) -- (-phi:{cos{alpha}}) -- (0,-yintercept) -- cycle;
\draw[fill=red,opacity=.0,fill opacity=.2] (0,0) -- (180+phi:{cos{alpha}}) -- (0,-yintercept) -- cycle;
\draw[fill=red,opacity=.0,fill opacity=.2] (0,0) -- (phi:{cos{alpha}}) -- (-phi:{cos{alpha}}) -- cycle;
\draw[fill=red,opacity=.0,fill opacity=.2] (0,0) -- (180-phi:{cos{alpha}}) -- (180+phi:{cos{alpha}}) -- cycle;
\end{scope}
\end{tikzpicture}}
\\
(a) Feasible space for $\sqrt{\varepsilon} > (\hre+ \Delta) \sin{\phi}$.
&
(b) Feasible space for $\sqrt{\varepsilon} < (\hre+ \Delta) \sin{\phi}$.
\end{tabular}
\caption{Visualizing QCQP in \cref{eq_qcqp_limited_approx} for $\varepsilon = 0.9$, $\hre = 1.6$, and $\hthetae = 0$. We set $\Delta = 0.2$ and $\phi = 15$ for panel $(a)$, and $\Delta = 0.4$ and $\phi = 30$ for panel $(b)$. Each point is shown in polar coordinates, i.e., a point $(r,\theta)$ denotes $(r\cos{\theta}, r\sin{\theta})$. The annular sector $\cS$ is shown in orange. The shaded region is the feasible space in \cref{eq_qcqp_limited_approx}. The points $\Bex^{(1)}$, $\Bex^{(2)}$, and $\Bex^{(3)}$ are shown in magenta, blue and green, respectively.
}
\label{fig_proofs_qcqp_limited_approx}
\end{figure}

First, we note that the straight lines in the feasible space $\cap_{\B \in \bcS} \ParameterSetTrue(\B)$ in \cref{fig_proofs_qcqp_limited_infinite}(a) and \cref{fig_proofs_qcqp_limited_infinite}(b) are generated by the extreme points of the arc $\bcS$. These extreme points are precisely $\Bex^{(2)}$ and $\Bex^{(3)}$. Therefore, the same straight lines also arise in the feasible space $\ParameterSetTrue(\Bex^{(1)}) \cap \ParameterSetTrue(\Bex^{(2)}) \cap \ParameterSetTrue(\Bex^{(3)})$ in \cref{fig_proofs_qcqp_limited_approx}(a) and \cref{fig_proofs_qcqp_limited_approx}(b). Second, we note that the arcs in the feasible space $\cap_{\B \in \bcS} \ParameterSetTrue(\B)$ in \cref{fig_proofs_qcqp_limited_infinite}(a) and \cref{fig_proofs_qcqp_limited_infinite}(b) are at a distance $\sqrt{\varepsilon} / (\hre + \Delta)$ away from the origin. The point $\Bex^{(1)}$ is chosen precisely such that the space $\ParameterSetTrue(\Bex^{(1)})$ intersects each of these arcs at its extreme points (see \cref{fig_proofs_qcqp_limited_approx}(a) and \cref{fig_proofs_qcqp_limited_approx}(b)). Therefore, it is easy to see that the boundary of the feasible space $\cap_{\B \in \bcS} \ParameterSetTrue(\B)$ is not closer to the origin than the boundary of the feasible space $\ParameterSetTrue(\Bex^{(1)}) \cap \ParameterSetTrue(\Bex^{(2)}) \cap \ParameterSetTrue(\Bex^{(3)})$. This completes the proof.

\subsection{Characterizing optimal $\rvba$ in \cref{eq_qcqp_limited_approx}}
\label{subsec_characterizing_qcqp_limited_approx}
Fix any $\varepsilon > 0$. Let $\BexEstimated = \polar{\hre}{\hthetae}$ for some $\hre > 0$ and $\hthetae \in [0,2\pi]$. Consider some $\Delta \geq 0$ and $\phi \in [0, \pi/2]$ such that $\Bex \in \cS$ where $\cS$ is as defined in \cref{eq_angular_sector}. Let $\Byx = \polar{\ry}{\thetay}$ for some $\ry > 0$ and $\thetay \in [0,2\pi]$. Let $\limitedinfrvba$ denote the set of optimal $\rvba$ in \cref{eq_qcqp_limited_approx}, i.e.,
\begin{align}
    \limitedrvba = \arg\max_{\rvba \in \ball} \biginner{\rvba}{\Byx}^2 \stext{s.t}  \biginner{\rvba}{\Bex^{(1)}}^2 \leq \varepsilon, \biginner{\rvba}{\Bex^{(2)}}^2 \leq \varepsilon, \biginner{\rvba}{\Bex^{(3)}}^2 \leq \varepsilon. \label{eq_qcqp_limited_optimal_a}
\end{align}
Then, the following Lemma characterizes $\limitedinfrvba$ depending on the values of $\BexEstimated$, $\Byx$, $\Delta$, $\phi$, and $\varepsilon$. The proof is similar to the proof of \cref{lemma_optimal_a_limited_infinite} and is omitted for brevity.
\begin{lemma}\label{lemma_optimal_a_limited_approx}
Let $\balpha \in [0,2\pi]$ be such that $\cos{\balpha} =
\sqrt{\varepsilon}/(\hre + \Delta)
$ and $\sin{\balpha} = 
\sqrt{1 - \varepsilon/(\hre + \Delta)^2}
$. Define the function $\Afun: [0,2\pi] \to [-1,1]^2$ such that $\Afun[(\theta)] = \polar{}{\theta}$. Then,
\begin{enumerate}[leftmargin=7mm]
    \item If $\sqrt{\varepsilon} \geq (\hre+ \Delta) \sin{\phi}$:
\begin{enumerate}[label=Case 4.\arabic*., leftmargin=12mm, labelsep=0.5em]
    \item \label{item:case_41} $\limitedinfrvba = \frac{\sqrt{\varepsilon}}{\hre + \Delta} \bigbraces{\Afun[(\hthetae + \phi)], \Afun[(\pi + \hthetae + \phi)]}$ $\iff$ $\thetay \in [\hthetae, \hthetae + \phi] \cup [\pi + \hthetae, \pi +\hthetae + \phi]$,
    \item \label{item:case_42} $\limitedinfrvba = \bigbraces{\Afun[(\hthetae + \phi + \balpha)], \Afun[(\pi + \hthetae + \phi + \balpha)]}$ $\iff$ $\thetay \in [\hthetae + \phi, \hthetae + \phi + \balpha] \cup [\pi + \hthetae+\phi, \pi + \hthetae+\phi + \balpha]$, 
    \item \label{item:case_43} $\limitedinfrvba = \bigbraces{\Afun[(\thetay)], \Afun[(\pi + \thetay)]}$ $\iff$ $\thetay \in [\hthetae + \phi + \balpha, \pi + \hthetae- \phi - \balpha] \cup [\pi +\hthetae + \phi +\balpha, 2\pi + \hthetae - \phi - \balpha]$,
    \item \label{item:case_44} $\limitedinfrvba = \bigbraces{\Afun[(\pi + \hthetae-\phi - \balpha)], \Afun[(\hthetae -\phi - \balpha)]}$ $\iff$ $\thetay \in [\pi +\hthetae-\phi - \balpha, \pi +\hthetae-\phi] \cup [\hthetae-\phi - \balpha,\hthetae -\phi]$,
    \item \label{item:case_45} $\limitedinfrvba = \frac{\sqrt{\varepsilon}}{\hre + \Delta} \bigbraces{\Afun[(\hthetae - \phi)], \Afun[(\pi + \hthetae - \phi)]}$ $\iff$ $\thetay \in  [\pi + \hthetae -\phi, \pi +\hthetae] \cup [\hthetae - \phi, \hthetae]$.
\end{enumerate}
\item If $\sqrt{\varepsilon} \leq (\hre+ \Delta) \sin{\phi}$:
\begin{enumerate}[label=Case 5.\arabic*., leftmargin=12mm, labelsep=0.5em]
    \item \label{item:case_51} $\limitedinfrvba = \frac{\sqrt{\varepsilon}}{\hre + \Delta} \bigbraces{\Afun[(\hthetae + \phi)], \Afun[(\pi + \hthetae + \phi)]}$ $\iff$ $\thetay \in [\hthetae, \hthetae + \phi] \cup [\pi + \hthetae, \pi +\hthetae + \phi]$,
    \item \label{item:case_52} $\limitedinfrvba = \frac{\sqrt{\varepsilon}}{(\hre + \Delta)\sin{\phi}} \bigbraces{\Afun(\frac{\pi}{2} + \hthetae),  \Afun(\frac{-\pi}{2} + \hthetae)}$ $\iff$ $\thetay \in [\hthetae + \phi, \pi + \hthetae - \phi] \cup [\pi + \hthetae + \phi, 2\pi + \hthetae- \phi]$,
    \item \label{item:case_53} $\limitedinfrvba = \frac{\sqrt{\varepsilon}}{\hre + \Delta} \bigbraces{\Afun[(\hthetae - \phi)], \Afun[(\pi + \hthetae - \phi)]}$ $\iff$ $ \thetay \in [\pi + \hthetae -\phi, \pi +\hthetae] \cup [\hthetae - \phi, \hthetae]$.
\end{enumerate}
\end{enumerate}
\end{lemma}

\section{Analyzing the power of labeled sensitive attributes}
\label{subsec_power_each_sample}
In this section, we consider the case where uncertainty only stems from sensitive attributes missing at random, and the uncertainty can be improved by collecting more labeled sensitive attributes. We characterize the difference in the optimal objectives in \cref{eq_qcqp} and \cref{eq_qcqp_limited_approx} to obtain the power of each new labeled sensitive attribute.

Fix any $\varepsilon >0$. For any $n$, let $\BexEstimated(n) \defn \polar{\hre(n)}{\hthetae(n)}$ denote the estimate of $\Bex$ as a function of $n$ and let $\Deltal(n) \geq 0$ and $\phil(n) \in [0,\pi/2]$ denote the corresponding uncertainty parameters as a function of $n$, i.e., $\Bex = \polar{\re}{\thetae} \in \cSn$, with probability $1-\delta$. Let $\perf$ denote the difference in the optimal objectives in \cref{eq_qcqp,eq_qcqp_limited_approx}, i.e., $\perf \defn \Psi - \perf[2]$ where
\begin{align}
    \Psi & \defn \max_{\rvba \in \ball} \biginner{\rvba}{\Byx}^2 \qtext{s.t}  \biginner{\rvba}{\Bex}^2 \leq \varepsilon
\intertext{and}
    \perf[2] & \defn \max_{\rvba \in \ball} \biginner{\rvba}{\Byx}^2 \qtext{s.t}  \biginner{\rvba}{\Bex^{(1)}(n)}^2 \leq \varepsilon, \biginner{\rvba}{\Bex^{(2)}(n)}^2 \leq \varepsilon, \biginner{\rvba}{\Bex^{(3)}(n)}^2 \leq \varepsilon.
\end{align}
Then, the following lemma characterizes $\perf$ as a function of $n$ when $\thetae \geq \hthetae(n)$. The characterization when $\thetae \leq \hthetae(n)$ can be obtained analogously. The proof follows from \cref{lemma_optimal_a_qcqp} and \cref{lemma_optimal_a_limited_approx}, and is omitted for brevity.
\begin{lemma}\label{lemma_power_new_sample}
Let $\alpha, \balpha(n) \in [0,2\pi]$ be such that $(a)$ $\cos{\alpha} =
\sqrt{\varepsilon}/\re
$ and $\sin{\alpha} = 
\sqrt{1 - \varepsilon/\re^2}
$ and $(b)$ $\cos{\balpha(n)} =
\sqrt{\varepsilon}/(\hre(n) + \Delta(n))
$ and $\sin{\balpha(n)} = 
\sqrt{1 - \varepsilon/(\hre(n) + \Delta(n))^2}
$. 
Then,
\begin{enumerate}[leftmargin=7mm]
    \item \label{item:11} If $\sqrt{\varepsilon} \geq (\hre(n) + \Delta(n)) \sin{\phi(n)}$:
\begin{enumerate}[label=Case \Alph*., leftmargin=10mm, labelsep=1em]
    \item \label{item:case_A} $\perf \!=\! \cos{(\alpha + \thetae - \thetay)} - 
\cos{\balpha(n)}\cos{(\hthetae(n) + \phi(n) - \thetay)}$   when $\thetay \in [\thetae, \hthetae(n) + \phi(n)] \cup [\pi + \thetae, \pi + \hthetae(n) + \phi(n)]$,
    \item \label{item:case_B} $\perf \!=\! \cos{(\alpha + \thetae - \thetay)} - 
\cos{(\hthetae(n) + \balpha(n) + \phi(n) - \thetay)}$ when $\thetay \in [\hthetae(n) + \phi(n), \alpha + \thetae] \cup [\pi + \hthetae(n) + \phi(n), \pi + \alpha + \thetae]$,
    \item \label{item:case_C} $\perf \!=\! 1 - 
\cos{(\hthetae(n) + \balpha(n) + \phi(n) - \thetay)}$ when $\thetay \in [\alpha + \thetae, \hthetae(n) + \phi(n) + \balpha(n)] \cup [\pi +\alpha + \thetae, \pi + \hthetae(n) + \phi(n) + \balpha(n)]$,
    \item \label{item:case_D} $\perf \!=\! 0$ when $\thetay \in [\hthetae(n) + \phi(n) + \balpha(n), \pi + \hthetae(n)- \phi(n) - \balpha(n)] \cup [\pi +\hthetae(n) + \phi(n) +\balpha(n), 2\pi + \hthetae(n) - \phi(n) - \balpha(n)]$,
    \item \label{item:case_E} $\perf \!=\! 1 - 
\cos{(\thetay + \balpha(n) + \phi(n) - \hthetae(n))}$ when $\thetay \in [\pi +\hthetae(n)-\phi(n) - \balpha(n), \pi -\alpha + \thetae] \cup [\hthetae(n)-\phi(n) - \balpha(n), -\alpha + \thetae]$,
    \item \label{item:case_F} $\perf \!=\! \cos{(\thetay + \alpha - \thetae)} - 
\cos{(\thetay + \balpha(n) + \phi(n) - \hthetae(n))}$when $\thetay \in  [\pi -\alpha + \thetae, \pi + \hthetae(n) -\phi(n)] \cup [-\alpha+ \thetae, \hthetae(n) - \phi(n)]$,
    \item \label{item:case_G} $\perf \!=\! \cos{(\thetay + \alpha - \thetae)} - 
\cos{\balpha(n)}\cos{(\thetay  + \phi(n) - \hthetae(n))}$  when $\thetay \in  [\pi + \hthetae(n) -\phi(n), \pi +\hthetae(n)] \cup [\hthetae(n) - \phi(n), \hthetae(n)]$,
    \item \label{item:case_H} $\perf \!=\! \cos{(\thetay + \alpha - \thetae)} - 
\cos{\balpha(n)}\cos{(\thetay  - \phi(n) - \hthetae(n))}$ when $\thetay \in  [\pi + \hthetae(n), \pi +\thetae] \cup [\hthetae(n), \thetae]$.
\end{enumerate}
\item \label{item:22} If $\sqrt{\varepsilon}s \leq (\hre(n) + \Delta(n)) \sin{\phi(n)}$:
\begin{enumerate}[label=Case \Alph*., leftmargin=10mm, labelsep=1em]
    \item \label{item:case_A2} $\perf \!=\! \cos{(\alpha + \thetae - \thetay)} - 
\cos{\balpha(n)}\cos{(\hthetae(n) + \phi(n) - \thetay)}$   when $\thetay \in [\thetae, \hthetae(n) + \phi(n)] \cup [\pi + \thetae, \pi + \hthetae(n) + \phi(n)]$,
    \item \label{item:case_B2} $\perf \!=\! \cos{(\alpha + \thetae - \thetay)} - 
\cos{\balpha(n)} \frac{\sin{(\thetay - \hthetae(n))}}{\sin{\phi(n)}}$ when $\thetay \in [\hthetae(n) + \phi(n), \alpha + \thetae] \cup [\pi + \hthetae(n) + \phi(n), \pi + \alpha + \thetae]$,
    \item \label{item:case_C2} $\perf \!=\! 1 - 
\cos{\balpha(n)} \frac{\sin{(\thetay - \hthetae(n))}}{\sin{\phi(n)}}$ when $\thetay \in [\alpha + \thetae, \pi -\alpha + \thetae] \cup [\pi +\alpha + \thetae, 2\pi -\alpha + \thetae]$,
    \item \label{item:case_F2} $\perf \!=\! \cos{(\thetay + \alpha - \thetae)} - 
\cos{\balpha(n)} \frac{\sin{(\thetay - \hthetae(n))}}{\sin{\phi(n)}}$when $\thetay \in  [\pi -\alpha + \thetae, \pi + \hthetae(n) -\phi(n)] \cup [-\alpha+ \thetae, \hthetae(n) - \phi(n)]$,
    \item \label{item:case_G2} $\perf \!=\! \cos{(\thetay + \alpha - \thetae)} - 
\cos{\balpha(n)}\cos{(\thetay  + \phi(n) - \hthetae(n))}$  when $\thetay \in  [\pi + \hthetae(n) -\phi(n), \pi +\hthetae(n)] \cup [\hthetae(n) - \phi(n), \hthetae(n)]$,
    \item \label{item:case_H2} $\perf \!=\! \cos{(\thetay + \alpha - \thetae)} - 
\cos{\balpha(n)}\cos{(\thetay  - \phi(n) - \hthetae(n))}$ when $\thetay \in  [\pi + \hthetae(n), \pi +\thetae] \cup [\hthetae(n), \thetae]$.
\end{enumerate}
\end{enumerate}
\end{lemma}
\begin{remark}
\label{remark_1}
The cases above can be classified into one of the following three categories: 
\begin{enumerate}[label=Category \arabic*., leftmargin=24mm, labelsep=1em]
    \item \textbf{Any uncertainty hurts:} Here, the optimal performance of the robust QCQP in \cref{eq_qcqp_limited_infinite_worse} matches the optimal performance in \cref{eq_qcqp} only when all the uncertainty is removed, i.e., when $n \to N$. When the uncertainty parameters are not too large, i.e., $(\hre(n) + \Delta(n)) \sin{\phi(n)} \leq \sqrt{\varepsilon}$, \cref{item:case_A}, \cref{item:case_B}, \cref{item:case_F}, \cref{item:case_G}, and \cref{item:case_H} fall into this category. When the uncertainty parameters are large, i.e., $(\hre(n) + \Delta(n)) \sin{\phi(n)} \geq \sqrt{\varepsilon}$, \cref{item:case_A}, \cref{item:case_B}, \cref{item:case_D}, \cref{item:case_E}, and \cref{item:case_F} fall into this category.
    \item \textbf{Some uncertainty does not hurt:} Here, the optimal performance of the robust QCQP in \cref{eq_qcqp_limited_infinite_worse} matches  the optimal performance in \cref{eq_qcqp} when some uncertainty is removed, i.e., by collecting some additional labeled sensitive attributes. When the uncertainty parameters are not too large, i.e., $(\hre(n) + \Delta(n)) \sin{\phi(n)} \leq \sqrt{\varepsilon}$, \cref{item:case_C} and \cref{item:case_E} fall into this category. When the uncertainty parameters are large, i.e., $(\hre(n) + \Delta(n)) \sin{\phi(n)} \geq \sqrt{\varepsilon}$, \cref{item:case_C} fall into this category.
    \item \textbf{Uncertainty does not hurt:} Here, the optimal performance of the robust QCQP in \cref{eq_qcqp_limited_infinite_worse} matches  the optimal performance in \cref{eq_qcqp} without removing any uncertainty, i.e., without collecting any additional labeled sensitive attributes. This only happens in \cref{item:case_D} where $\perf = 0$ when the uncertainty parameters are not too large, i.e., $(\hre(n) + \Delta(n)) \sin{\phi(n)} \leq \sqrt{\varepsilon}$. Such a situation arises when $\Byx$ and $\BexEstimated$ are very close to being perpendicular to each other. In other words, when $\Byx$ and $\BexEstimated$ are close to being independent, our proposed robust QCQP achieves optimal performance while ensuring a strict fairness guarantee. We call such a phenomenon as ``free fairness'' (see \cref{coro_free_fairness}).
\end{enumerate}
\end{remark}

\section{Monotonic performance with new labeled sensitive attributes}
\label{sec_mono_perf}
As in the previous section, we consider the case where uncertainty only stems from sensitive attributes missing at random. Then, we show that the optimal objective in \cref{eq_qcqp_limited_approx} either monotonically increases or coincides with the optimal objective in \cref{eq_qcqp} whenever the uncertainty set monotonically decreases with $n$.

To that end, below, we express the uncertainty set $\cS$ and the associated uncertainty parameters $\Delta$ and $\phi$ as a function of $n$. 
\newcommand{\uncertaintyfuncn}{Uncertainty as a function of $n$}
\begin{proposition}[\uncertaintyfuncn]\label{prop_conc_bound}
Fix any $\delta > 0$. Given $n$ samples $\normalbraces{(e^{(i)}, \svbx^{(i)})}_{i \in [n]}$ of $(\rve, \rvbx)$, let 
\begin{align}
    \BexEstimated \defn \SigmaEE^{-1/2} \SigmaEXestimated \SigmaXX^{-1/2} \stext{with} \SigmaEXestimated \!\defn\! \frac{1}{n}\!\sum_{i \in [n]}\! e^{(i)} \svbx^{(i)}.
\end{align}
Then, $\Bex \in \balluncn$, with probability $1-\delta$, where
\begin{align}
    \radius(n) \defn \stwonorm{\SigmaXX^{-1/2}} \frac{c \sqrt{\sigma_e}\max_{i \in [d]}\normalbraces{ \sigma_{i} }}{n \sqrt{d}}\log \frac{4}{\delta},
\end{align}
with $\sigma_{1}^2, \sigma_{2}^2, \sigma_e^2$ denoting variances of $\rvbx_1$, $\rvbx_2$, $\rve$, respectively, and $c$ is a universal constant. Further, for $d = 2$, $\Bex \in \cSn$, with probability at least $1-\delta$, where
\begin{align}
    \phi(n) \defn \sin^{-1}\biggparenth{\frac{\radius(n)}{\stwonorm{\BexEstimated}}} \stext{and} \Delta(n) \defn \radius(n).
\end{align}
\end{proposition}

\newcommand{\monotonic}{Monotonic performance}
\begin{proof}
For any sub-Gaussian (sub-exponential) random variable, we denote its sub-Gaussian (sub-exponential) norm by $\subGnorm{\cdot} (\subEnorm{\cdot})$. For every $i \in [d]$, we note that $\rvx_i$ is a sub-Gaussian random variable with sub-Gaussian norm  $\subGnorm{\rvx_i} = c \sigma_i$ where $\sigma_i^2 = \SigmaXXi$. Similarly, $\rve$ is a sub-Gaussian random variable with sub-Gaussian norm  $\subGnorm{\rve} = c \sigma_e$ where $\sigma_e^2 = \SigmaEE$. Then, from \citet[Lemma. 2.7.7]{vershynin2018high}, for every $i \in [d]$, $\rvx_i \rve$ is a sub-exponential random variable with sub-exponential norm  $\subEnorm{\rvx_i \rve} = c \sigma_i \sigma_e$. Therefore, from Bernstein's inequality \citep[Corollary 2.8.3]{vershynin2018high}, with probability $1-\delta$, we have $\sinfnorm{\SigmaEXestimated - \SigmaEX} \leq t$ whenever 
$t \geq \frac{c \sigma_e \max_{i \in [d]}\sigma_i}{n}\log \frac{2d}{\delta}$.
As a result, we have $\stwonorm{\SigmaEXestimated - \SigmaEX} \leq t$ with probability $1-\delta$ whenever $t \geq \frac{c \sigma_e \max_{i \in [d]}\sigma_i}{n \sqrt{d}}\log \frac{2d}{\delta}$. Conditioning on this event, we have, with probability $1 - \delta$,
\begin{align}
    \stwonorm{\BexEstimated - \Bex} \sequal{(a)} \stwonorm{\SigmaEE^{-1/2} (\SigmaEXestimated - \SigmaEX) \SigmaXX^{-1/2}} & \sless{(b)} \frac{1}{\sqrt{\sigma_e}} \stwonorm{\SigmaEXestimated - \SigmaEX} \stwonorm{\SigmaXX^{-1/2}} \\
    & \leq \frac{1}{\sqrt{\sigma_e}} \stwonorm{\SigmaXX^{-1/2}} t,
\end{align}
where $(a)$ follows from \cref{def_ccm} and $(b)$ follows because induced matrix norms are sub-multiplicative. Then, letting $d = 2$ and $\radius(n) = \SigmaEE^{-1/2} \sinfnorm{\SigmaXX^{-1/2}} t$, it is easy to verify $\Bex \in \cSn$ whenever
\begin{align}
    \Delta(n) = \radius(n) \qtext{and}  \phi(n) = \sin^{-1}\biggparenth{\frac{\radius(n)}{\stwonorm{\BexEstimated}}}.
\end{align}
\end{proof}

\subsection{\monotonic\   with decrease in uncertainty set}
Now, in the following theorem, we show that the optimal objective of the robust QCQP with 3 constraints in \cref{eq_qcqp_limited_approx} either monotonically increases or coincides with the optimal objective of the QCQP in \cref{eq_qcqp} whenever the uncertainty set $\cSn$ monotonically decreases with $n$.  
\begin{theorem}[\monotonic]\label{thm_power_sample}
For any number of labeled sensitive attributes $n$, let $\cSn$ denote the uncertainty set containing $\Bex$. Let $\limitedinfrvba(n)$ denote the optimal solution $\rvba$ in \cref{eq_qcqp_limited_approx} as a function of $n$. If $\cSn[n+1] \subset \cSn$, then 
\begin{align}
\biginner{\limitedinfrvba(n+1)}{\Byx}^2 & > \biginner{\limitedinfrvba(n)}{\Byx}^2 \qtext{or}\\
\biginner{\limitedinfrvba(n+1)}{\Byx}^2 & = \biginner{\limitedinfrvba(n)}{\Byx}^2 = \biginner{\truervba}{\Byx}^2,
\end{align}
where $\truervba$ is the optimal solution of the QCQP in \cref{eq_qcqp}.
\end{theorem}
\label{sec_proof_monotonic}
Fix any $\varepsilon > 0$. For any $\B$, let $\ParameterSetTrue(\B) \defn \normalbraces{\rvba: \normalinner{\rvba}{\B}^2 \leq \varepsilon \stext{and} \rvba \in \ball}$ be the set of all $\rvba$ satisfying the fairness constraint in \cref{eq_qcqp} w.r.t $\B$. For a given $\B$, $\ParameterSetTrue(\B)$ can be constructed as described in the proof of \cref{lemma_optimal_a_qcqp} in \cref{subsec_characterizing_qcqp}. See \cref{fig_proofs_qcqp} for reference.

Now, using $\cSn[n+1] \subset \cSn$, it is easy to see that
$\hre(n+1) < \hre(n)$, $\bigabs{\hthetae(n+1) + \phi(n+1)} < \bigabs{\hthetae(n) + \phi(n)}$, $\bigabs{\hthetae(n+1) - \phi(n+1)} < \bigabs{\hthetae(n) - \phi(n)}$, and $\bigabs{\cos{\phi(n+1)}} < \bigabs{\cos{\phi(n)}}$. Then, it follows that the feasible space increases with $n$, i.e., $\ParameterSetTrue(\Bex^{(1)}(n+1)) \cap \ParameterSetTrue(\Bex^{(2)}(n+1)) \cap \ParameterSetTrue(\Bex^{(3)}(n+1)) \supset \ParameterSetTrue(\Bex^{(1)}(n)) \cap \ParameterSetTrue(\Bex^{(2)}(n)) \cap \ParameterSetTrue(\Bex^{(3)}(n))$. As a result, the optimal objective in \cref{eq_qcqp_limited_approx} at $n+1$ is either less than the optimal objective in \cref{eq_qcqp_limited_approx} at $n$ or equal to the optimal objective in \cref{eq_qcqp_limited_approx} at $n$. It remains to show that if the optimal objective in \cref{eq_qcqp_limited_approx} at $n+1$ equals the optimal objective in \cref{eq_qcqp_limited_approx} at $n$, then it also equals the optimal objective of the QCQP in \cref{eq_qcqp}. This follows directly from the expressions for $\perf$ in various cases of \cref{lemma_power_new_sample}.

\section{Additional experimental results}
\label{sec_add_exp_details}
In this section, we provide additional empirical results for Gaussian and real-world data as well as more implementation details. 

\subsection{Additional results for Gaussian data}
\label{sec_more_details_gaussian}
\textbf{Covariance matrices.} First, we provide the covariance matrices of $(\rvbx, \rvy, \rve)$ used in the experiments in \cref{subsec_gaussian_expts}:
\[
    \Sigma^{\text{gen}}_2  = \left[\begin{array}{cc|c|c}
    1& 0.1& 0.5& 0.4\\
    0.1& 1& 0.5& 0.25\\
    \hline
    0.5& 0.5& 1& 0.75\\
    \hline
    0.4& 0.25& 0.75& 1	
    \end{array}\right]
    \qtext{   }
    \Sigma^{\text{fair}}_2  =  \left[\begin{array}{cc|c|c}
    1& 0.1& 0.5& 0.05\\
    0.1& 1& 0.05& 0.25\\
    \hline
    0.5& 0.05& 1& 0.75\\
    \hline
    0.05& 0.25& 0.75& 1	
    \end{array}\right]
\]
The choice of $\Sigma_2^{\text{fair}}$ is deliberate to demonstrate the free-fairness behavior.
As pointed out in Category 3 of \cref{remark_1}, the free-fairness behavior may happen when $\Byx$ and $\BexEstimated$ are almost perpendicular.
It is easy to check that this can be achieved by the following steps, for example:
(1) Strongly correlate $\rve$ with $\rvx_1$, and $\rvy$ with $\rvx_2$.
(2) Set the covariance of $\rvbx$ to be close to the identity matrix.\\

\noindent \textbf{Results for $d=3$.} Next, we provide results for $d = 3$ where we use the following covariance matrices:
\[
        \Sigma^{\text{gen}}_3  = \left[\begin{array}{ccc|c|c}
    1& 0.1& 0.5& 0.5& 0.4\\
    0.1& 1& 0.5& 0.5& 0.25\\
    0.5& 0.5& 1& 0.2& 0.2\\
    \hline
    0.5& 0.5& 0.2& 1& 0.75\\
    \hline
    0.4& 0.25& 0.2& 0.75& 1	
    \end{array}\right]
    \qtext{ }
    \Sigma^{\text{fair}}_3  = \left[\begin{array}{ccc|c|c}
    1& 0.1& 0.5& 0.5& 0.05\\
    0.1& 1& 0.5& 0.05& 0.25\\
    0.5& 0.5& 1& 0.2& 0.2\\
    \hline
    0.5& 0.05& 0.2& 1& 0.75\\
    \hline
    0.05& 0.25& 0.2& 0.75& 1	
    \end{array}\right]
 \]
The reasoning behind the choice of $\Sigma_3^{\text{fair}}$ is the same as above.
As we see in \cref{fig_gaussian_general_d=3} (a) and \cref{fig_gaussian_general_d=3} (b), $\bss$ has significantly less fairness violations compared to $\texttt{Baseline}$ with a similar MSE even for $d = 3$.\\ 
\begin{figure*}[h]
    \centering
    \begin{tabular}{ccc}
    \includegraphics[width=0.3\linewidth,clip]{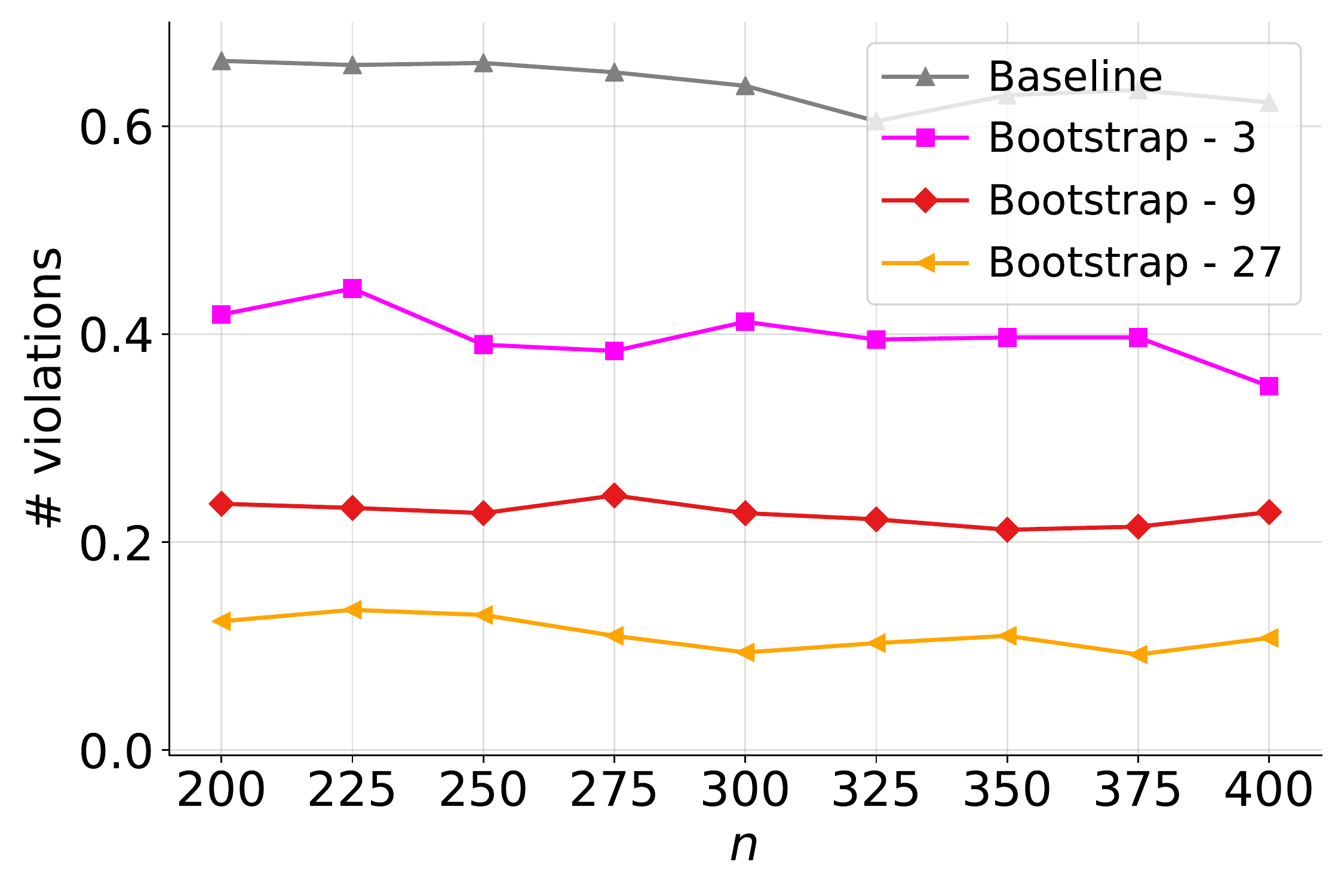}&
    \includegraphics[width=0.3\linewidth,clip]{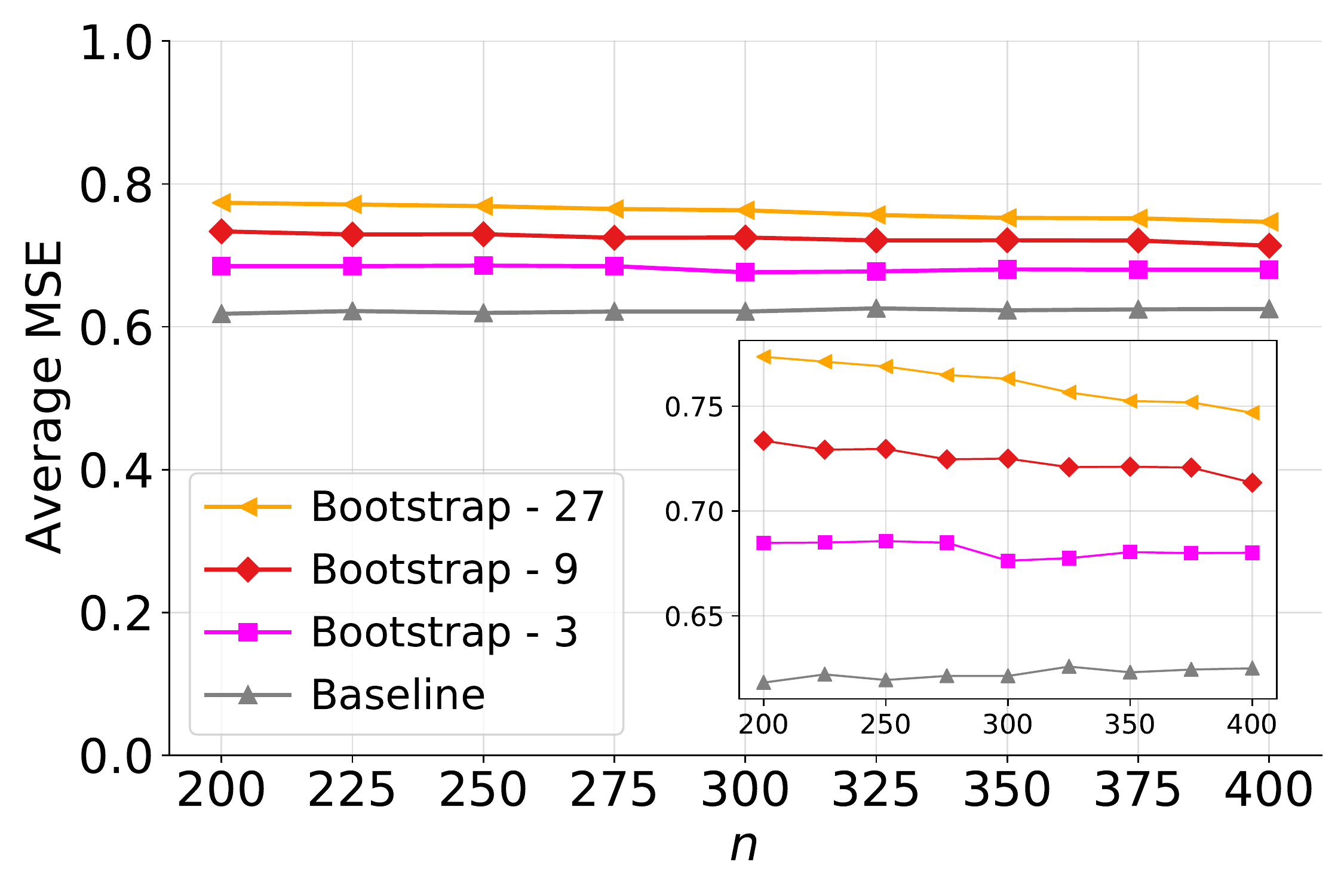}&
    \includegraphics[width=0.3\linewidth,clip]{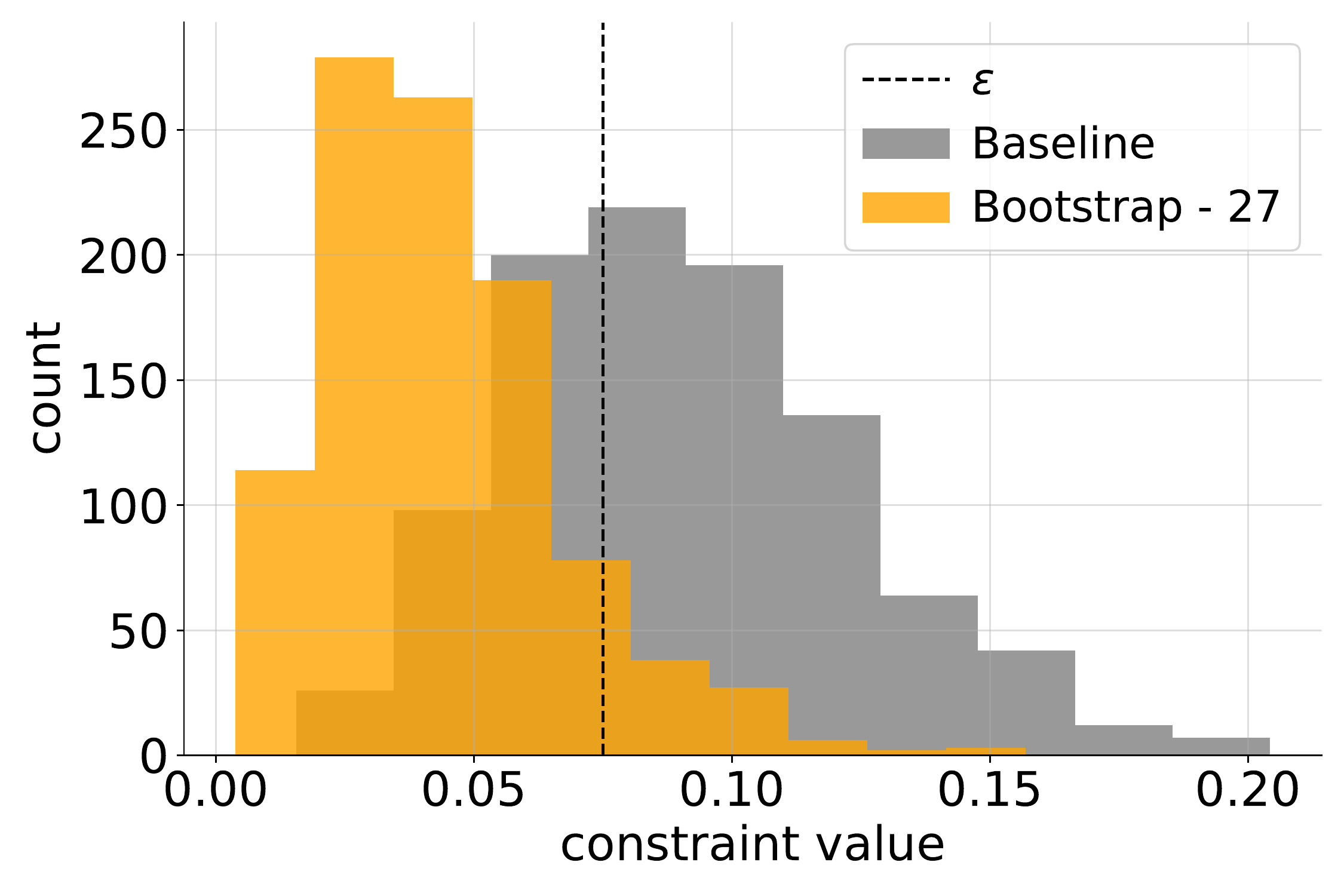}
    \\ & (a) $d = 3$, $\Sigma^{\star} = \Sigma^{\text{gen}}_3$, $\varepsilon = 0.075$ & \\
    \includegraphics[width=0.3\linewidth,clip]{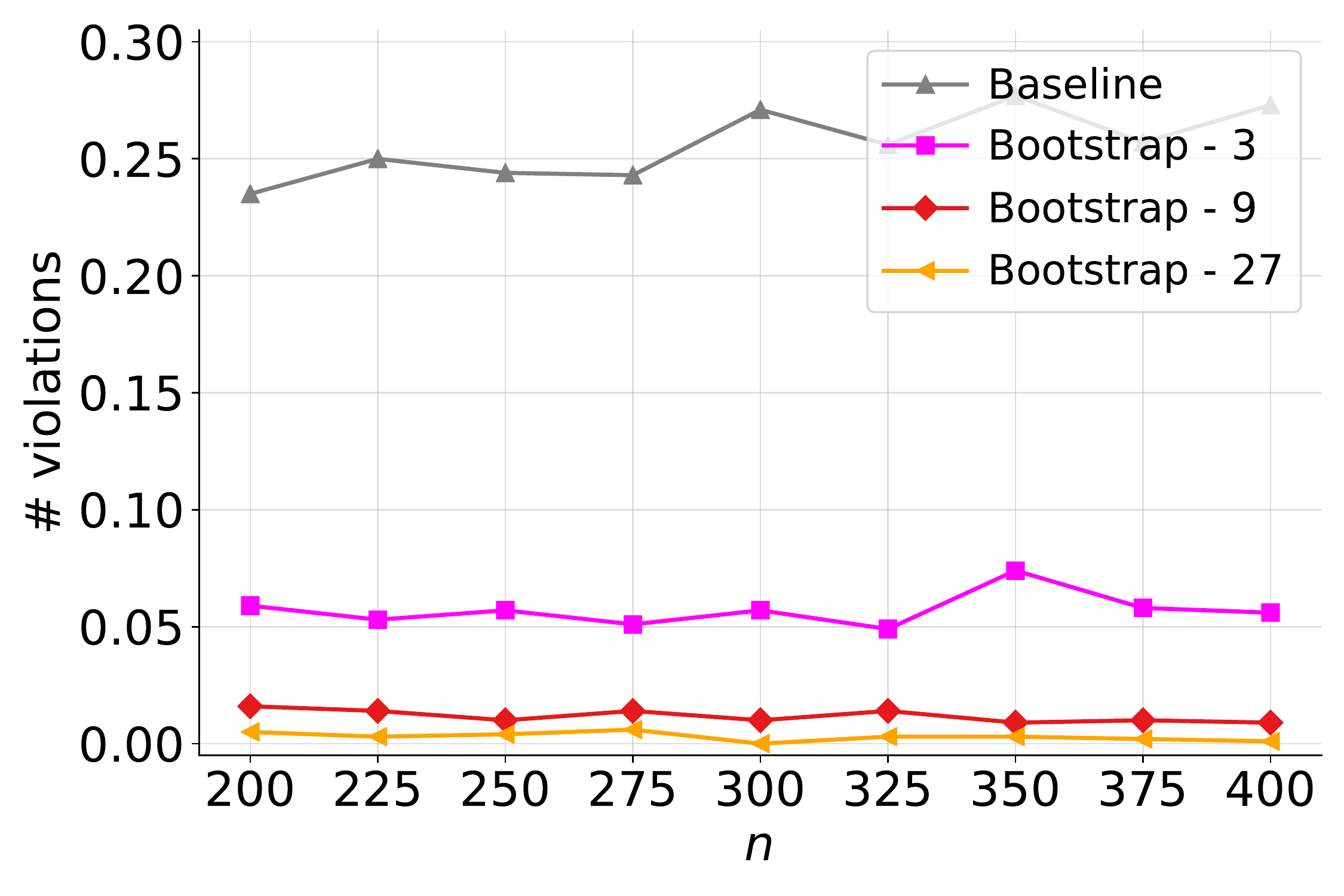}&
    \includegraphics[width=0.3\linewidth,clip]{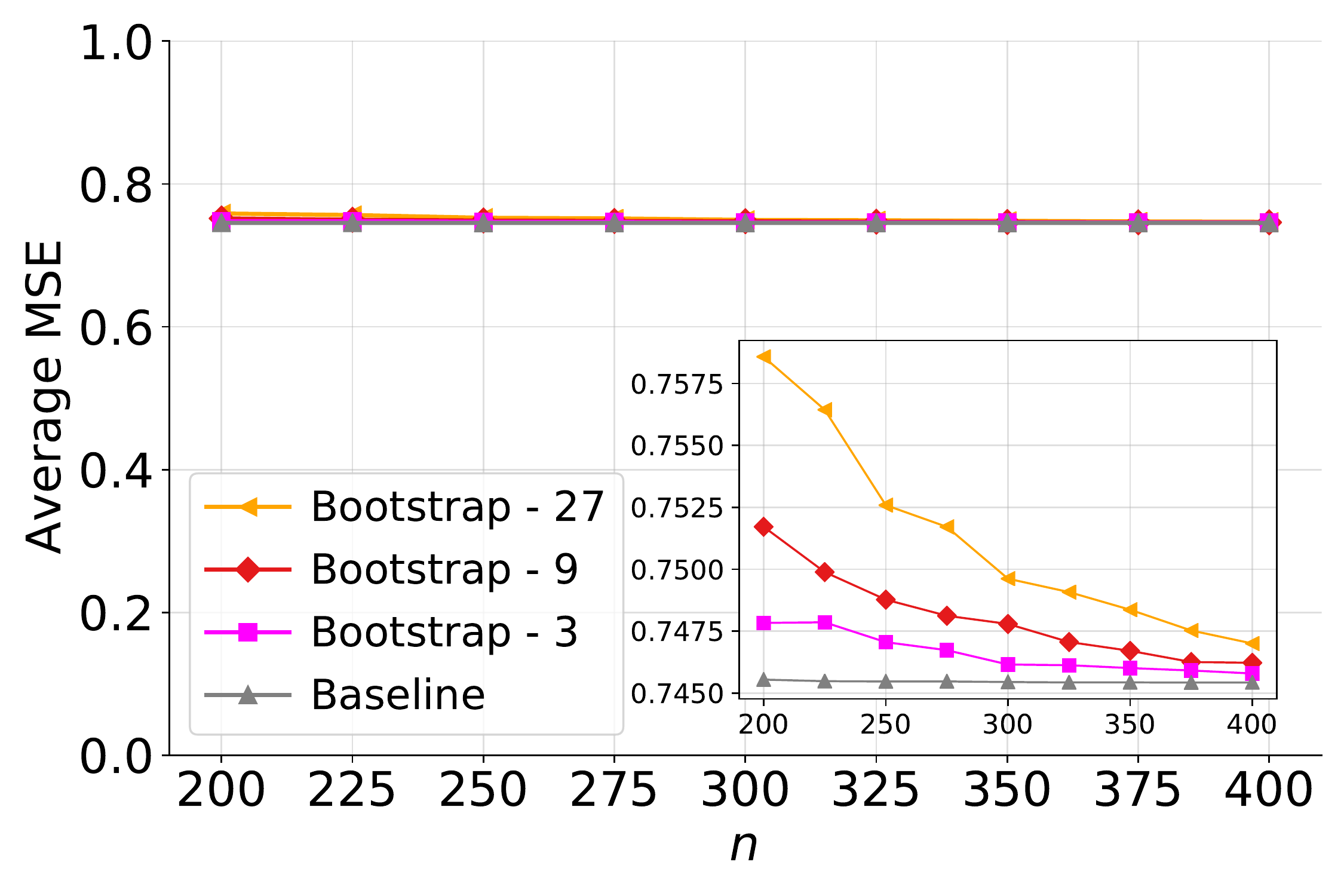}&
    \includegraphics[width=0.3\linewidth,clip]{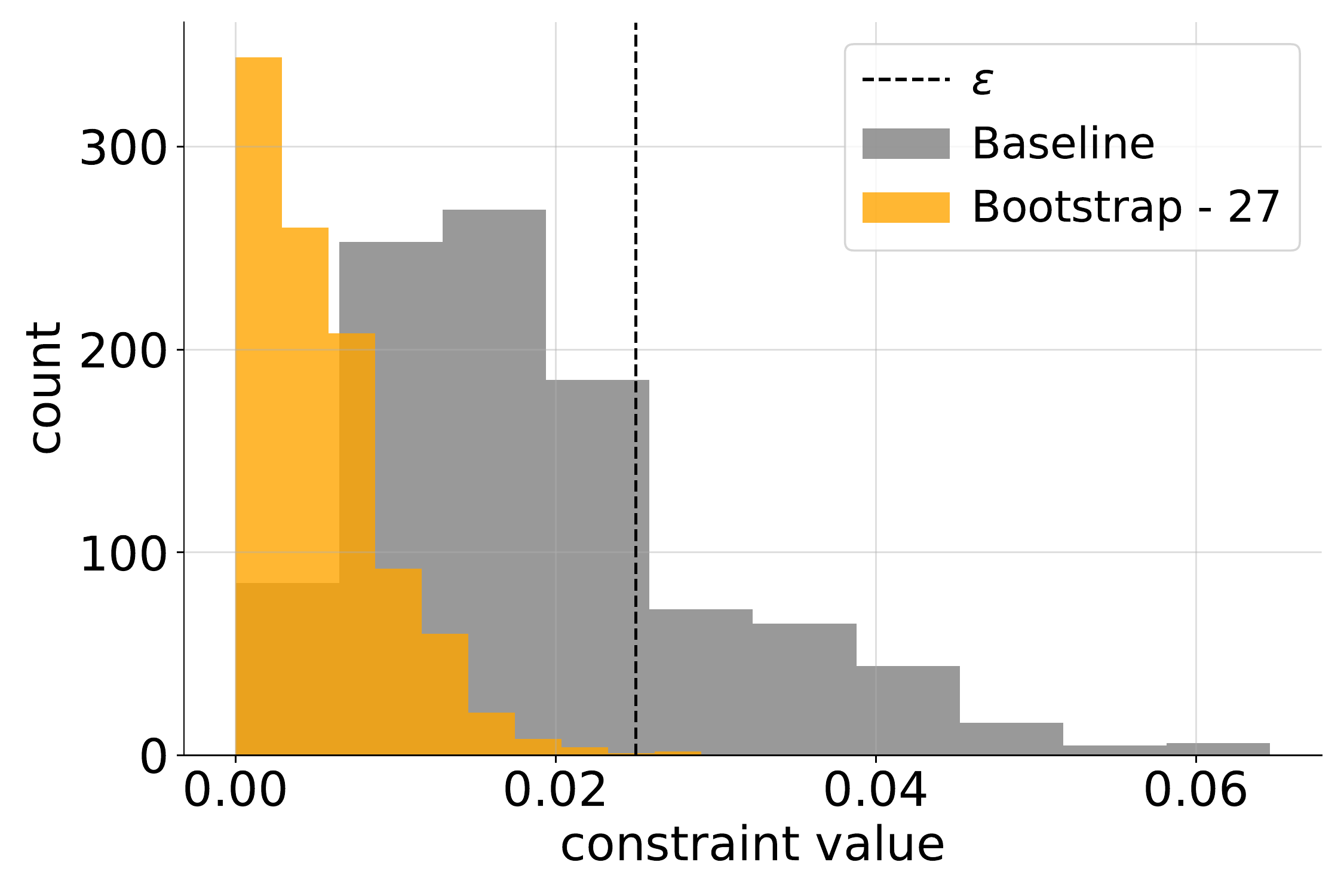}
    \\ & (b) $d = 3$, $\Sigma^{\star} = \Sigma^{\text{fair}}_3$, $\varepsilon = 0.025$ & 

    \end{tabular}
\caption{The performance of $\bss$ with $\s\in \{3,9,27\}$ and \texttt{Baseline} for $d=3$ and various $\Sigma^{\star}$. 
In the left column, we plot the fraction of violations of the fairness constraint $\normalinner{\rvba}{\Bex}^2 \leq \varepsilon$ vs. $n$; in the middle column, we plot average MSE vs. $n$; in the right column, we plot the histogram of the value of $\normalinner{\rvba}{\Bex}^2$ over 1,000 trials for $n = 250$.}
\label{fig_gaussian_general_d=3}
\end{figure*}

\noindent \textbf{A different choice of covariance matrices.} Next, we provide results with a different choice of covariance matrices:
\[   
    \Sigma^{\text{a}}_2  =  \left[\begin{array}{cc|c|c}
    1& 0.1& 0.5& 0.01\\
    0.1& 1& 0.01& 0.25\\
    \hline
    0.5& 0.01& 1& 0.75\\
    \hline
    0.01& 0.25& 0.75& 1	
    \end{array}\right]
    \qtext{ }
    \Sigma^{\text{a}}_3  = \left[\begin{array}{ccc|c|c}
    1& 0.1& 0.5& 0.5& 0.01\\
    0.1& 1& 0.5& 0.01& 0.25\\
    0.5& 0.5& 1& 0.2& 0.2\\
    \hline
    0.5& 0.01& 0.2& 1& 0.75\\
    \hline
    0.01& 0.25& 0.2& 0.75& 1	
    \end{array}\right]
\]
As we see in \cref{fig_gaussian_another}, the behavior is similar as in other scenarios.\\
\begin{figure*}[h!]
    \centering
    \begin{tabular}{ccc}
    \includegraphics[width=0.31\linewidth,clip]{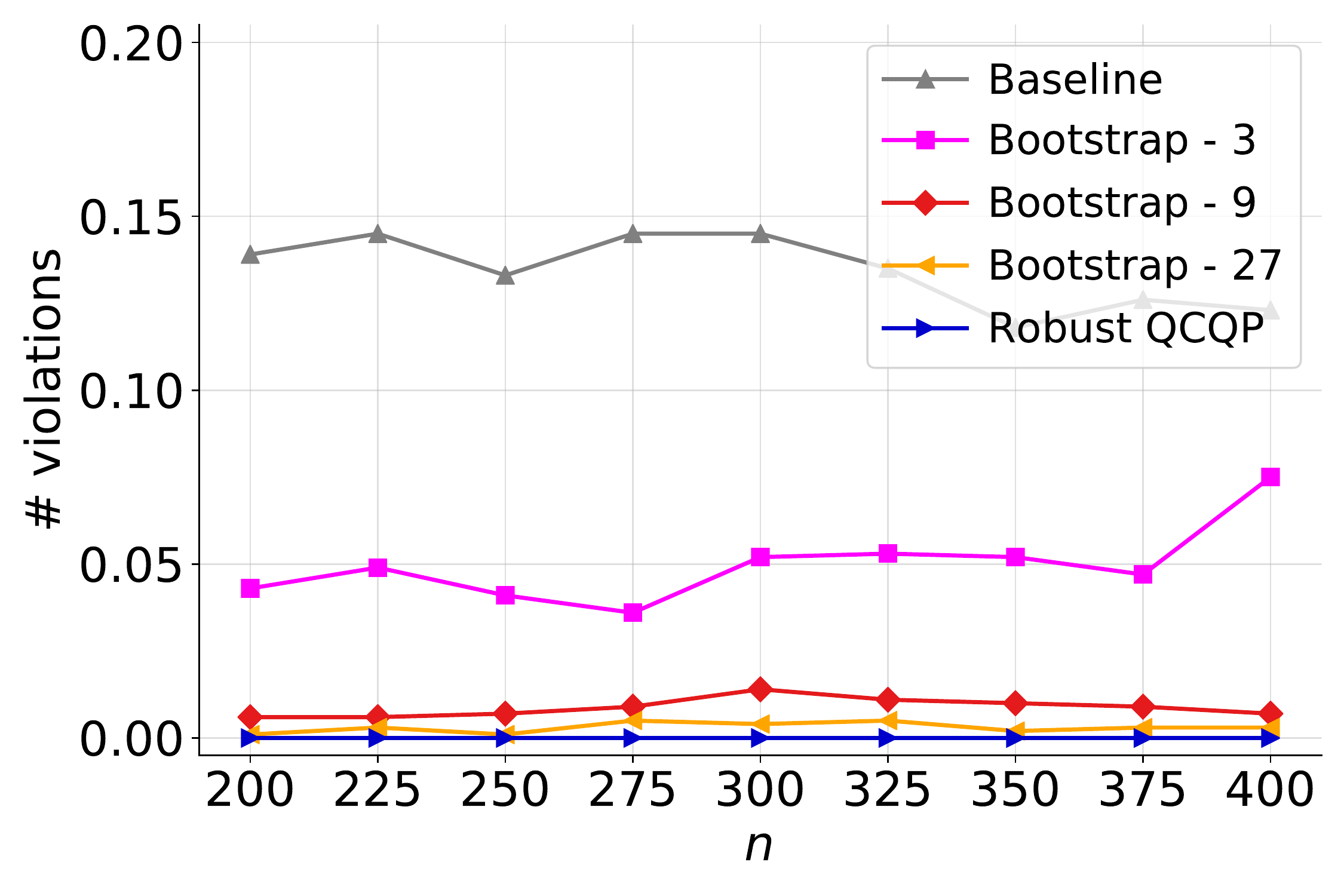}&
    \includegraphics[width=0.31\linewidth,clip]{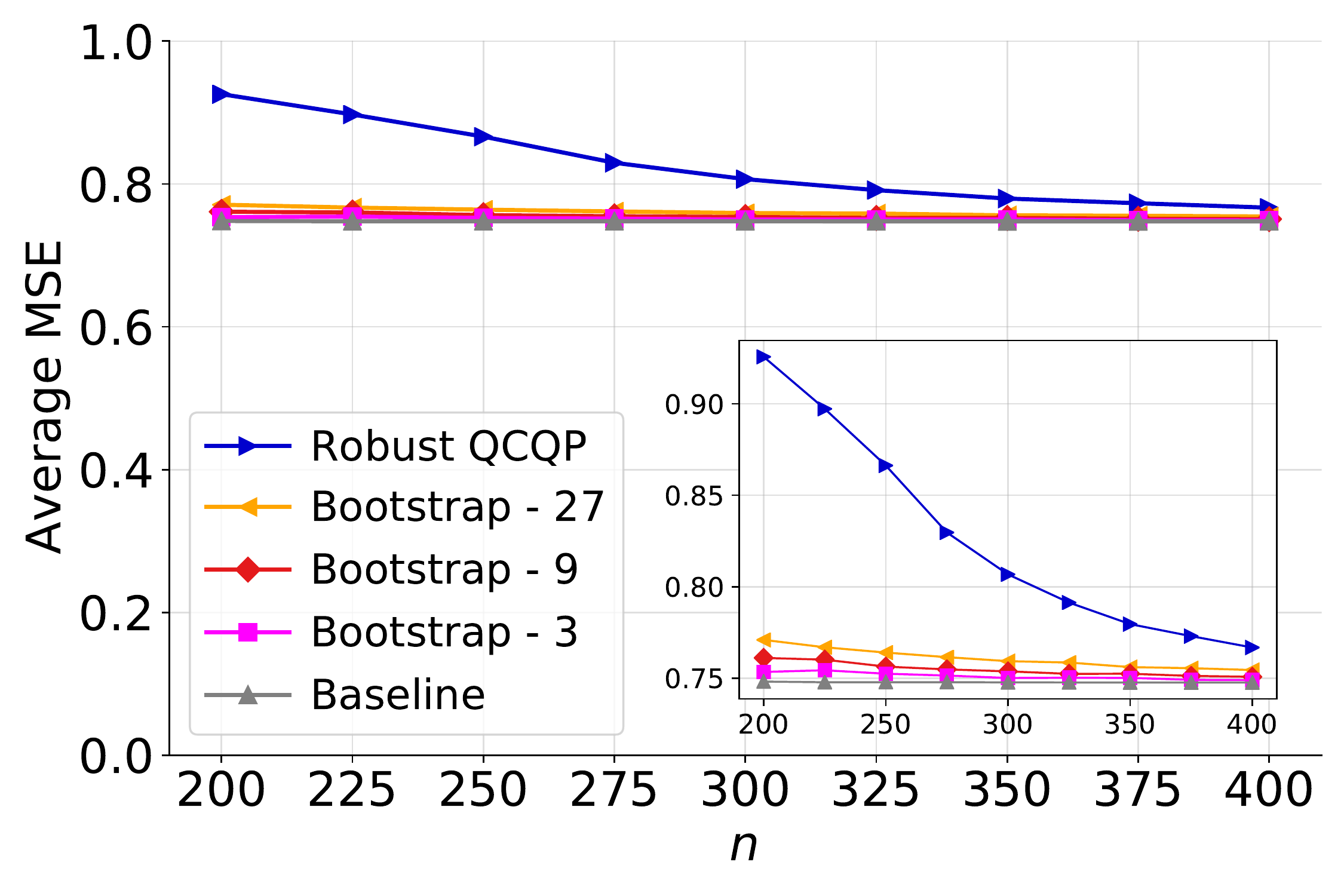}&
    \includegraphics[width=0.31\linewidth,clip]{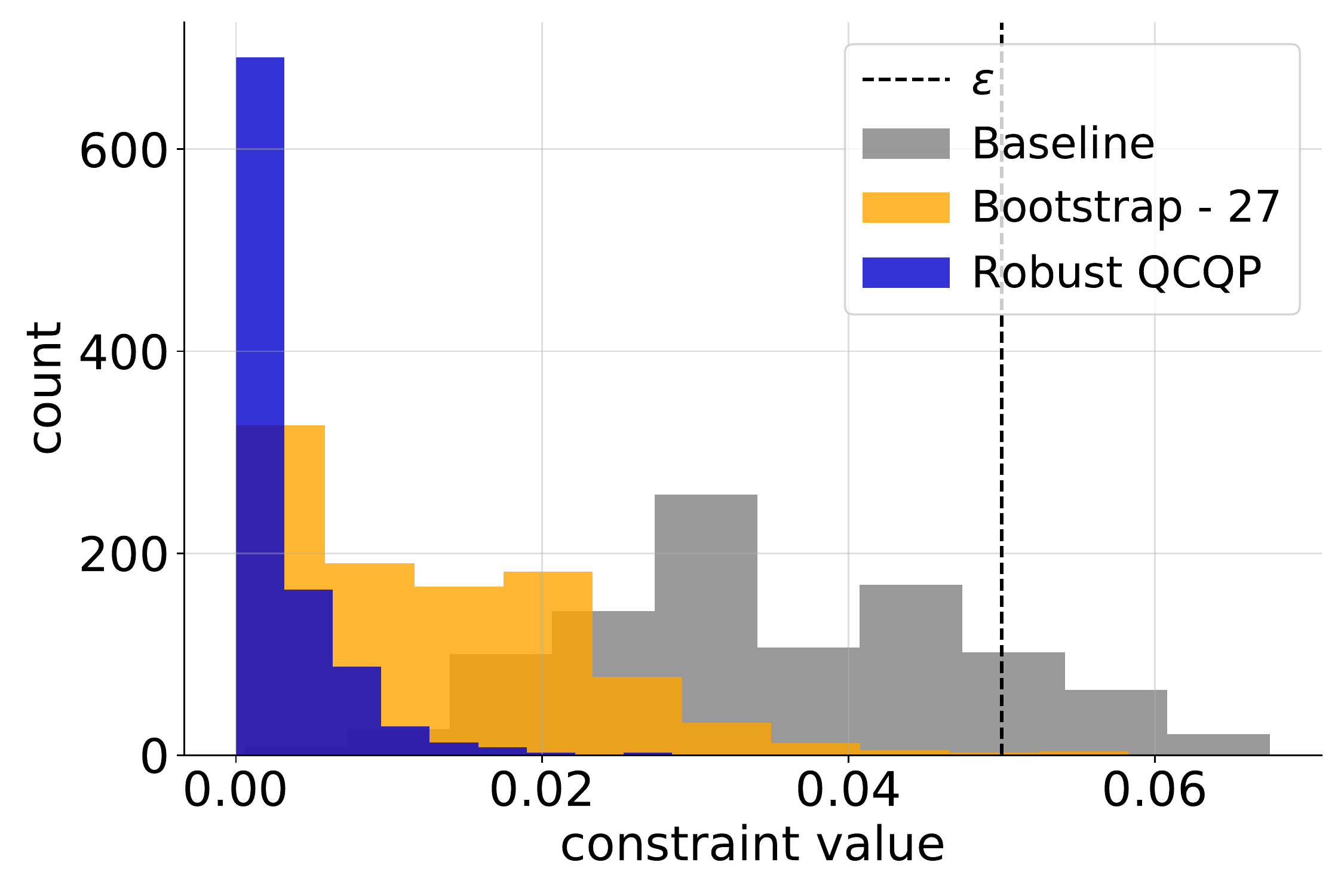}
    \\ & (a) $d = 2$, $\Sigma^{\star} = \Sigma^{\text{a}}_2$, $\varepsilon = 0.05$ &  \\
    \includegraphics[width=0.31\linewidth,clip]{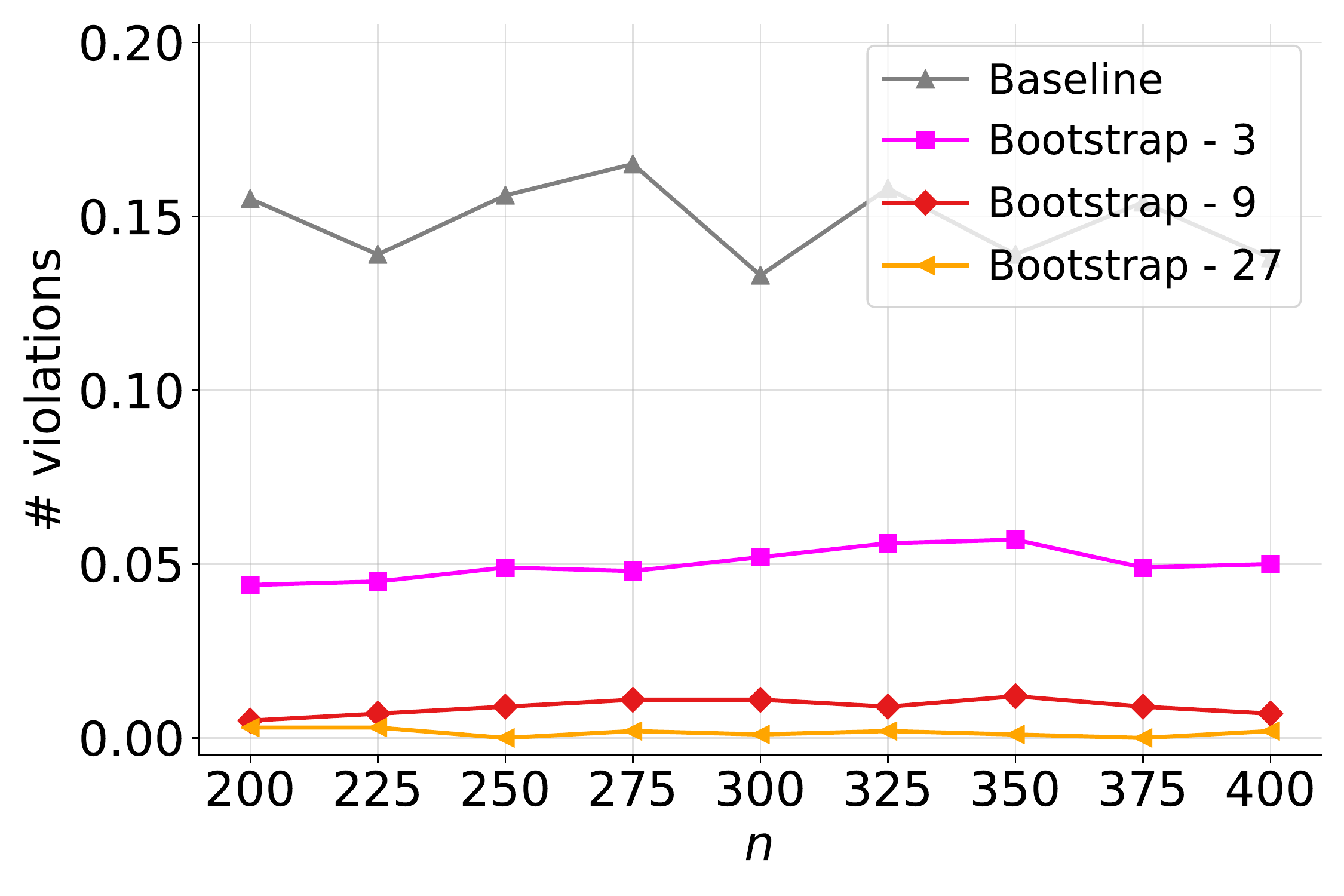}&
    \includegraphics[width=0.31\linewidth,clip]{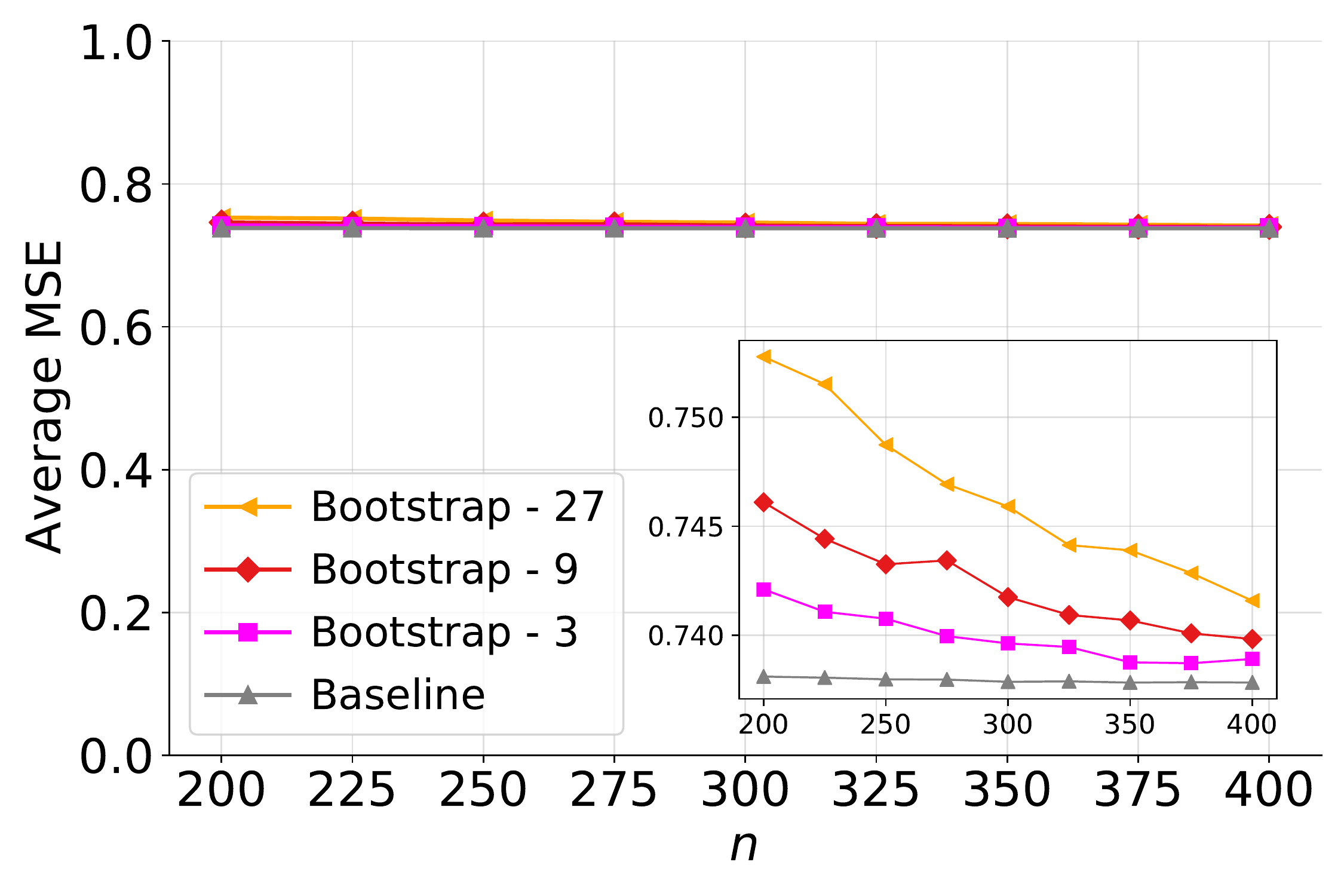}&
    \includegraphics[width=0.31\linewidth,clip]{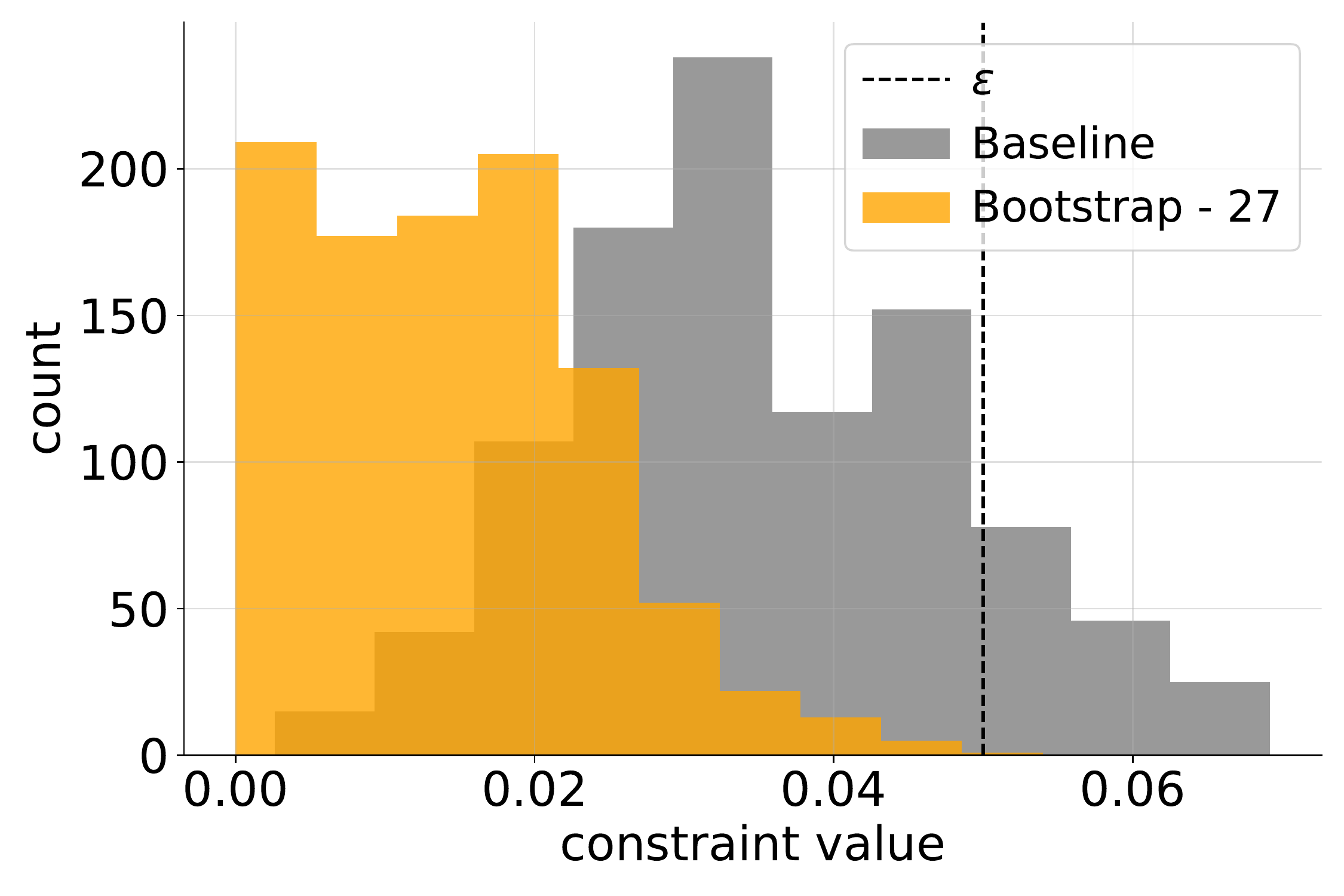}
    \\  & (b) $d = 3$, $\Sigma^{\star} = \Sigma^{\text{a}}_3$, $\varepsilon = 0.05$ & 
    \end{tabular}
\caption{The performance of robust QCQP in \cref{eq_qcqp_limited_approx} (for $d = 2$), $\bss$ with $\s\in\{3,9,27\}$, and \texttt{Baseline} for various $d$ and $\Sigma^{\star}$. In the left column, we plot the fraction of violations of the fairness constraint $\normalinner{\rvba}{\Bex}^2 \leq \varepsilon$ vs. $n$; in the middle column, we plot average MSE vs. $n$; in the right column, we plot the histogram of the value of $\normalinner{\rvba}{\Bex}^2$ over 1000 trials for $n = 250$.}
\label{fig_gaussian_another}
\end{figure*}

\noindent \textbf{Results for uncertainty due to noise.}
In \cref{subsec_gaussian_expts}, we considered uncertainty arising due to limited number of sensitive attributes. Here, we consider uncertainty arising due to noise. Specifically, we generate synthetic data using the covariance matrix $\Sigma^{\text{fair}}_2$ and conduct two sets of experiments as follows. In the first experiment, we add independent zero-mean Gaussian noise, i.e., $\mathcal{N}(0,\sigma^2)$, with varying noise level, i.e., $\sigma = 1, \cdots, 10 $, to the sensitive attributes. Such kind of noise is common when the sensitive attributes need to be privatized. 
In the second experiment, we add independent zero-mean Gaussian noise, i.e., $\mathcal{N}(0,\sigma^2)$, with varying noise level, i.e., $\sigma = 1, \cdots, 10 $, to the sensitive attributes only if the sensitive attribute magnitude is larger than a certain threshold, i.e., $|\rve| \geq \SigmaEE$. Such data-dependent noise simulates noisy responses when sensitive attributes are collected via a survey, e.g., minority groups pretend to be the majority due to fear of discrimination. In our experiment, samples with sensitive attribute magnitudes larger than the threshold have noisy responses. We provide the results for both experiments in ~\cref{fig_gaussian_noise} where we observe that $\bss$ achieves significantly fewer fairness violations than $\texttt{Baseline}$ while maintaining  comparable MSE.

\begin{figure*}[h]
    \centering
    \begin{tabular}{cc}
    \includegraphics[width=0.31\linewidth,clip]{Gaussian_noise_indep_fairness.pdf}\qquad
    \includegraphics[width=0.31\linewidth,clip]{Gaussian_noise_indep_mse.pdf}
    \\ \hspace{1cm} (a) Independence Noise, $d = 2$, $\Sigma^{\star} = \Sigma^{\text{fair}}_2$, $\varepsilon = 0.025$ \\
    \includegraphics[width=0.31\linewidth,clip]{Gaussian_noise_dep_fairness.pdf}\qquad
    \includegraphics[width=0.31\linewidth,clip]{Gaussian_noise_dep_mse.pdf}
    \\ \hspace{1cm} (b) Dependent Noise, $d = 2$, $\Sigma^{\star} = \Sigma^{\text{fair}}_2$, $\varepsilon = 0.025$ \\
    \end{tabular}
\caption{The performance of $\bss$ with $\s\in\{3,9,27\}$, and \texttt{Baseline}. In the left column, we plot the fraction of violations of the fairness constraint $\normalinner{\rvba}{\Bex}^2 \leq \varepsilon$ vs. the level of noise added to the sensitive attributes; in the right column, we plot average MSE vs. the level of noise added to the sensitive attributes.}
\label{fig_gaussian_noise}
\end{figure*}

\subsection{Additional results for real-world data}
\label{sec_more_details_real}
\textbf{Dataset description.} The Adult data~\citep{lantz2019machine} (\url{https://archive.ics.uci.edu/ml/datasets/adult}) considers predicting whether an individual's income is more than \$50,000  from the 1994 Census database using 14 demographic features such as age, education, marital status, and country of origin. 
The sensitive attribute is sex. During pre-processing, we remove the sensitive attribute from the set of input features and discard rows with any missing data. The data has 48,842 samples, with 30,527 males and 14,695 females. We use a train-test split ratio of 0.72:0.28.

The Insurance data~\citep{lantz2019machine} (\url{https://www.kaggle.com/datasets/teertha/ushealthinsurancedataset}) considers predicting the total annual medical expenses of individuals using 5 demographic features from the U.S. Census Bureau, such as BMI, number of children, and age. The sensitive attribute is sex. During pre-processing, we remove the sensitive attribute from the set of input features and perform normalization. The data has 1,338 samples with 676 males and 662 females. We use a train-test split ratio 0.8:0.2.

The Crime dataset~\citep{redmond2002data} (\url{https://archive.ics.uci.edu/ml/datasets/communities+and+crime}) considers predicting the number of violent crimes per 100K population using socio-economic information of communities in the U.S. The sensitive attribute is the percentage of people belonging to a particular race in the community. During pre-processing, we drop all the samples with the value of sensitive attribute less than $5\%$ to remove any outliers. We also remove the non-predictive attributes and the sensitive attribute from the set of input features, and normalize all attributes to the standardized range of $[0, 1]$. The resulting data has 1,112 samples, and we use a train-test split ratio 0.8:0.2.\\

\noindent \textbf{Implementation details.} We use two-layer fully connected networks in all our experiments. For all hidden layers, we use the selu activation function. For the output layer, we use softmax activation for classification and no activation for regression. We use 80 (50) units in the hidden layer and train the network for 30 (200) epochs for classification (regression). Empirically, we do not observe any substantial improvement in performance for $\s\ge 5$, and thus report for a single $\s$ for each dataset. The hyperparameters $\lambda, \lambda_1\, \cdots, \lambda_{\s}$ are initialized to $10$ ($5$) for classification (regression). We optimize the model parameters and the hyperparameters $\lambda, \lambda_1\, \cdots, \lambda_{\s}$ using different optimizers. For model optimization, we use Adam optimizer with batch size $128$ $(100)$, initial learning rate $10^{-3}$ $(10^{-4})$, and weight decay $0$ $(0.01)$ for classification (regression). For optimizing the hyperparameters, we use SGD optimizer with learning rate $10^{-2}$. The batch sizes for Adult, Crime, and Insurance datasets are set to be 128, 100, and 128, respectively. All experiments are implemented in PyTorch using Tesla V100 GPUs with 32 GB memory.\\

\noindent \textbf{Results of different levels of uncertainty.} In \cref{plots_real_different}, we provide results for independence notion of fairness with $n = 200$ for Adult dataset, $\sigma = 0.25$ for the Crime dataset, and $n = 20$ for Insurance dataset, respectively, where $n$ denotes the number of sensitive attributes kept out of $N$ and  $\sigma$ decides the amount of noise added to the sensitive attributes. The observations are consistent to that in \cref{plots_real}(a), (b), and (c).
\begin{figure}[ht]
    \centering
    \begin{tabular}{ccc}
    \includegraphics[width=0.31\linewidth,clip]{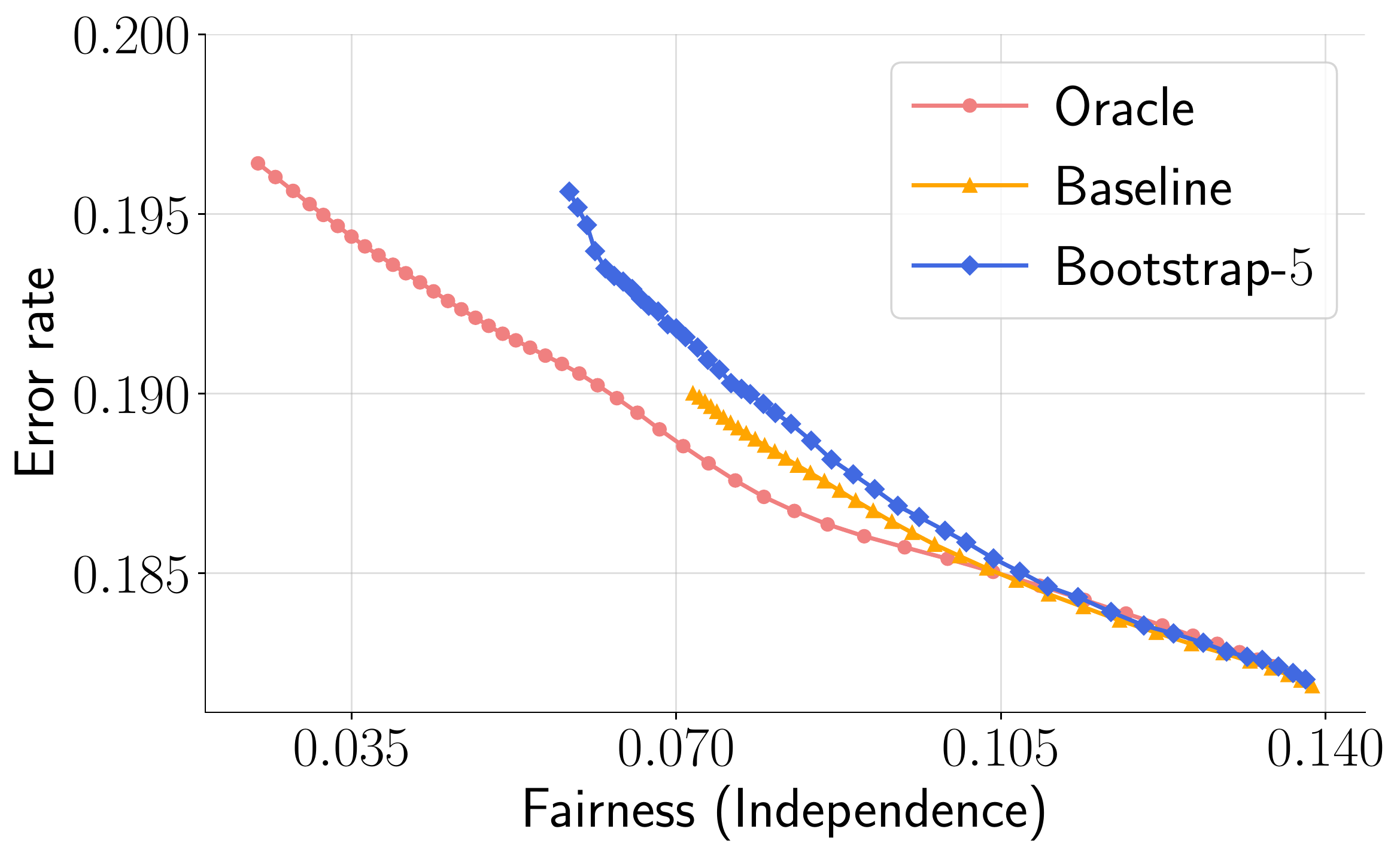}&
    \includegraphics[width=0.31\linewidth,clip]{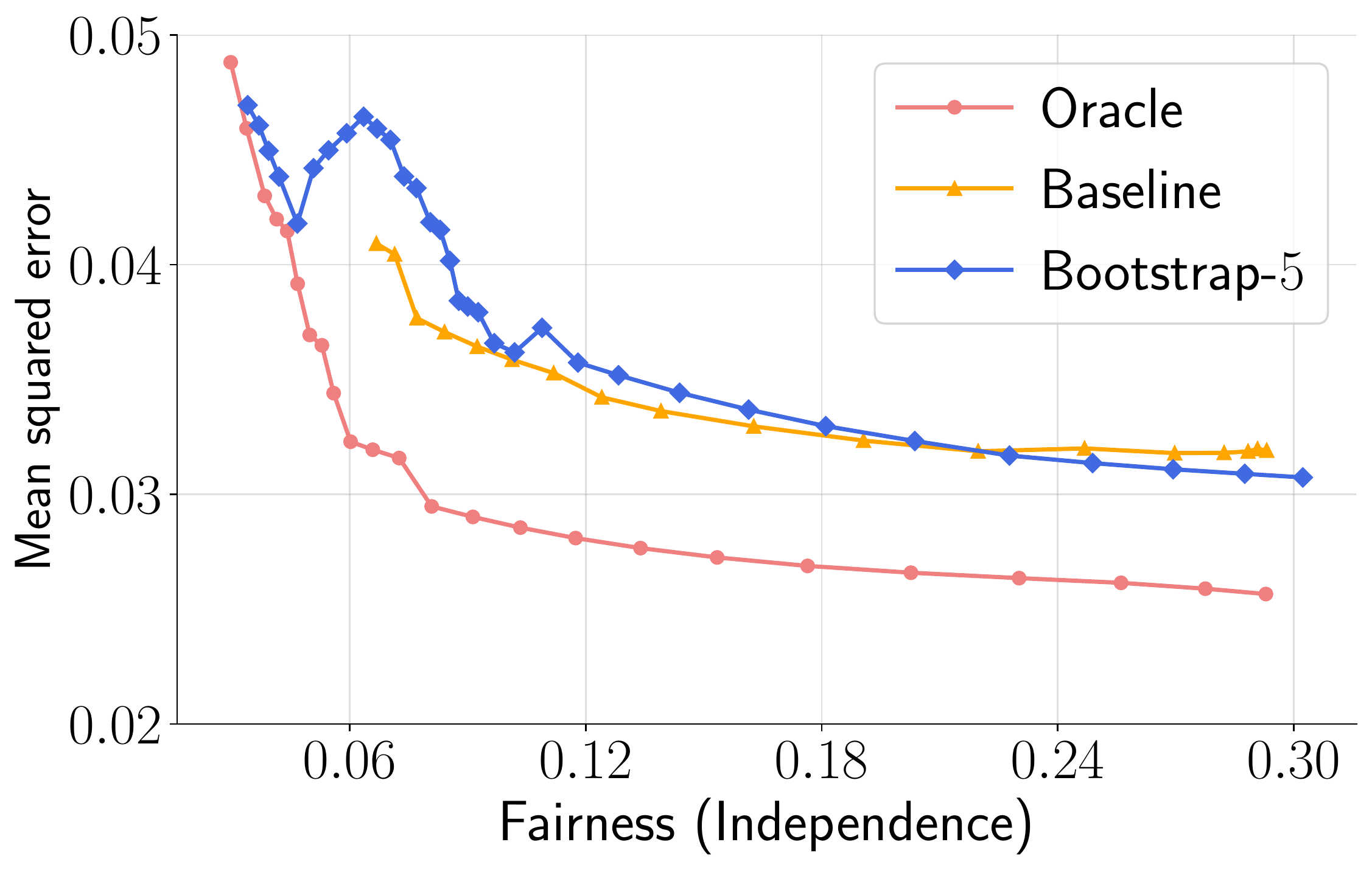}&
    \includegraphics[width=0.31\linewidth,clip]{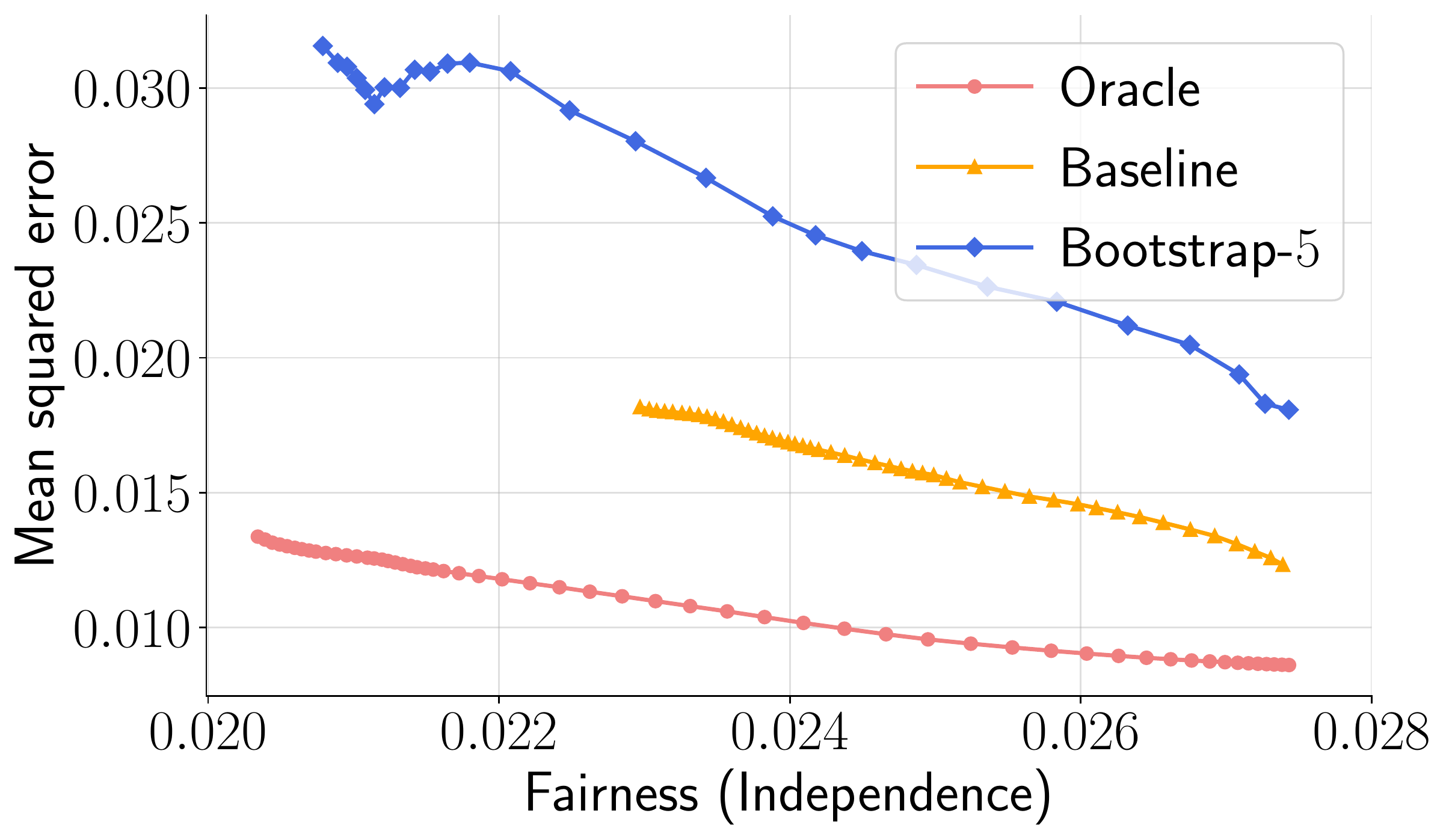}
    \\ (a) Adult dataset & (b) Crime dataset & (c) Insurance  dataset
    \end{tabular}
\caption{Performance of $\bss$ and \texttt{Baseline} for independence for various datasets. $\bss$ achieves much better fairness levels compared to \texttt{Baseline} throughout.}
\label{plots_real_different}
\end{figure}

%% file: main.bbl
\begin{thebibliography}{36}
\providecommand{\natexlab}[1]{#1}
\providecommand{\url}[1]{\texttt{#1}}
\expandafter\ifx\csname urlstyle\endcsname\relax
  \providecommand{\doi}[1]{doi: #1}\else
  \providecommand{\doi}{doi: \begingroup \urlstyle{rm}\Url}\fi

\bibitem[Agarwal et~al.(2018)Agarwal, Beygelzimer, Dud{\'\i}k, Langford, and
  Wallach]{agarwal2018reductions}
A.~Agarwal, A.~Beygelzimer, M.~Dud{\'\i}k, J.~Langford, and H.~Wallach.
\newblock A reductions approach to fair classification.
\newblock In \emph{International Conference on Machine Learning}, pages 60--69.
  PMLR, 2018.

\bibitem[Awasthi et~al.(2020)Awasthi, Kleindessner, and
  Morgenstern]{awasthi2020equalized}
P.~Awasthi, M.~Kleindessner, and J.~Morgenstern.
\newblock Equalized odds postprocessing under imperfect group information.
\newblock In \emph{International Conference on Artificial Intelligence and
  Statistics}, pages 1770--1780. PMLR, 2020.

\bibitem[Barocas et~al.(2019)Barocas, Hardt, and
  Narayanan]{barocas-hardt-narayanan}
S.~Barocas, M.~Hardt, and A.~Narayanan.
\newblock \emph{Fairness and Machine Learning: Limitations and Opportunities}.
\newblock fairmlbook.org, 2019.
\newblock \url{http://www.fairmlbook.org}.

\bibitem[Bu et~al.(2021)Bu, Wang, and Wornell]{bu2021sdp}
Y.~Bu, T.~Wang, and G.~W. Wornell.
\newblock {SDP} methods for sensitivity-constrained privacy funnel and
  information bottleneck problems.
\newblock In \emph{2021 IEEE International Symposium on Information Theory
  (ISIT)}, pages 49--54. IEEE, 2021.

\bibitem[Calmon et~al.(2017)Calmon, Wei, Vinzamuri, Ramamurthy, and
  Varshney]{calmon2017optimized}
F.~Calmon, D.~Wei, B.~Vinzamuri, K.~N. Ramamurthy, and K.~R. Varshney.
\newblock Optimized pre-processing for discrimination prevention.
\newblock In \emph{Advances in Neural Information Processing Systems}, pages
  3992--4001, 2017.

\bibitem[Castelnovo et~al.(2022)Castelnovo, Crupi, Greco, Regoli, Penco, and
  Cosentini]{castelnovo2022clarification}
A.~Castelnovo, R.~Crupi, G.~Greco, D.~Regoli, I.~G. Penco, and A.~C. Cosentini.
\newblock A clarification of the nuances in the fairness metrics landscape.
\newblock \emph{Scientific Reports}, 12\penalty0 (1):\penalty0 1--21, 2022.

\bibitem[Celis et~al.(2021{\natexlab{a}})Celis, Huang, Keswani, and
  Vishnoi]{celis2021fair}
L.~E. Celis, L.~Huang, V.~Keswani, and N.~K. Vishnoi.
\newblock Fair classification with noisy protected attributes: A framework with
  provable guarantees.
\newblock In \emph{International Conference on Machine Learning}, pages
  1349--1361. PMLR, 2021{\natexlab{a}}.

\bibitem[Celis et~al.(2021{\natexlab{b}})Celis, Mehrotra, and
  Vishnoi]{celis2021fairclassification}
L.~E. Celis, A.~Mehrotra, and N.~Vishnoi.
\newblock Fair classification with adversarial perturbations.
\newblock \emph{Advances in Neural Information Processing Systems},
  34:\penalty0 8158--8171, 2021{\natexlab{b}}.

\bibitem[Chen et~al.(2019)Chen, Kallus, Mao, Svacha, and
  Udell]{chen2019fairness}
J.~Chen, N.~Kallus, X.~Mao, G.~Svacha, and M.~Udell.
\newblock Fairness under unawareness: Assessing disparity when protected class
  is unobserved.
\newblock In \emph{Proceedings of the conference on fairness, accountability,
  and transparency}, pages 339--348, 2019.

\bibitem[Diamond and Boyd(2016)]{diamond2016cvxpy}
S.~Diamond and S.~Boyd.
\newblock Cvxpy: A python-embedded modeling language for convex optimization.
\newblock \emph{The Journal of Machine Learning Research}, 17\penalty0
  (1):\penalty0 2909--2913, 2016.

\bibitem[Efron(1992)]{efron1992}
B.~Efron.
\newblock Bootstrap methods: another look at the jackknife.
\newblock In \emph{Breakthroughs in statistics}, pages 569--593. Springer,
  1992.

\bibitem[G{\"o}lz et~al.(2019)G{\"o}lz, Kahng, and
  Procaccia]{golz2019paradoxes}
P.~G{\"o}lz, A.~Kahng, and A.~D. Procaccia.
\newblock Paradoxes in fair machine learning.
\newblock In \emph{Advances in Neural Information Processing Systems}, pages
  8340--8350, 2019.

\bibitem[Gupta et~al.(2018)Gupta, Cotter, Fard, and Wang]{gupta2018proxy}
M.~Gupta, A.~Cotter, M.~M. Fard, and S.~Wang.
\newblock Proxy fairness.
\newblock \emph{arXiv preprint arXiv:1806.11212}, 2018.

\bibitem[Hardt et~al.(2016)Hardt, Price, and Srebro]{hardt2016equality}
M.~Hardt, E.~Price, and N.~Srebro.
\newblock Equality of opportunity in supervised learning.
\newblock \emph{Advances in neural information processing systems}, 29, 2016.

\bibitem[Hashimoto et~al.(2018)Hashimoto, Srivastava, Namkoong, and
  Liang]{hashimoto2018fairness}
T.~Hashimoto, M.~Srivastava, H.~Namkoong, and P.~Liang.
\newblock Fairness without demographics in repeated loss minimization.
\newblock In \emph{International Conference on Machine Learning}, pages
  1929--1938. PMLR, 2018.

\bibitem[Huang et~al.(2019)Huang, Makur, Wornell, and
  Zheng]{huang2019universal}
S.-L. Huang, A.~Makur, G.~W. Wornell, and L.~Zheng.
\newblock On universal features for high-dimensional learning and inference.
\newblock \emph{arXiv preprint arXiv:1911.09105}, 2019.

\bibitem[Jung et~al.(2022)Jung, Chun, and Moon]{jung2022learning}
S.~Jung, S.~Chun, and T.~Moon.
\newblock Learning fair classifiers with partially annotated group labels.
\newblock In \emph{Proceedings of the IEEE/CVF Conference on Computer Vision
  and Pattern Recognition}, pages 10348--10357, 2022.

\bibitem[Kallus et~al.(2022)Kallus, Mao, and Zhou]{kallus2022assessing}
N.~Kallus, X.~Mao, and A.~Zhou.
\newblock Assessing algorithmic fairness with unobserved protected class using
  data combination.
\newblock \emph{Management Science}, 68\penalty0 (3):\penalty0 1959--1981,
  2022.

\bibitem[Kamishima et~al.(2011)Kamishima, Akaho, and
  Sakuma]{kamishima2011fairness}
T.~Kamishima, S.~Akaho, and J.~Sakuma.
\newblock Fairness-aware learning through regularization approach.
\newblock In \emph{2011 IEEE 11th International Conference on Data Mining
  Workshops}, pages 643--650. IEEE, 2011.

\bibitem[Krumpal(2013)]{krumpal2013determinants}
I.~Krumpal.
\newblock Determinants of social desirability bias in sensitive surveys: a
  literature review.
\newblock \emph{Quality \& quantity}, 47\penalty0 (4):\penalty0 2025--2047,
  2013.

\bibitem[Lahoti et~al.(2020)Lahoti, Beutel, Chen, Lee, Prost, Thain, Wang, and
  Chi]{lahoti2020fairness}
P.~Lahoti, A.~Beutel, J.~Chen, K.~Lee, F.~Prost, N.~Thain, X.~Wang, and E.~Chi.
\newblock Fairness without demographics through adversarially reweighted
  learning.
\newblock \emph{Advances in neural information processing systems},
  33:\penalty0 728--740, 2020.

\bibitem[Lamy et~al.(2019)Lamy, Zhong, Menon, and Verma]{lamy2019noise}
A.~Lamy, Z.~Zhong, A.~K. Menon, and N.~Verma.
\newblock Noise-tolerant fair classification.
\newblock \emph{Advances in Neural Information Processing Systems}, 32, 2019.

\bibitem[Lantz(2019)]{lantz2019machine}
B.~Lantz.
\newblock \emph{Machine learning with {R}: expert techniques for predictive
  modeling}.
\newblock Packt publishing ltd, 2019.

\bibitem[Lee et~al.(2019)Lee, Hashimoto, and Liang]{lee2019learning}
M.~Lee, T.~B. Hashimoto, and P.~Liang.
\newblock Learning autocomplete systems as a communication game.
\newblock \emph{arXiv preprint arXiv:1911.06964}, 2019.

\bibitem[Mary et~al.(2019)Mary, Calauzenes, and El~Karoui]{mary2019fairness}
J.~Mary, C.~Calauzenes, and N.~El~Karoui.
\newblock Fairness-aware learning for continuous attributes and treatments.
\newblock In \emph{International Conference on Machine Learning}, pages
  4382--4391. PMLR, 2019.

\bibitem[Mehrabi et~al.(2021)Mehrabi, Morstatter, Saxena, Lerman, and
  Galstyan]{mehrabi2021survey}
N.~Mehrabi, F.~Morstatter, N.~Saxena, K.~Lerman, and A.~Galstyan.
\newblock A survey on bias and fairness in machine learning.
\newblock \emph{ACM Computing Surveys (CSUR)}, 54\penalty0 (6):\penalty0 1--35,
  2021.

\bibitem[Mozannar et~al.(2020)Mozannar, Ohannessian, and
  Srebro]{mozannar2020fair}
H.~Mozannar, M.~Ohannessian, and N.~Srebro.
\newblock Fair learning with private demographic data.
\newblock In \emph{International Conference on Machine Learning}, pages
  7066--7075. PMLR, 2020.

\bibitem[Pleiss et~al.(2017)Pleiss, Raghavan, Wu, Kleinberg, and
  Weinberger]{pleiss2017fairness}
G.~Pleiss, M.~Raghavan, F.~Wu, J.~Kleinberg, and K.~Q. Weinberger.
\newblock On fairness and calibration.
\newblock In \emph{Advances in Neural Information Processing Systems}, pages
  5680--5689, 2017.

\bibitem[Redmond and Baveja(2002)]{redmond2002data}
M.~Redmond and A.~Baveja.
\newblock A data-driven software tool for enabling cooperative information
  sharing among police departments.
\newblock \emph{European Journal of Operational Research}, 141\penalty0
  (3):\penalty0 660--678, 2002.

\bibitem[Shah et~al.(2022)Shah, Bu, Lee, Das, Panda, Sattigeri, and
  Wornell]{shah2022selective}
A.~Shah, Y.~Bu, J.~K. Lee, S.~Das, R.~Panda, P.~Sattigeri, and G.~W. Wornell.
\newblock Selective regression under fairness criteria.
\newblock In \emph{International Conference on Machine Learning}, pages
  19598--19615. PMLR, 2022.

\bibitem[Verma and Rubin(2018)]{verma2018fairness}
S.~Verma and J.~Rubin.
\newblock Fairness definitions explained.
\newblock In \emph{2018 ieee/acm international workshop on software fairness
  (fairware)}, pages 1--7. IEEE, 2018.

\bibitem[Vershynin(2018)]{vershynin2018high}
R.~Vershynin.
\newblock \emph{High-dimensional probability: An introduction with applications
  in data science}, volume~47.
\newblock Cambridge university press, 2018.

\bibitem[Wang et~al.(2020)Wang, Guo, Narasimhan, Cotter, Gupta, and
  Jordan]{wang2020robust}
S.~Wang, W.~Guo, H.~Narasimhan, A.~Cotter, M.~Gupta, and M.~Jordan.
\newblock Robust optimization for fairness with noisy protected groups.
\newblock \emph{Advances in Neural Information Processing Systems},
  33:\penalty0 5190--5203, 2020.

\bibitem[Wasserman(2006)]{wasserman2006}
L.~Wasserman.
\newblock \emph{All of nonparametric statistics}.
\newblock Springer Science \& Business Media, 2006.

\bibitem[Zafar et~al.(2017)Zafar, Valera, Rogriguez, and
  Gummadi]{zafar2017fairness}
M.~B. Zafar, I.~Valera, M.~G. Rogriguez, and K.~P. Gummadi.
\newblock Fairness constraints: Mechanisms for fair classification.
\newblock In \emph{Artificial Intelligence and Statistics}, pages 962--970.
  PMLR, 2017.

\bibitem[Zemel et~al.(2013)Zemel, Wu, Swersky, Pitassi, and
  Dwork]{zemel2013learning}
R.~Zemel, Y.~Wu, K.~Swersky, T.~Pitassi, and C.~Dwork.
\newblock Learning fair representations.
\newblock In \emph{International Conference on Machine Learning}, pages
  325--333, 2013.

\end{thebibliography}
